\documentclass{article}
\PassOptionsToPackage{numbers, compress}{natbib}
\usepackage[main, final]{neurips_2026}

\usepackage{amsmath,amsfonts,bm}

\def\1{\bm{1}}

\DeclareMathAlphabet{\mathsfit}{\encodingdefault}{\sfdefault}{m}{sl}
\SetMathAlphabet{\mathsfit}{bold}{\encodingdefault}{\sfdefault}{bx}{n}

\usepackage{wrapfig,lipsum}
\usepackage[utf8]{inputenc} 
\usepackage[T1]{fontenc}    
\usepackage{hyperref}       
\usepackage{url}            
\usepackage{booktabs}       
\usepackage{amsfonts}       
\usepackage{nicefrac}       
\usepackage{microtype}      
\usepackage{xcolor}         
\usepackage{graphicx}
\usepackage{subcaption}
\usepackage{algorithm}
\usepackage{amsmath}
\usepackage{amssymb}
\usepackage{mathtools}
\usepackage{amsthm}
\usepackage[capitalize,noabbrev]{cleveref}
\usepackage{pifont}
\usepackage{comment}
\usepackage{multirow}
\usepackage{svg}
\usepackage{placeins}
\usepackage{needspace}

\theoremstyle{plain}
\newtheorem{theorem}{Theorem}[section]
\newtheorem{proposition}[theorem]{Proposition}
\newtheorem{lemma}[theorem]{Lemma}

\theoremstyle{definition}

\newtheorem{assumption}[theorem]{Assumption}
\theoremstyle{remark}

\usepackage{kotex}
\usepackage[textsize=tiny]{todonotes}

\newcommand{\cmark}{\textcolor{green!50!black}{\ding{51}}} 
\newcommand{\xmark}{\textcolor{red!70!black}{\ding{55}}} 

\title{Sparsely-Supervised Data Assimilation \\via Physics-Informed Schr\"odinger Bridge}

\author{Dohyun Bu\thanks{Equal contribution.} \qquad
Chanho Kim\footnotemark[1] \qquad
Seokun Choi \qquad
Jong-Seok Lee\thanks{Corresponding author.} \\ 
Korea Advanced Institute of Science and Technology \\
\texttt{\{dohyun.bu, chanh99kim, seokunchoi, jongseok.lee\}@kaist.ac.kr}
}

\begin{document}

\maketitle

\begin{abstract}
Data assimilation (DA) for systems governed by partial differential equations (PDE) aims to reconstruct full spatiotemporal fields from sparse high-fidelity (HF) observations while respecting physical constraints.
While full-grid low-fidelity (LF) simulations provide informative priors in multi-fidelity settings, recovering an HF field consistent with both sparse observations and the governing PDE typically requires per-instance test-time optimization, which becomes a major bottleneck in time-critical applications.
To alleviate this, amortized reconstruction using generative models has recently been proposed; however, such approaches rely on full-field HF supervision during training, which is often impractical in real-world settings.
From a more realistic perspective, we propose the \textit{Physics-Informed Conditional Schr\"odinger Bridge} (PICSB), which transports an informative LF prior toward an auxiliary-energy-defined, observation-conditioned HF endpoint law without any additional inference-time guidance.
To enable learning without HF endpoints, PICSB employs an iterative surrogate-endpoint refresh scheme, and directly incorporates PDE residuals into the training objective while enforcing observations via hard conditioning throughout sampling.
Experiments on fluid PDE benchmarks demonstrate that PICSB enables extremely fast spatiotemporal field reconstruction while achieving strong accuracy under sparse HF supervision.
\end{abstract}

\section{Introduction}\label{sec:int}
\vspace{-1em}

Partial differential equations (PDEs) have long provided fundamental tools for modeling and predicting natural phenomena. 
A central challenge in many fluid and flow applications is full-field reconstruction: recovering a high-resolution spatiotemporal field from sparse observations while remaining consistent with the governing dynamics, a setting broadly studied under data assimilation (DA) \citep{farchi2023online, stuart2010inverse}.
In such regimes, the available observations provide only partial coverage, and the underlying physical model provides feasibility constraints on recovering the full field \citep{buzzicotti2023data, rozet2023score}.

A common way to obtain accurate full-fields is high-fidelity (HF) numerical simulation, but it is often prohibitively time-consuming at the resolutions and regimes of interest \citep{zhou2024multi, he2020multi}.
This computational latency is a major bottleneck in time-critical applications, including rapid-update numerical weather forecasting \citep{stanley2004an}, operational ocean prediction \citep{moore2019synthesis}, and real-time hydrologic forecasting \citep{jafarzadegan2021sequential}.
Practitioners therefore rely on low-fidelity (LF) models or proxy fields---from simplified physics and coarse discretizations to observation-based interpolation---which are fast but systematically biased \citep{niu2024multi, fernandez2016review, peherstorfer2018survey}.

These constraints make multi-fidelity DA a common practical regime: abundant LF fields provide informative priors, while HF information is available only through sparse observations \citep{kou2022transfer, chakraborty2021transfer}.
However, even in this multi-fidelity setting, classical physics-based DA typically enforces observation and PDE consistency by solving a new constrained inverse problem for each instance, incurring substantial computational cost \citep{li2025physics}.
This motivates amortized reconstruction, which maps a new LF field and a few HF observations directly to a full HF field without test-time re-optimization \citep{rozet2023score}.

Recent progress toward amortized reconstruction has been driven by generative modeling.
Diffusion models, in particular, have become a standard tool for high-dimensional conditional generation, making sparse-to-full field recovery a natural application \citep{shysheya2024conditional, xu2025partial}.
In PDE-constrained DA, a prevailing recipe learns a diffusion prior over fully observed HF fields and then produces reconstructions via iterative sampling, while enforcing consistency with sparse observations and the governing PDE via residual-based guidance during sampling \citep{huang2024diffusionpde, cheng2025gradientfree, utkarsh2025physicsconstrained, du2024conditional}.
While effective, this paradigm typically relies on many sampling steps with inference-time corrections to enforce observations and PDE constraints, increasing inference cost in time-critical settings.

An appealing alternative is provided by Schr\"odinger bridge (SB) formulations, which learn a stochastic optimal transport between endpoint distributions. 
For multi-fidelity DA, SB offers a particularly natural geometry by transporting samples from an informative LF field, often requiring substantially fewer sampling steps than diffusion models that start from uninformative noise.
Although SB-based methods for PDE field modeling remain relatively scarce, SB-style transport has been studied extensively in broader generative settings such as image synthesis \citep{liu2023i2sb} and related bridge formulations \citep{li2025walking}, with recent PDE-based instances emerging as well \citep{li2025physics}.
However, existing SB approaches typically still rely on full-field HF supervision to learn the transport, which can be costly in practice because obtaining such supervision requires generating full-grid HF simulation targets \citep{tripura2024multi, xu2026data}.
Consequently, amortized multi-fidelity DA under sparse HF supervision remains largely underexplored, particularly for PDE-governed systems that require physically consistent generative models.

In this work, we propose the \emph{Physics-Informed Conditional Schr\"odinger Bridge Model} (PICSB) for amortized multi-fidelity DA with sparse HF supervision.
PICSB models the conditional distribution of the full HF field given an LF input and sparse HF observations, enforcing the observations via hard conditioning.
To avoid reliance on full-field HF supervision, we train PICSB solely from physics residuals, using an auxiliary-energy-driven conditional bridge that starts from an informative LF prior and transports it toward an observation-conditioned endpoint law induced by the governing PDE constraints and sparse observations.
This physics-residual objective enables guidance-free inference while retaining the efficiency of LF-to-HF transport, yielding fast reconstruction with few sampling steps.
Crucially, since HF endpoints are unavailable, we enable bridge learning via an iterative surrogate-endpoint refresh that periodically updates a slowly moving terminal reference distribution.
Our contributions are as follows:
\vspace{-0.3em}
\begin{itemize}
    \item We introduce a conditional SB for amortized reconstruction from LF inputs and sparse HF observations via hard conditioning, enabling fast inference with few sampling steps.
    \item We characterize the population endpoint selected by PICSB as the LF-proximal minimizer of an auxiliary inverse-PDE energy.
    \item We empirically show that PICSB achieves competitive accuracy under sparse HF supervision while delivering orders-of-magnitude faster reconstruction than test-time optimization and guidance-based sampling.
\end{itemize}

\section{Related Works}\label{sec:rw}

\vspace{-0.5em}
\paragraph{Optimization-based DA.} 
Traditional per-instance DA for PDE-governed systems has primarily relied on numerical methods that solve observation-conditioned inverse problems under physical constraints \citep{elia2012variational, lemke2016adjoint}. 
To reduce the high computational cost of such approaches, extensive efforts have focused on efficiency-oriented modeling aimed at HF simulations \citep{zauner2022nudging,wang2024accelerating}.  
A related class of approaches replaces explicit numerical solvers with neural-network-based optimization formulations, which parameterize solution fields with neural networks and minimize a composite objective that enforces both observation fitting and governing PDE constraints \citep{raissi2019physics}.
More recently, implicit neural representation methods have been explored to reconstruct continuous spatiotemporal fields from sparse observations by directly optimizing latent representations \citep{luo2024continuous}.

\vspace{-0.5em}

\paragraph{Amortized Reconstruction.}
To avoid repeated per-instance optimization, recent work has sought zero-shot DA by learning reusable models that can be directly queried at test time.
One line of research learns forward solution operators that map PDE parameters or initial conditions to full fields, enabling fast evaluation without iterative solvers \citep{li2021fourier, chen2024dataefficient, cho2024parameterized}.
A second line of work targets amortized inverse reconstruction via generative models, typically trained with fully observed HF fields.
Many diffusion-based formulations enforce observation and PDE consistency through inference-time residual guidance \citep{huang2024diffusionpde, cheng2025gradientfree, utkarsh2025physicsconstrained}, whereas more recent approaches incorporate PDE residuals into the training objective to reduce reliance on inference-time corrections \citep{shu2023physics, li2025physics, xu2025partial}.

\begin{table*}[t]
    \small
    \centering
    \setlength{\tabcolsep}{5pt}
    \caption{
    Comparison of learning-based field reconstruction methods for PDE-governed systems.
    \textbf{Zero-shot} indicates amortized reconstruction without test-time re-optimization.
    \textbf{DA} indicates the capability to solve an observation-conditioned inverse problem using sparse HF observations.
    \textbf{Label-Efficient} indicates training under sparse HF supervision, where HF targets are available only at partial spatiotemporal locations.
    \textbf{Guidance-free} indicates that PDE residuals are incorporated directly into training, avoiding inference-time correction.
    }
    \vspace{-0.5em}
    \begin{tabular}{lcccc}
        \toprule
        Method 
        & \begin{tabular}[c]{c}Zero-shot\end{tabular}
        & \begin{tabular}[c]{c}DA\end{tabular}
        & \begin{tabular}[c]{c}Label-Efficient\end{tabular}
        & \begin{tabular}[c]{c}Guidance-free\end{tabular}
        \\
        \midrule
        PINNs \citep{raissi2019physics}
        & \xmark
        & \cmark
        & \cmark
        & \cmark \\
        MMGN \citep{luo2024continuous}
        & \xmark
        & \cmark
        & \cmark
        & \xmark \\
        FNO, NOICL \citep{li2021fourier,chen2024dataefficient}
        & \cmark
        & \xmark
        & \xmark
        & \xmark \\
        P\(^2\)INN \citep{cho2024parameterized}
        & \cmark
        & \xmark
        & \cmark
        & \cmark \\
        DiffusionPDE, ECI, PCFM \citep{huang2024diffusionpde,cheng2025gradientfree,utkarsh2025physicsconstrained}
        & \cmark
        & \cmark
        & \xmark
        & \xmark \\
        PIDM, PalSB, PPIDM \citep{shu2023physics,li2025physics,xu2025partial}
        & \cmark
        & \cmark
        & \xmark
        & \cmark \\
        \midrule
        \textbf{PICSB (Ours)}
        & \cmark
        & \cmark
        & \cmark
        & \cmark \\
        \bottomrule
    \end{tabular}
    \label{tab:comparison}
    \vspace{-1em}
\end{table*}

\vspace{-0.5em}

\paragraph{Limitations.}
Table~\ref{tab:comparison} summarizes prior learning-based field reconstruction methods for PDE-governed systems in terms of four desirable features: \textbf{Zero-shot} inference, \textbf{DA} capability, \textbf{Label-Efficient} learning under sparse HF supervision, and \textbf{Guidance-free} physics integration via training objectives.
First, optimization-based methods can assimilate sparse observations while explicitly enforcing the governing physics, but they are not zero-shot and require test-time re-optimization.
Second, operator-learning methods enable zero-shot evaluation but largely act as forward surrogates and thus lack DA capability for observation-conditioned reconstruction.
Finally, generative inverse methods provide DA capability, but most existing approaches are not label-efficient, as they assume full-field HF supervision; moreover, guidance-based formulations incur additional inference-time correction costs.
Consequently, no existing method simultaneously satisfies all four criteria.


\section{Preliminaries}\label{sec:pre}

\vspace{-0.3em}
We consider a multi-fidelity field reconstruction problem on a fixed discrete grid \(\mathcal X\!=\!\mathbb R^n\).
\(X_0 \!\in\! \mathcal X\) denotes a LF field defined on the full grid, \(X_1 \!\in\! \mathcal X\) denotes the corresponding latent HF field on the same grid, and \(Y \!\in\! \mathbb{R}^m\) denotes sparse HF observations with \(m \!\ll\! n\). 
The observations are related through a known operator \(\mathcal{H}:\mathcal X \!\to\! \mathbb{R}^m\) such that \(Y \!=\!\mathcal{H}(X_1)\).
Given \(Y\!=\!y\), we denote by \(\pi_0(\cdot\,|\,y)\) the conditional LF distribution.
In a fully supervised bridge setting, one would also have access to full-field HF endpoint samples from the conditional distribution \(\pi_1(\cdot\,|\,y)\).
In our sparse-supervision setting, however, such full-field HF endpoints are not empirically available.
We therefore specify the desirable HF endpoint through an auxiliary inverse-PDE energy, rather than through observed full-field HF labels.

\subsection{Desirable HF endpoint via auxiliary inverse-PDE energy}
\label{subsec:pre:desendpoint}
\vspace{-0.5em}
For a fixed observed pair \((x_0,y)\) where \(x_0\) is a realization of the LF field \(X_0\), let \(\mathcal C_y := \{x\in\mathcal X:\mathcal H(x)=y\}\) be the observation-consistent set and \(\mathcal R_h:\mathcal X\to\mathbb R^r\) be the discretized PDE residual operator on grid scale \(h\).
Following the least-squares finite element viewpoint, a PDE solution can be characterized as a minimizer of a squared residual functional \citep{bochev2009least}. 
We therefore define
\begin{equation}
\label{eq:pre:aux_energy}
    \mathcal E_{\lambda,h}(x;x_0,y)
    :=
    \frac{\lambda_{\mathrm{PDE}}}{2}\|\mathcal R_h(x)\|_2^2
    +
    \frac{\lambda_{\mathrm{LF}}}{2}\|x-x_0\|_2^2,
    \quad x\in\mathcal C_y.
\end{equation}
We set \(\mathcal E_{\lambda,h}(x;x_0,y)=+\infty\) for \(x\notin\mathcal C_y\).
Here, the first term measures PDE consistency, while the second term anchors the solution to the LF input.
Because the energy is finite only on \(\mathcal C_y\), every finite-energy endpoint satisfies the sparse observations exactly.

The desirable HF endpoint law for a given LF input and sparse observation is then defined as the Gibbs law
\begin{equation}
\label{eq:pre:gibbs_law}
    \pi_{1+}^{\lambda,\beta}(dx\,|\,x_0,y)
    =
    \frac{1}{Z_{\lambda,\beta}(x_0,y)}
    \exp\!\left[-\beta\,\mathcal E_{\lambda,h}(x;x_0,y)\right]\,d\lambda_y(x),
\end{equation}
where \(\beta>0\) is an inverse temperature, \(Z_{\lambda,\beta}\) is the normalizing constant, and \(\lambda_y\) denotes the Lebesgue measure over the unobserved coordinates with the observed coordinates fixed to \(y\).
Thus, in the sparse inverse-PDE setting, the target HF endpoint is not an observed full-field label; it is an observation-consistent, PDE residual minimizing, LF-anchored endpoint induced by the auxiliary energy.
For compact notation, we write \(\pi_{1+}(\cdot\,|\,y)\) for the marginal endpoint law obtained by averaging the input-conditioned law \(\pi_{1+}^{\lambda,\beta}(\cdot\,|\,x_0,y)\) over \(x_0\sim\pi_0(\cdot\,|\,y)\).

\subsection{Conditional Schr\"odinger bridge formulation}\label{subsec:pre:consb}
\vspace{-0.5em}
For a fixed conditioning variable \(Y\!=\!y\), the conditional SB problem seeks a path measure \(\mathbb{P}(\cdot\,|\,y)\) on \(\mathcal{C}\!=\!C([0,1],\mathbb{R}^n)\) whose endpoint marginals satisfy \(\mathbb{P}_0(\cdot \mid y) \!=\! \pi_0(\cdot\,|\,y)\) and \(\mathbb{P}_1(\cdot \mid y) \!=\! \pi_{1}(\cdot\,|\,y)\) \citep{shi2022conditional}.
Specifically, given an uncontrolled Wiener measure \(\mathbb{Q}\) as a reference path measure, the conditional SB is defined as 
\begin{equation}
\label{eq:csb_kl_compact}
\mathbb{P}^\star(\cdot\,|\,y) = 
\underset{\mathbb{P}(\cdot\,|\,y)}{\text{argmin}}\;
\mathrm{KL}\!\left(\mathbb{P}(\cdot\,|\,y)\,\|\,\mathbb{Q}\right),
\quad
\text{s.t.}\quad \mathbb{P}_0(\cdot\,|\,y)\!=\!\pi_0(\cdot\,|\,y), \quad \mathbb{P}_1(\cdot\,|\,y)\!=\!\pi_{1}(\cdot\,|\,y).
\end{equation}
In supervised bridge learning, the terminal distribution is usually specified by empirical HF endpoint samples.
In our inverse-PDE setting, it is instead specified implicitly by the auxiliary energy in Eq.~\eqref{eq:pre:aux_energy}.
The SB formulation then provides a principled way to transport probability mass from the LF-informed source distribution \(\pi_0(\cdot\,|\,y)\) toward this desirable endpoint law \(\pi_{1+}(\cdot\,|\,y)\) while remaining close to simple reference dynamics.

\vspace{-0.5em}
\subsection{Bridge reference and flow-based parameterization}\label{subsec:pre:bridge}
\vspace{-0.5em}
Under \(\mathbb{Q}\), conditioning on the endpoints \((X_0,X_1)=(x_0,x_1)\) yields the Brownian bridge \(\mathbb{Q}(\cdot\mid x_0, x_1)\).  
With noise scale \(\varepsilon\!>\!0\), the intermediate state satisfies
\begin{equation}
\label{eq:bb_conditional_compact}
X_\tau \mid (x_0,x_1) ~\sim~ \mathcal{N}\!\big((1\!-\!\tau)x_0\!+\!\tau x_1,\; \varepsilon\,\tau(1\!-\!\tau)\,I\big),
\end{equation}
which allows direct sampling at an arbitrary intermediate time \(\tau\).
This property underlies Iterative Markovian Fitting (IMF) and its diffusion-based implementations \citep{liu2023i2sb, shi2023diffusion, peluchetti2023diffusion}.

To construct an explicit sampler, one may parameterize a conditional path measure via a time-dependent velocity field \(v_\tau(\cdot\,|\,y)\) \citep{bortoli2024schrodinger}.
Such a parameterization induces the controlled diffusion driven by standard Brownian motion \(B_\tau\),
\begin{equation}
\label{eq:controlled_diffusion_compact}
dX_\tau = v_\tau(X_\tau\,|\,y)\,d\tau + \sqrt{\varepsilon}\,dB_\tau,
\quad X_0 \sim \pi_0(\cdot\,|\,y).
\end{equation}

When paired samples \(\{(x_0^i,y^i,x_1^i)\}_{i=1}^N\) are available, IMF methods learn \(v\) by matching the drift of the Brownian bridge, leading to the bridge-matching objective given by
\begin{equation}
\label{eq:bridge_matching_compact}
\min_{v}\;
\mathbb{E}_{i, \tau}\!\left[
\mathbb{E}_{x_\tau^i \sim \mathbb{Q}(\cdot \mid x_0^i,x_1^i)}
\left\|v_\tau(x_\tau^i \!\mid\! y^i) - \frac{x_1^i \!-\! x_\tau^i}{1\!-\!\tau}\right\|
\, 
\right].
\end{equation}
In our sparse-supervision setting, the full-field endpoints \(x_1^i\) are unavailable; the methodology below therefore uses model-induced surrogate endpoints shaped by the auxiliary energy in Eq.~\eqref{eq:pre:aux_energy}.

\section{Methodology}\label{sec:met}

\vspace{-0.3em}
We consider the multi-fidelity DA setting in Section~\ref{sec:pre}.
The training data consist of pairs \(\mathcal{D}_{\mathrm{train}}\!=\!\{(x_0^i,y^i)\}_{i=1}^N\), where each \(x_0^i\in\mathbb{R}^n\) is a full-grid LF field and \(y^i\in\mathbb{R}^m\) denotes sparse HF observations of an unobserved latent HF field \(x_1^i\in\mathbb{R}^n\), i.e., \(y^i=\mathcal H(x_1^i)\).
The full-field HF endpoint \(x_1^i\) is not used as a training label.
Given a new pair \((x_0,y)\) at test time, our goal is to sample plausible HF reconstructions that respect the LF prior information in \(x_0\), satisfy the sparse observations \(y\), and remain physically consistent.
We adopt the auxiliary-energy-driven conditional SB viewpoint introduced in Section~\ref{subsec:pre:desendpoint}--\ref{subsec:pre:consb}.
In the idealized formulation, the conditional bridge transports the LF source law \(\pi_0(\cdot\,|\,y)\) toward the desirable endpoint law \(\pi_{1+}(\cdot\,|\,y)\), rather than toward an empirical full-field HF distribution.
In practice, we are given a concrete LF instance \(x_0\) rather than an explicit conditional density \(\pi_0(\cdot\,|\,y)\).
Accordingly, our sampler is implemented as an input-conditioned transport process that starts from the provided \(x_0\) and is explicitly conditioned on the sparse observations \(y\), generating a terminal sample \(\hat{x}_1\sim \pi_\theta(\cdot\,|\,x_0,y)\) whose population target is the desirable endpoint law \(\pi_{1+}^{\lambda,\beta}(\cdot\,|\,x_0,y)\).

\subsection{Physics-informed conditional reconstruction}
\label{subsec:method:formulation}
\vspace{-0.5em}
Recall from Section~\ref{subsec:pre:desendpoint} that the desirable endpoint is defined on the observation-consistent set \(\mathcal C_y\) and shaped by the PDE residual operator \(\mathcal R_h\).
Here we specify how these two components are enforced in the sampler.
Throughout, we focus on the standard sparse-sensor setting where \(\mathcal{H}\) selects a subset of grid coordinates.
Let \(\odot\) denote element-wise multiplication.
We use the coordinate projection onto \(\mathcal{C}_y\):
\begin{equation}
\Pi_{\mathcal{C}_y}(x) \;:=\; \mathbf{M}\odot \tilde{y} \;+\; (1-\mathbf{M})\odot x,
\label{eq:method:proj_Cy}
\end{equation}
where \(\mathbf{M}\in\{0,1\}^n\) is a binary mask indicating the observed coordinates and \(\tilde{y}\in\mathbb R^n\) denotes the full-grid embedding of the sparse observations.
When no confusion arises, we simply write \(y\) for \(\tilde{y}\).

Due to discretization and numerical approximation, we define the residual floor and \(\delta\)-feasible physics set as
\begin{equation}
\delta^\star(y ; h) \;:=\; \min_{x\in\mathcal{C}_y}\,\|\mathcal{R}_h(x)\|,
\qquad
\mathcal{P}^{\delta, y}_{\mathcal{R}_h} \;:=\; \{x\in\mathcal{C}_y:\|\mathcal{R}_h(x)\|\le \delta\}.
\label{eq:method:res_floor_set}
\end{equation}
Our goal is to learn a conditional sampler whose terminal samples remain in \(\mathcal C_y\) by construction and have small PDE residual.
The LF anchoring term in Eq.~\eqref{eq:pre:aux_energy} is implemented through the Brownian-bridge reference initialized from \(x_0\), so that the learned transport is biased toward HF endpoints close to the LF input.

\subsection{Learning with surrogate endpoints}
\label{subsec:method:theory}
\vspace{-0.5em}
We now specify the sampling procedure and the training objective.
Our sampler is parameterized by a time-dependent velocity field \(v_\theta(x,\tau\mid y)\), inducing the controlled diffusion in Eq.~\eqref{eq:controlled_diffusion_compact}. 
Hard observation feasibility is enforced by projection at every discretized step.

\vspace{-0.8em}
\begin{wrapfloat}{algorithm}{t}{0.55\textwidth}
\caption{Hard-conditioned sampler \(\mathrm{Sample}_{\theta}\)}
\label{alg:sample-theta}
\small

\noindent\textbf{Input:} 
initial state \(x_{\tau}\) at time \(\tau=\tau_{t_0}\);
sparse observations \(y\);
velocity network \(v_\theta\);
projection \(\Pi_{\mathcal{C}_y}\) in Eq.~\eqref{eq:method:proj_Cy};
noise scale \(\varepsilon\);
step size \(s\).\\[-0.5em]

\textbf{Output:} terminal sample \(x_{\tau_{T}}\).\\[-0.5em]

\noindent\textbf{Initialize:} 
\(x_{\tau_{t_0}}\leftarrow \Pi_{\mathcal C_y}(x_{\tau})\).\\[-0.4em]

\noindent\textbf{For} \(t=t_0,\ldots,T-1\), draw
\(z_t\sim\mathcal N(0,I)\) and update
\begin{equation*}
x_{\tau_{t+1}}
\leftarrow
\Pi_{\mathcal C_y}\!\left(
x_{\tau_t}
+
s\,v_\theta(x_{\tau_t}, \tau_t\mid y)
+
\mathbf{1}_{\{t<T-1\}}\sqrt{\varepsilon s}\,z_t
\right).
\end{equation*}

\end{wrapfloat}

\paragraph{Surrogate bridge reference.}
Bridge-matching objectives such as Eq.~\eqref{eq:bridge_matching_compact} require paired HF endpoints \(x_1\), which are unavailable in our setting.
We therefore introduce a surrogate terminal reference \(\tilde{\pi}_1(\cdot \mid x_0,y)\), defined as the terminal law induced by a slowly updated static velocity network \(v_{\tilde{\theta}}\) under hard conditioning.
Concretely, drawing \(\tilde{x}_1\sim\tilde{\pi}_1(\cdot\,|\,x_0,y)\) means running the sampler \(\mathrm{Sample}_{\tilde{\theta}}\) initialized from the given LF field \(x_0\) and projected onto \(\mathcal{C}_y\) throughout.
Given \((x_0,y)\), we sample \(\tau\!\sim\!\mathrm{Unif}([0,1])\) and draw an intermediate state \(\tilde{x}_\tau\) from the Brownian-bridge closed form in Eq.~\eqref{eq:bb_conditional_compact} between endpoints \((x_0,\tilde{x}_1)\), followed by projection to ensure \(\tilde{x}_\tau\in\mathcal{C}_y\).
Starting from \(\tilde{x}_\tau\), Algorithm~\ref{alg:sample-theta} produces a terminal sample \(\hat{x}_1 \!=\! \mathrm{Sample}_\theta(\tilde{x}_\tau \mid y,\tau).\)

\vspace{-0.3em}
\paragraph{Physics-residual training objective.}
We train the \emph{trainable} sampler \(v_\theta\) using the PDE-consistency term of the auxiliary energy:
\begin{align}
\mathcal{L}(\theta)
\;:=\;
\mathbb{E}_{(x_0,y)\sim\mathcal{D}_{\mathrm{train}}}
\mathbb{E}_{\tau\sim\mathrm{Unif}([0,1])}
\left[
\frac{1}{2}\|\mathcal{R}_h(\hat{x}_1)\|_2^2
\right],
\label{eq:method:pde_loss}
\end{align}
The observation constraint in Eq.~\eqref{eq:pre:aux_energy} is enforced by hard projection, while the LF anchoring is induced by constructing the bridge from \(x_0\) under the Brownian reference.
Thus, the residual objective shapes the terminal samples toward the auxiliary energy defined desirable endpoint without requiring empirical full-field HF labels.

\vspace{-0.3em}
\paragraph{Iterative refresh of \(\tilde{\pi}_1\).}
Because \(\tilde{\pi}_1(\cdot\,|\,x_0,y)\) is model-induced, we adopt a periodic refresh scheme with a slowly updated static model, analogous to target-network stabilization \citep{mnih2015human}.
At refresh round \(k\), we (i) sample surrogate endpoints \(\tilde{x}_1\sim\tilde{\pi}_1^{(k)}(\cdot\,|\,x_0,y)\) using \(v_{\tilde{\theta}}\) to instantiate a surrogate Brownian-bridge reference, and (ii) update \(\theta\) by minimizing Eq.~\eqref{eq:method:pde_loss} under this fixed reference.
This yields a distributional update map
\begin{equation}\label{eq:method:refresh_map}
\tilde{\pi}_1^{(k)}(\cdot\,|\,x_0,y)
\;\mapsto\;
\tilde{\pi}_1^{(k+1)}(\cdot\,|\,x_0,y),
\end{equation}
where \(\tilde{\pi}_1^{(k+1)}\) denotes the terminal law induced by the updated sampler \(v_\theta\).
The refresh procedure should therefore be understood as amortizing the auxiliary inverse-PDE solve defined in Section~\ref{subsec:pre:desendpoint}.

The following theorem characterizes the limiting population endpoint of the PICSB refresh map under a vanishing terminal free-energy gap induced by the auxiliary energy in Eq.~\eqref{eq:pre:aux_energy}.

\begin{theorem}[Population endpoint characterization under a vanishing terminal free-energy gap]
\label{thm:aux_endpoint_limit}
Fix \((x_0,y)\).
Assume that hard projection is exact, \(\mathcal C_y\) is nonempty and closed, \(\mathcal R_h\) is continuous on \(\mathcal C_y\), and the residual floor \(\delta^\star(y;h)\) in Eq.~\eqref{eq:method:res_floor_set} is attained.
For fixed \(\lambda_{\mathrm{PDE}},\lambda_{\mathrm{LF}},\beta>0\), let \(\{\tilde{\pi}_1^{(k)}(\cdot|x_0,y)\}_{k\ge0}\) be the population terminal laws generated by the PICSB refresh map in Eq.~\eqref{eq:method:refresh_map}.
If the terminal free-energy gap of this refresh sequence under \(\mathcal E_{\lambda,h}\) vanishes, then \(\tilde{\pi}_1^{(k)}(\cdot|x_0,y)\) converges to \(\pi_{1+}^{\lambda,\beta}(\cdot|x_0,y)\) in KL divergence, and hence also in total variation and weakly.
Moreover, in the zero-temperature and residual-dominant limit \(\beta\to\infty\) and \(\lambda_{\mathrm{PDE}}/\lambda_{\mathrm{LF}}\to\infty\), the population PICSB
endpoint law concentrates on
\[
\mathcal X_{1+}^\star(x_0,y)
:=
\operatorname{argmin}_{x}
\left\{
\|x-x_0\|_2^2
\,:\,
x\in\mathcal P_{\mathcal R_h}^{\delta^\star,y}
\right\}.
\]
\end{theorem}

Theorem~\ref{thm:aux_endpoint_limit} characterizes the limiting endpoint approached by the population PICSB refresh map under the stated free-energy condition.
As the terminal free-energy gap vanishes, the terminal law converges to the Gibbs endpoint law in Eq.~\eqref{eq:pre:gibbs_law}.
The concentration result further shows that, in the inverse-PDE limit, hard projection enforces observation consistency, the PDE residual selects residual-minimizing states in \(\mathcal C_y\), and the Brownian-bridge geometry supplies the LF anchor that selects the endpoint closest to \(x_0\) when multiple residual minimizers exist.
Thus, this population characterization links the PICSB endpoint to an LSFEM~\citep{bochev2009least}-style auxiliary-energy label.
A formal proof is given in Appendix~\ref{app:thm:aux_endpoint_limit}.

\begin{figure*}[!t]
\centering
\includegraphics[width=1\linewidth]{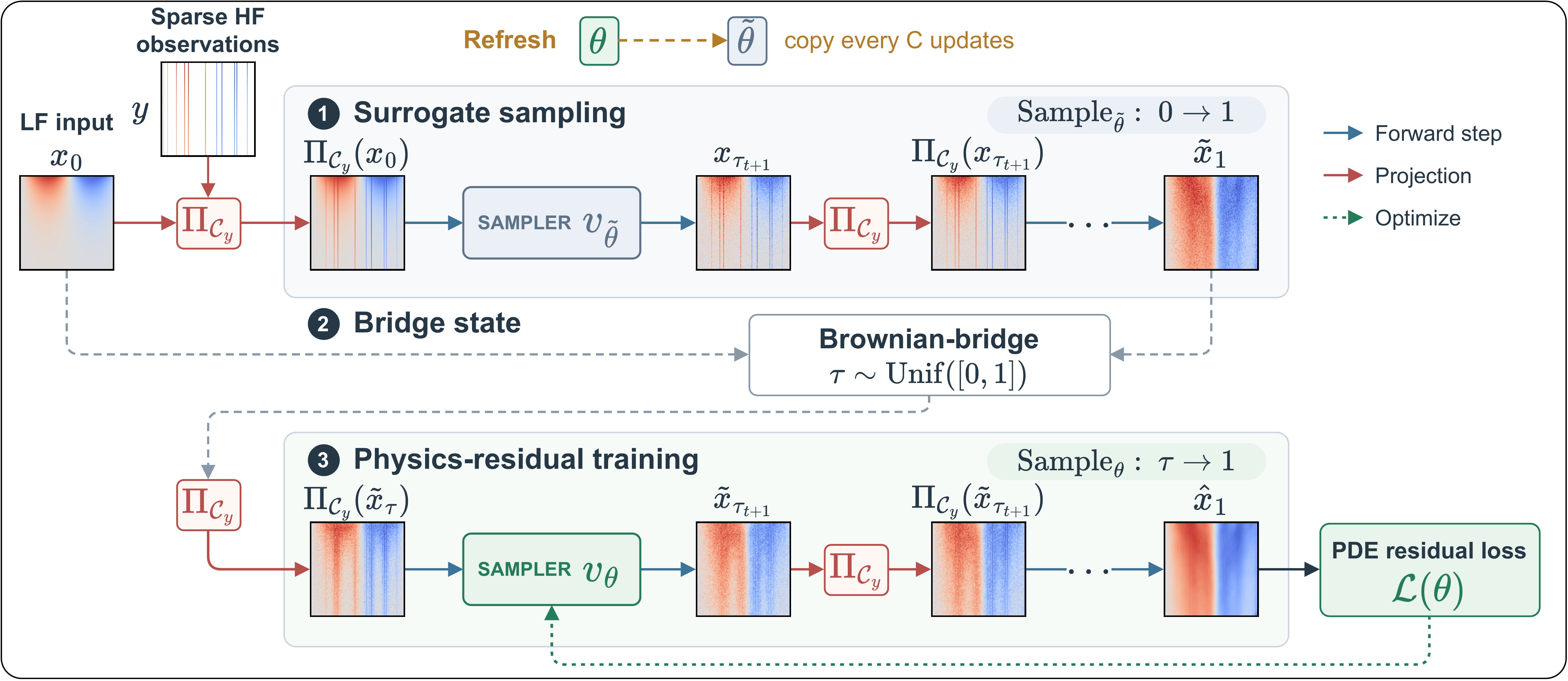}
\vspace{-1.8em}
\caption{Overview of PICSB training with sparse supervision. 
A periodically refreshed static sampler with parameters \(\tilde{\theta}\) generates surrogate endpoints for bridge construction, while the trainable sampler with parameters \(\theta\) is optimized using the PDE residual under hard observation projection.}
\vspace{-0.5em}
\label{fig:architecture}
\end{figure*}

\subsection{PICSB training procedure}\label{sec:method:implementation}
\vspace{-0.3em}

Figure~\ref{fig:architecture} summarizes the PICSB training procedure.
Training alternates between a \emph{static} sampler \(v_{\tilde{\theta}}\), which supplies the current surrogate bridge reference, and a \emph{trainable} sampler \(v_\theta\), which is optimized using the PDE residual.
At the beginning of each refresh round, the current parameters are copied as
\(\tilde{\theta}\leftarrow\theta\), and \(v_{\tilde{\theta}}\) is held fixed for the following \(C\) training updates.
Given a training pair \((x_0,y)\), each update follows the three stages.

First, in \emph{surrogate sampling}, the static sampler starts from the LF field \(x_0\) and generates
\(
\tilde{x}_1
=
\mathrm{Sample}_{\tilde{\theta}}(x_0\mid y,0),
\)
which serves as the current surrogate HF endpoint in place of an unavailable full-field HF label.
The bridge time \(\tau\in[0,1]\) is discretized on a uniform grid
\(\tau_t=t\cdot s\) for \(t=0,\ldots,T\), with \(s=1/T\).
At each sampling step, the controlled drift is advanced, Brownian noise is injected, and the resulting state is projected by
\(\Pi_{\mathcal C_y}\).
Consequently, every discretized state remains in \(\mathcal C_y\).
Following common inpainting practice, the velocity network receives the projected state
\citep{liu2023i2sb,Lugmayr_2022_CVPR}
together with the corresponding bridge-time embedding.
Second, in \emph{bridge-state construction}, we sample
\(\tau\sim\mathrm{Unif}([0,1])\) and draw
\(\tilde{x}_\tau\) from the Brownian-bridge closed form in Eq.~\eqref{eq:bb_conditional_compact} using the endpoint pair
\((x_0,\tilde{x}_1)\).
The sampled bridge state is then projected as
\(
\tilde{x}_\tau
\leftarrow
\Pi_{\mathcal C_y}\!\left(
    \mathrm{BrownianBridge}(x_0,\tilde{x}_1,\tau)
\right),
\)
so that the trainable transport starts from an observation-consistent intermediate state.
Third, in \emph{physics-residual training}, the trainable sampler evolves this bridge state to
\(
\hat{x}_1
=
\mathrm{Sample}_{\theta}(\tilde{x}_\tau\mid y,\tau),
\)
again enforcing hard projection throughout the trajectory.
The parameters \(\theta\) are then updated by minimizing Eq.~\eqref{eq:method:pde_loss}.
Together, these stages realize an amortized auxiliary inverse-PDE solve:
the Brownian bridge provides the LF-to-HF transport geometry, hard projection enforces the sparse observations, and the PDE residual supplies the implicit endpoint energy without requiring empirical full-field HF endpoints.
This energy-defined endpoint interpretation is aligned with data-to-energy SB learning
\citep{tamogashev2026datatoenergy}
and is analogous to simulation-based Bayesian learning via diffusion bridges
\citep{park2024stochastic}.
Algorithm~\ref{alg:PICSB} summarizes the three stages together with the periodic refresh of the static sampler.

\begin{algorithm}[t]
\caption{Training procedure with iterative \(\tilde{\pi}_1\) refresh}
\label{alg:PICSB}
\small
\noindent\textbf{Input:}
\(D_{\mathrm{train}}\);
refresh period \(C\) and rounds \(L\);
trainable velocity \(v_{\theta^{(0)}}\);
projection \(\Pi_{\mathcal{C}_y}\);
sampler \(\mathrm{Sample}_\theta(\cdot)\) in Algorithm~\ref{alg:sample-theta};
\(\mathrm{BrownianBridge}(\cdot)\);
residual operator \(\mathcal{R}_h(\cdot)\).\\[2pt]
\textbf{Output:} trained parameters \(\theta\).\\[4pt]
\noindent\textbf{for} \(k = 0,1,\dots,L-1\) \textbf{do}\\
\hspace*{1.0em}\textbf{(0) Refresh static model:}
\(\tilde{\theta} \leftarrow \theta\).\\
\hspace*{1.0em}\textbf{for} \(j = 0,1,\dots,C-1\) \textbf{do}\\
\hspace*{2.0em}\textbf{(1) Surrogate endpoint:}
\(\tilde{x}_1 \leftarrow \mathrm{Sample}_{\tilde{\theta}}(x_0\,|\,y, 0) \quad \text{where} \,\, (x_0,y)\sim D_{\mathrm{train}}\).\\
\hspace*{2.0em}\textbf{(2) Bridge state:}
\(\tau\sim\mathrm{Unif}([0,1])\) and \(\tilde{x}_\tau\leftarrow \Pi_{\mathcal{C}_y}\big(\mathrm{BrownianBridge}(x_0, \tilde{x}_1, \tau)\big)\).\\
\hspace*{2.0em}\textbf{(3) Loss (Eq.~\eqref{eq:method:pde_loss}) and update:}
\(\theta \leftarrow \theta - \eta\nabla_{\theta}\mathcal{L}(\theta)
\quad \text{where}\,\, \hat{x}_1 = \mathrm{Sample}_{\theta}(\tilde{x}_\tau\,|\,y,\tau)\).\\
\hspace*{1.0em}\textbf{end for}\\
\noindent\textbf{end for}
\end{algorithm}

\section{Experiments}\label{sec:ex}

\vspace{-0.5em}
We evaluate PICSB on four PDE benchmarks---Burgers, Darcy flow, Gray--Scott, and Kolmogorov flow---in a multi-fidelity DA setting where sparse HF measurements are used both for sparse training supervision and for inference-time conditioning. 
The benchmarks include fluid-related systems, namely Burgers, Darcy flow, and Kolmogorov flow, as well as the Gray--Scott reaction--diffusion system. They span steady-state and time-dependent systems: Darcy flow is steady-state, while the remaining benchmarks are dynamic; they also cover both scalar HF fields and the multi-channel concentration field in Gray--Scott.

To reflect practical multi-fidelity settings, we consider two LF acquisition scenarios. For Burgers, the HF and LF simulations are initialized from the same analytical initial condition but evolved independently, with the LF field generated using a coarser-grid solver and mismatched viscosity, whereas for Darcy flow, Gray--Scott, and Kolmogorov flow, the LF input is an observation-based coarse proxy constructed from sparse HF measurements using nearest-neighbor or bicubic interpolation, depending on the benchmark.
We randomly resample \(10\%\) of the relevant grid locations for each instance and use the corresponding HF values as the only HF information available to the model, with full HF fields reserved solely for evaluation. 
For Burgers, Darcy flow, and Gray--Scott, we use \(32\) training fields and evaluate on \(40\) held-out test fields. 
For Kolmogorov flow, where the available trajectories are limited, we use a \(5\)-fold protocol to secure the same evaluation scale: each fold trains on \(32\) fields and tests on the remaining \(8\) fields, and final statistics are aggregated over all \(40\) held-out evaluations.
Detailed PDE definitions, solution-field sources, LF construction procedures, implementation details, and visualization settings are deferred to Appendix~\ref{app:sec:ex}.

\subsection{Reference Baselines and Evaluation Metrics}
\vspace{-0.5em}
To the best of our knowledge, no off-the-shelf method is directly applicable to our setting, since a valid method must jointly support \emph{zero-shot} reconstruction, observation-conditioned \emph{DA} from sparse HF measurements, and training without full-field HF supervision; see Table~\ref{tab:comparison}. 
We therefore report two reference baselines to contextualize performance, and also include the error of the LF input itself as a reference for the initial LF--HF discrepancy.
\textit{PINNs} relax the zero-shot requirement and serve as a strong per-instance optimization reference for assessing the difficulty of each benchmark. 
\textit{Guidance} is the closest amortized comparison: we adapt a PDE-guided diffusion sampler to the same sparse supervision regime by replacing HF prior training with LF pretraining.

\vspace{-0.5em}
\paragraph{Physics-Informed Neural Networks (\textit{PINNs})}
When the \emph{zero-shot} requirement is removed, PINNs provide a natural re-optimization-based DA reference. 
We parameterize the unknown solution field that takes spatio-temporal coordinates as input. 
For each test instance, the network is optimized at inference time using a data residual on sparse HF observations and a PDE residual evaluated at collocation points. 
Thus, PINNs are reported as an infeasible \emph{zero-shot} reference.

\vspace{-0.5em}
\paragraph{Guidance-based sampling (\textit{Guidance})}
Guidance is adapted from PDE-guided diffusion sampling~\citep{huang2024diffusionpde}. 
While the original formulation learns a generative prior from full-grid HF fields, we pretrain the prior only on full-grid LF fields. 
At inference time, the LF-pretrained sampler is steered by gradient-based guidance from two losses: a data residual on the sparse HF observations and a PDE residual over the reconstructed field. 
This baseline satisfies the same sparse-HF information constraint as PICSB and provides the closest zero-shot DA comparison.

\begin{table*}[t]
\captionsetup{skip=0.05in}
\caption{
Main quantitative comparison over 40 test samples. 
Error entries report 95\% confidence intervals, and wall-time entries report average per-sample inference time. 
Lower is better; \textbf{bold} marks the lowest mean, and \(\dagger\) denotes CI overlap with the best method.
}
\label{tab:ci_table}
\centering
\begin{scriptsize}
\resizebox{\textwidth}{!}{
\begin{tabular}{llcccc}
\toprule
Dataset & Metric & \textit{LF} & \textit{PINNs} & \textit{Guidance} & PICSB \\
\midrule







\multirow{5}{*}{Burgers} 
& RMSE (\(\times10^{-2}\))
& [9.576, 11.183]	
& [1.893, 2.842]\(^\dagger\)
& [2.969, 4.102]
& \textbf{[2.223, 2.381]} \\

& MAE (\(\times10^{-3}\))
& [78.461, 90.933]
& [14.110, 21.402]
& [17.417, 24.602]
& \textbf{[5.752, 6.798]} \\

& RelError (\%)
& [70.426, 71.428]
& [12.445, 22.565]\(^\dagger\)
& [20.177, 25.015]
& \textbf{[14.889, 19.700]} \\

& fRMSE (\(\times10^{-4}\))
& [9.959, 11.889]
& [3.463, 5.093]
& [4.520, 6.298]
& \textbf{[3.069, 3.453]} \\

& Average wall-time (seconds)
& -
& 4,068
& 155
& \textbf{0.026} \\

\midrule







\multirow{5}{*}{Darcy}
& RMSE (\(\times10^{-3}\))
& [1.808, 1.845]
& [2.496, 3.212]
& [1.519, 1.551]
& \textbf{[0.121, 0.129]} \\

& MAE (\(\times10^{-3}\))
& [1.296, 1.311]
& [1.998, 2.575]
& [1.093, 1.103]
& \textbf{[0.089, 0.094]} \\

& RelError (\%)
& [4.405, 4.451]
& [6.048, 7.833]
& [3.699, 3.745]
& \textbf{[0.293, 0.313]} \\

& fRMSE (\(\times10^{-4}\))
& [0.199, 0.205]
& [0.384, 0.518]
& [0.178, 0.185]
& \textbf{[0.018, 0.020]} \\

& Average wall-time (seconds)
& -
& 4,647
& 141
& \textbf{0.013} \\

\midrule

\multirow{5}{*}{Gray--Scott}
& RMSE (\(\times10^{-1}\))
& [1.159, 1.287]
& \textbf{[0.273, 0.362]}
& [1.240, 1.370]
& [0.330, 0.370]\(^\dagger\) \\

& MAE (\(\times10^{-2}\))
& [5.870, 6.990]
& [1.109, 1.667]
& [6.598, 7.689]
& \textbf{[0.859, 1.061]} \\

& RelError (\%)
& [17.521, 19.740]
& \textbf{[4.129, 5.534]}
& [18.736, 21.028]
& [4.989, 5.693]\(^\dagger\) \\

& fRMSE (\(\times10^{-3}\))
& [0.537, 0.600]
& \textbf{[0.167, 0.312]}
& [1.022, 1.142]
& [0.219, 0.268]\(^\dagger\) \\

& Average wall-time (seconds)
& -
& 2,557
& 498
& \textbf{0.109} \\

\midrule

\multirow{5}{*}{Kolmogorov}
& RMSE (\(\times10^{0}\))
& [0.668, 0.729]
& [2.971, 3.217]
& [0.547, 0.587]
& \textbf{[0.196, 0.220]} \\

& MAE (\(\times10^{0}\))
& [0.286, 0.319]
& [2.344, 2.555]
& [0.312, 0.338]
& \textbf{[0.109, 0.124]} \\

& RelError (\%)
& [13.497, 14.433]
& [60.590, 63.026]
& [11.043, 11.665]
& \textbf{[3.972, 4.299]} \\

& fRMSE (\(\times10^{-2}\))
& [0.159, 0.173]
& [3.387, 3.695]
& [0.232, 0.247]
& \textbf{[0.118, 0.144]} \\

& Average wall-time (seconds)
& -
& 9,527
& 787
& \textbf{0.028} \\

\bottomrule
\end{tabular}}
\end{scriptsize}
\end{table*}

\vspace{-0.5em}
\paragraph{Evaluation metrics} 
Following the evaluation protocol of RealPDEBench~\citep{hu2026realpdebench}, we assess reconstruction quality using complementary spatial-domain and spectral-domain errors. 
In the spatial domain, we report RMSE, MAE, and relative \(L^2\) error between the prediction \(\hat{x}_1\) and the held-out HF reference \(x_1\). 
As a physics-oriented spectral metric, we report fRMSE, which measures reconstruction error in Fourier space and evaluates whether the recovered field preserves physically relevant spatial-frequency content. 
We additionally report average per-sample inference wall-time.

\subsection{Results}
\vspace{-0.5em}
Table~\ref{tab:ci_table} reports 95\% confidence intervals (CIs) for reconstruction and spectral errors together with average inference wall-time, while Figure~\ref{fig:experiment_vis} visualizes representative reconstructions. 
Additional experimental settings and analyses are reported in Appendix~\ref{app:sup_ex}.

\vspace{-0.5em}
\paragraph{Accuracy.}
PICSB achieves competitive overall reconstruction accuracy, obtaining the lowest mean errors on Burgers, Darcy flow and Kolmogorov flow, while remaining statistically competitive with the infeasible \emph{zero-shot} \textit{PINNs} reference on Gray--Scott. 
On Burgers, Darcy flow, and Kolmogorov flow, PICSB achieves lower mean errors than both \textit{PINNs} and \textit{Guidance} across all metrics, with non-overlapping improvements on Darcy flow and Kolmogorov flow and CI overlap with \textit{PINNs} on some Burgers metrics, showing consistent gains under sparse HF information.
On Gray--Scott, PICSB overlaps with \textit{PINNs} on these metrics and achieves the lowest MAE, whereas \textit{Guidance} is consistently worse.
The robustness study in Table~\ref{tab:challenge_setting} further shows that PICSB remains competitive under more challenging conditions.

\vspace{-0.5em}
\paragraph{Time Complexity.}
All wall-times were measured on a single NVIDIA RTX 4090 GPU, except for Gray--Scott, which was evaluated on a single NVIDIA H200 NVL GPU. 
The largest practical gain of PICSB is inference speed: it produces reconstructions in \(0.026\)--\(0.109\) seconds across benchmarks, compared with \(2{,}557\)--\(9{,}527\) seconds for per-instance \textit{PINNs} optimization and \(141\)--\(787\) seconds for residual-guided diffusion sampling. 
This orders-of-magnitude reduction enables time-critical DA, where reconstruction must fit within each physical update cycle.

\begin{figure*}[t!]
    \centering
    \includegraphics[width=1.\linewidth]{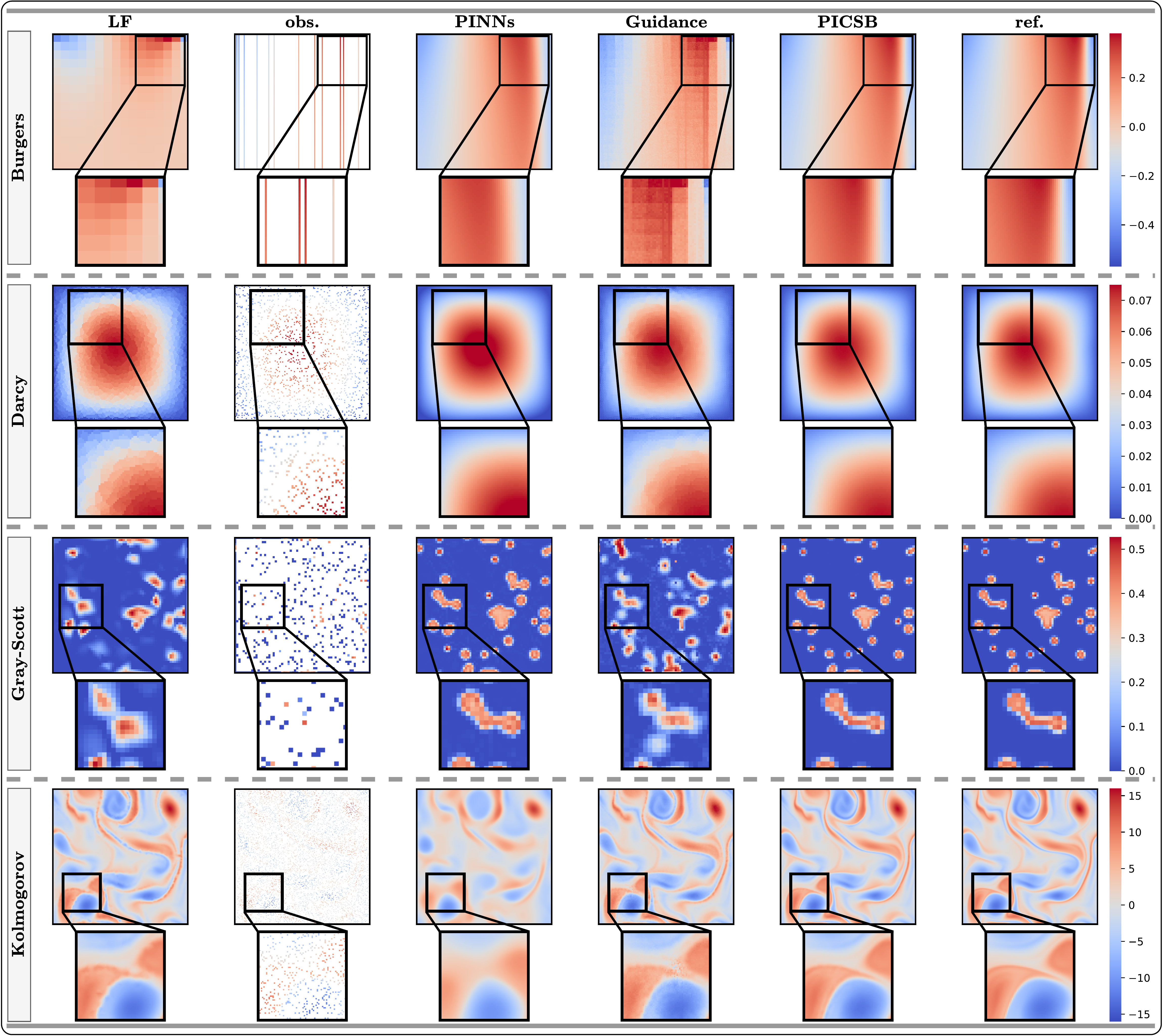}
    \vspace{-1.7em}
    \caption{
    Qualitative comparison of reconstructed fields. 
    We show representative test samples for Burgers and Darcy flow, and snapshots for Gray--Scott and Kolmogorov flow; LF, sparse HF observations (obs.), \textit{PINNs}, \textit{Guidance}, and PICSB are compared with the HF reference (ref.), with zoomed-in regions highlighting local reconstruction quality.
    }
    \label{fig:experiment_vis}
    \vspace{-1.5em}
\end{figure*}

\vspace{-0.5em}
\paragraph{Qualitative visualization.}
Figure~\ref{fig:experiment_vis} corroborates these results. 
Compared with the over-smoothed fields from \textit{PINNs} and the LF-biased outputs from \textit{Guidance}, PICSB more faithfully recovers the dominant structures across the benchmarks.
It also captures sharp transitions, localized patterns, and vortical structures while reducing interpolation and structural artifacts.

\section{Conclusion}\label{sec:con}

\vspace{-0.5em}
We studied sparsely supervised multi-fidelity DA for PDE-governed systems, where full-grid LF fields are available but HF supervision is limited to sparse observations.
We proposed PICSB, an auxiliary-energy-driven conditional Schr\"odinger bridge whose HF endpoint is defined implicitly by observation consistency, PDE residuals, and LF anchoring rather than empirical full-field HF targets, and learned through iterative surrogate-endpoint refresh.
Across diverse PDE benchmarks, PICSB achieves strong reconstruction accuracy with sub-second, orders-of-magnitude faster inference than per-instance optimization and inference-time-guided sampling, demonstrating the potential of amortized, guidance-free LF-to-HF transport for time-critical DA.

The present study is limited to settings in which the governing operators are known and sparse HF observations are available for training, with the current implementation restricted to direct coordinate-subset observations.
Extending PICSB to settings with unknown or partially known governing operators, nonlinear or indirect observation operators, severe measurement corruption, or data-free training driven solely by the governing PDEs is left for future work.

\begin{ack}
This research was supported in part by the National Research Foundation of Korea grant funded by the Korean government (MSIT) (Grant numbers RS-2022-NR068758 and RS-2024-00361377), in part by the Korea Evaluation Institute of Industrial Technology grant funded by the Korean government (MOTIR) (Grant number RS-2025-25458052), and in part by the National IT Industry Promotion Agency grant funded by the Korean government (MSIT) (Grant number RS-2026-25621650).
\end{ack}

{
\bibliographystyle{unsrtnat}
\bibliography{NeurIPS_reference}
}


\clearpage
\appendix

\section{Supplementary for Section~\ref{sec:met}}

\subsection{Proof for Theorem~\ref{thm:aux_endpoint_limit}}\label{app:thm:aux_endpoint_limit}

\paragraph{Theorem 4.1}
(Population endpoint selected by PICSB under an auxiliary energy)\textbf{.}
\textit{
Fix \((x_0,y)\).
Assume that hard projection is exact, \(\mathcal C_y\) is nonempty and closed, \(\mathcal R_h\) is continuous on \(\mathcal C_y\), and the residual floor \(\delta^\star(y;h)\) in Eq.~\eqref{eq:method:res_floor_set} is attained.
For fixed \(\lambda_{\mathrm{PDE}},\lambda_{\mathrm{LF}},\beta>0\), let \(\{\tilde{\pi}_1^{(k)}(\cdot|x_0,y)\}_{k\ge0}\) be the population terminal laws generated by the PICSB refresh map in Eq.~\eqref{eq:method:refresh_map}.
If this refresh sequence has vanishing terminal free-energy gap induced by \(\mathcal E_{\lambda,h}\), then \(\tilde{\pi}_1^{(k)}(\cdot|x_0,y)\) converges to \(\pi_{1+}^{\lambda,\beta}(\cdot|x_0,y)\) in KL divergence, and hence also in total variation and weakly.
Moreover, in the zero-temperature and residual-dominant limit \(\beta\to\infty\) and \(\lambda_{\mathrm{PDE}}/\lambda_{\mathrm{LF}}\to\infty\), the population PICSB
endpoint law concentrates on
\[
\mathcal X_{1+}^\star(x_0,y)
:=
\operatorname{argmin}_{x}
\left\{
\|x-x_0\|_2^2
\,:\,
x\in\mathcal P_{\mathcal R_h}^{\delta^\star,y}
\right\}.
\]
}

\begin{proof}\(\\\)
Fix \((x_0,y)\), and write
\[
F(x):=\|\mathcal R_h(x)\|_2^2,
\qquad
G(x):=\|x-x_0\|_2^2.
\]
By exact hard projection, all terminal laws generated by the population sampler are supported on \(\mathcal C_y\).

Recall the auxiliary energy in Eq.~\eqref{eq:pre:aux_energy}:
\[
\mathcal E_{\lambda,h}(x;x_0,y)
=
\frac{\lambda_{\mathrm{PDE}}}{2}F(x)
+
\frac{\lambda_{\mathrm{LF}}}{2}G(x),
\qquad x\in\mathcal C_y,
\]
and
\[
\mathcal E_{\lambda,h}(x;x_0,y)=+\infty,
\qquad x\notin\mathcal C_y.
\]
Thus every finite-energy endpoint belongs to \(\mathcal C_y\) and therefore satisfies \(\mathcal H(x)=y\).

Let \(\lambda_y\) be the Lebesgue measure over the unobserved coordinates with the observed coordinates fixed to \(y\), as in Eq.~\eqref{eq:pre:gibbs_law}.
For a probability measure \(\nu\) on \(\mathcal C_y\), write \(\nu\ll\lambda_y\) when \(\nu\) is absolutely continuous with respect to \(\lambda_y\), i.e., whenever \(\lambda_y(A)=0\) then \(\nu(A)=0\).
In this case, the Radon--Nikodym derivative
\[
\rho(x):=\frac{d\nu}{d\lambda_y}(x)
\]
exists, and we define
\[
\mathfrak F_{\lambda,\beta}^{x_0,y}(\nu)
:=
\int_{\mathcal C_y}
\mathcal E_{\lambda,h}(x;x_0,y)\,d\nu(x)
+
\beta^{-1}
\int_{\mathcal C_y}\log\rho(x)\,d\nu(x).
\]
If \(\nu\not\ll\lambda_y\), we set
\[
\mathfrak F_{\lambda,\beta}^{x_0,y}(\nu)=+\infty.
\]

The assumption that refresh training is asymptotically optimal for the terminal
free energy means
\[
\Delta_k
:=
\mathfrak F_{\lambda,\beta}^{x_0,y}
\!\left(\tilde{\pi}_1^{(k)}(\cdot|x_0,y)\right)
-
\inf_{\nu\in\mathcal P(\mathcal C_y)}
\mathfrak F_{\lambda,\beta}^{x_0,y}(\nu)
\longrightarrow 0.
\]

We now identify the minimizer of this free energy.
Define the normalizing constant
\[
Z_{\lambda,\beta}(x_0,y)
:=
\int_{\mathcal C_y}
\exp[-\beta\mathcal E_{\lambda,h}(x;x_0,y)]
\,d\lambda_y(x).
\]
Since \(\lambda_{\mathrm{LF}}>0\) and
\[
\mathcal E_{\lambda,h}(x;x_0,y)
\ge
\frac{\lambda_{\mathrm{LF}}}{2}\|x-x_0\|_2^2
\quad
\text{on } \mathcal C_y,
\]
the integrand is bounded above by a Gaussian density on the unobserved coordinates. 
Hence
\[
0<Z_{\lambda,\beta}(x_0,y)<\infty.
\]
The desirable endpoint law is therefore
\[
\pi_{1+}^{\lambda,\beta}(dx|x_0,y)
=
\frac{1}{Z_{\lambda,\beta}(x_0,y)}
\exp[-\beta\mathcal E_{\lambda,h}(x;x_0,y)]\,d\lambda_y(x).
\]

Let \(\nu\ll\lambda_y\) and let \(\rho=d\nu/d\lambda_y\). 
Then
\[
\frac{d\pi_{1+}^{\lambda,\beta}}{d\lambda_y}(x)
=
Z_{\lambda,\beta}(x_0,y)^{-1}
\exp[-\beta\mathcal E_{\lambda,h}(x;x_0,y)].
\]
Therefore,
\[
\begin{aligned}
\mathrm{KL}\!\left(
\nu
\,\middle\|\,
\pi_{1+}^{\lambda,\beta}(\cdot|x_0,y)
\right)
&=
\int_{\mathcal C_y}
\log
\left(
\frac{d\nu/d\lambda_y}
{d\pi_{1+}^{\lambda,\beta}/d\lambda_y}
\right)
d\nu
\\
&=
\int_{\mathcal C_y}\log\rho(x)\,d\nu
+
\beta
\int_{\mathcal C_y}
\mathcal E_{\lambda,h}(x;x_0,y)\,d\nu
+
\log Z_{\lambda,\beta}(x_0,y).
\end{aligned}
\]
Rearranging gives the Gibbs variational identity
\[
\mathfrak F_{\lambda,\beta}^{x_0,y}(\nu)
=
-\beta^{-1}\log Z_{\lambda,\beta}(x_0,y)
+
\beta^{-1}
\mathrm{KL}\!\left(
\nu
\,\middle\|\,
\pi_{1+}^{\lambda,\beta}(\cdot|x_0,y)
\right).
\]
The same identity holds in the extended sense when \(\nu\not\ll\lambda_y\), because both the free energy and the KL divergence are then \(+\infty\).

Since KL divergence is nonnegative and equals zero only when the two measures are identical, the unique minimizer of
\(\mathfrak F_{\lambda,\beta}^{x_0,y}\) is
\[
\pi_{1+}^{\lambda,\beta}(\cdot|x_0,y),
\]
and
\[
\inf_{\nu\in\mathcal P(\mathcal C_y)}
\mathfrak F_{\lambda,\beta}^{x_0,y}(\nu)
=
-\beta^{-1}\log Z_{\lambda,\beta}(x_0,y).
\]
Applying the variational identity to
\(\nu=\tilde{\pi}_1^{(k)}(\cdot|x_0,y)\) gives
\[
\mathrm{KL}\!\left(
\tilde{\pi}_1^{(k)}(\cdot|x_0,y)
\,\middle\|\,
\pi_{1+}^{\lambda,\beta}(\cdot|x_0,y)
\right)
=
\beta\,\Delta_k.
\]
Since \(\Delta_k\to0\), we obtain
\[
\mathrm{KL}\!\left(
\tilde{\pi}_1^{(k)}(\cdot|x_0,y)
\,\middle\|\,
\pi_{1+}^{\lambda,\beta}(\cdot|x_0,y)
\right)
\longrightarrow 0.
\]
Pinsker's inequality then yields
\[
\left\|
\tilde{\pi}_1^{(k)}(\cdot|x_0,y)
-
\pi_{1+}^{\lambda,\beta}(\cdot|x_0,y)
\right\|
\le
\sqrt{
\frac{1}{2}
\mathrm{KL}\!\left(
\tilde{\pi}_1^{(k)}
\,\middle\|\,
\pi_{1+}^{\lambda,\beta}
\right)
}
\longrightarrow 0.
\]
Total-variation convergence implies weak convergence. 
Therefore, under asymptotically optimal refresh training for the terminal free energy, the population terminal laws generated by Algorithm~\ref{alg:PICSB} satisfy
\[
\mathrm{KL}\!\left(
\tilde{\pi}_1^{(k)}(\cdot|x_0,y)
\,\middle\|\,
\pi_{1+}^{\lambda,\beta}(\cdot|x_0,y)
\right)
\to 0,
\qquad
\left\|
\tilde{\pi}_1^{(k)}(\cdot|x_0,y)
-
\pi_{1+}^{\lambda,\beta}(\cdot|x_0,y)
\right\|
\to 0,
\]
and hence
\[
\tilde{\pi}_1^{(k)}(\cdot|x_0,y)
\Rightarrow
\pi_{1+}^{\lambda,\beta}(\cdot|x_0,y).
\]

\vspace{1em}
It remains to characterize the zero-temperature and residual-dominant limit.
Let
\[
r:=\frac{\lambda_{\mathrm{PDE}}}{\lambda_{\mathrm{LF}}},
\qquad
a_\beta:=\frac{\beta\lambda_{\mathrm{LF}}}{2}.
\]
On \(\mathcal C_y\), the Gibbs density is proportional to
\[
\exp[-a_\beta\{rF(x)+G(x)\}].
\]
Let
\[
\delta^\star:=\delta^\star(y;h)
=
\min_{x\in\mathcal C_y}\|\mathcal R_h(x)\|_2.
\]
By the definition of the feasible physics set in Eq.~\eqref{eq:method:res_floor_set},
\[
\mathcal P_{\mathcal R_h}^{\delta^\star,y}
=
\left\{
x\in\mathcal C_y:
\|\mathcal R_h(x)\|_2\le\delta^\star
\right\}.
\]
Since \(\delta^\star\) is the minimum residual over \(\mathcal C_y\), this set is
equivalently
\[
\mathcal P_{\mathcal R_h}^{\delta^\star,y}
=
\left\{
x\in\mathcal C_y:
F(x)=(\delta^\star)^2
\right\}.
\]
It is nonempty by the attainment assumption, and it is closed because
\(\mathcal C_y\) is closed and \(F\) is continuous on \(\mathcal C_y\).
Define
\[
g^\star
:=
\min_{x\in\mathcal P_{\mathcal R_h}^{\delta^\star,y}}G(x).
\]
This minimum is attained. Indeed, choose any
\(x_p\in\mathcal P_{\mathcal R_h}^{\delta^\star,y}\). It is enough to minimize
\(G\) over
\[
\mathcal P_{\mathcal R_h}^{\delta^\star,y}
\cap
\{x:G(x)\le G(x_p)\},
\]
which is nonempty, closed, and bounded in finite dimension, hence compact.
Therefore
\[
\mathcal X_{1+}^\star(x_0,y)
=
\left\{
x\in\mathcal P_{\mathcal R_h}^{\delta^\star,y}
:
G(x)=g^\star
\right\}
\]
is nonempty and compact.

For fixed \(r>0\), subtracting the constant \(r(\delta^\star)^2\) from the
energy does not change the Gibbs law. Define
\[
\Phi_r(x)
:=
r\big(F(x)-(\delta^\star)^2\big)+G(x),
\qquad x\in\mathcal C_y.
\]
Then the Gibbs law is proportional to
\[
\exp[-a_\beta\Phi_r(x)]\,d\lambda_y(x).
\]
Because \(F(x)\ge(\delta^\star)^2\) on \(\mathcal C_y\),
\[
\Phi_r(x)\ge G(x).
\]
Thus \(\Phi_r\) is coercive on \(\mathcal C_y\). Since \(F\) and \(G\) are
continuous, \(\Phi_r\) is continuous. Hence
\[
S_r:=\operatorname{argmin}_{x\in\mathcal C_y}\Phi_r(x)
\]
is nonempty and compact.

We next show zero-temperature concentration for fixed \(r\).
Let \(U\) be any open neighborhood of \(S_r\) in the relative topology of
\(\mathcal C_y\), and let
\[
m_r:=\min_{x\in\mathcal C_y}\Phi_r(x).
\]
Because \(\Phi_r\) is continuous and coercive, its sublevel sets are compact.
Therefore
\[
\eta_U
:=
\inf_{x\in\mathcal C_y\setminus U}\Phi_r(x)-m_r
\]
is strictly positive.

Choose \(x_r^\star\in S_r\). By continuity of \(\Phi_r\), there is a relative
ball \(B\subset\mathcal C_y\) around \(x_r^\star\) such that
\[
\lambda_y(B)>0,
\qquad
\Phi_r(x)\le m_r+\eta_U/4
\quad
\text{for }x\in B.
\]
Hence the normalizing constant satisfies
\[
Z_{r,\beta}
:=
\int_{\mathcal C_y}e^{-a_\beta\Phi_r(x)}\,d\lambda_y(x)
\ge
\lambda_y(B)e^{-a_\beta(m_r+\eta_U/4)}.
\]
On the other hand, for \(x\in\mathcal C_y\setminus U\),
\[
\Phi_r(x)\ge m_r+\eta_U.
\]
Since \(\Phi_r\) is coercive,
\[
C_r
:=
\int_{\mathcal C_y}
\exp\!\left[-\frac{1}{2}(\Phi_r(x)-m_r)\right]d\lambda_y(x)
<\infty.
\]
For \(a_\beta\ge1\),
\[
\begin{aligned}
\int_{\mathcal C_y\setminus U}
e^{-a_\beta\Phi_r(x)}\,d\lambda_y(x)
&=
e^{-a_\beta m_r}
\int_{\mathcal C_y\setminus U}
e^{-a_\beta(\Phi_r(x)-m_r)}\,d\lambda_y(x)
\\
&\le
e^{-a_\beta m_r}
e^{-a_\beta\eta_U/2}
\int_{\mathcal C_y}
e^{-\frac{1}{2}(\Phi_r(x)-m_r)}\,d\lambda_y(x)
\\
&=
C_r e^{-a_\beta(m_r+\eta_U/2)}.
\end{aligned}
\]
Dividing this upper bound by the lower bound on \(Z_{r,\beta}\) gives
\[
\pi_{1+}^{\lambda,\beta}(\mathcal C_y\setminus U|x_0,y)
\le
\frac{C_r}{\lambda_y(B)}
\exp[-a_\beta\eta_U/4]
\longrightarrow 0
\qquad
\text{as }\beta\to\infty.
\]
Therefore, for fixed \(r\),
\[
\lim_{\beta\to\infty}
\pi_{1+}^{\lambda,\beta}(U|x_0,y)=1
\]
for every open neighborhood \(U\) of \(S_r\).

We now let \(r\to\infty\). Take any \(x_r\in S_r\) and any
\(x^\star\in\mathcal X_{1+}^\star(x_0,y)\). Since \(x_r\) minimizes
\(\Phi_r\),
\[
\Phi_r(x_r)\le \Phi_r(x^\star).
\]
Because \(x^\star\in\mathcal P_{\mathcal R_h}^{\delta^\star,y}\),
\[
F(x^\star)=(\delta^\star)^2.
\]
Thus
\[
\Phi_r(x^\star)=G(x^\star)=g^\star.
\]
Therefore
\[
r\big(F(x_r)-(\delta^\star)^2\big)+G(x_r)
\le
g^\star.
\]
Both terms on the left-hand side are nonnegative, so
\[
G(x_r)\le g^\star,
\qquad
0\le F(x_r)-(\delta^\star)^2\le \frac{g^\star}{r}.
\]
The bound \(G(x_r)\le g^\star\) places all such \(x_r\) in the compact sublevel
set
\[
\{x\in\mathcal C_y:G(x)\le g^\star\}.
\]
Hence, for any sequence \(r_\ell\to\infty\), there exists a subsequence,
still denoted \(r_\ell\), such that
\[
x_{r_\ell}\to\bar x\in\mathcal C_y.
\]
By continuity of \(F\),
\[
F(\bar x)
=
\lim_{\ell\to\infty}F(x_{r_\ell})
=
(\delta^\star)^2.
\]
Hence
\[
\bar x\in\mathcal P_{\mathcal R_h}^{\delta^\star,y}.
\]
By continuity of \(G\),
\[
G(\bar x)
=
\lim_{\ell\to\infty}G(x_{r_\ell})
\le
g^\star.
\]
Since \(g^\star\) is the minimum of \(G\) over
\(\mathcal P_{\mathcal R_h}^{\delta^\star,y}\), we must have
\[
G(\bar x)=g^\star.
\]
Thus
\[
\bar x\in\mathcal X_{1+}^\star(x_0,y).
\]
So every cluster point of minimizers \(x_r\in S_r\) lies in
\(\mathcal X_{1+}^\star(x_0,y)\).

Let \(U\) be any open neighborhood of
\(\mathcal X_{1+}^\star(x_0,y)\).
If \(S_r\) were not eventually contained in \(U\), then there would exist
\(r_\ell\to\infty\) and \(x_{r_\ell}\in S_{r_\ell}\setminus U\).
By the compactness argument above, a subsequence would converge to a point in
\(\mathcal X_{1+}^\star(x_0,y)\subset U\), contradicting
\(x_{r_\ell}\notin U\).
Hence there exists \(r_0>0\) such that
\[
S_r\subset U,
\qquad r\ge r_0.
\]
For every \(r\ge r_0\), the fixed \(r\) zero-temperature concentration result
gives
\[
\lim_{\beta\to\infty}
\pi_{1+}^{\lambda,\beta}(U|x_0,y)=1.
\]
Therefore
\[
\lim_{r\to\infty}
\lim_{\beta\to\infty}
\pi_{1+}^{\lambda,\beta}(U|x_0,y)=1.
\]

Finally, for fixed \((\lambda_{\mathrm{PDE}},\lambda_{\mathrm{LF}},\beta)\),
we already proved
\[
\tilde{\pi}_1^{(k)}(\cdot|x_0,y)
\to
\pi_{1+}^{\lambda,\beta}(\cdot|x_0,y)
\]
in total variation. Hence, for every measurable set \(U\),
\[
\left|
\tilde{\pi}_1^{(k)}(U|x_0,y)
-
\pi_{1+}^{\lambda,\beta}(U|x_0,y)
\right|
\le
\left\|
\tilde{\pi}_1^{(k)}(\cdot|x_0,y)
-
\pi_{1+}^{\lambda,\beta}(\cdot|x_0,y)
\right\|
\to 0.
\]
Combining this with the concentration of
\(\pi_{1+}^{\lambda,\beta}\), we obtain, for every open neighborhood \(U\) of
\(\mathcal X_{1+}^\star(x_0,y)\),
\[
\lim_{r\to\infty}
\lim_{\beta\to\infty}
\lim_{k\to\infty}
\tilde{\pi}_1^{(k)}(U|x_0,y)
=
1.
\]
This is the precise meaning of concentration on
\[
\mathcal X_{1+}^\star(x_0,y)
=
\underset{x\in\mathcal P_{\mathcal R_h}^{\delta^\star,y}}
{\operatorname{argmin}}\,
\|x-x_0\|_2^2.
\]

Since \(\mathcal P_{\mathcal R_h}^{\delta^\star,y}\subset\mathcal C_y\), every point in the limiting support satisfies \(\mathcal H(x)=y\), giving zero observation error. 
Since points in \(\mathcal P_{\mathcal R_h}^{\delta^\star,y}\) attain the residual floor, the PDE residual is minimized over \(\mathcal C_y\). 
Finally, the last minimization over \(G(x)=\|x-x_0\|_2^2\) shows that, when multiple observation-consistent residual minimizers exist, the limiting endpoint selects the one closest to the LF input \(x_0\).
\end{proof}

\clearpage

\subsection{Additional Theoretical Notes --- Stability under identifiability}\label{app:thm:equivalence_weak}
\begin{lemma}[Existence of a localization gap]\label{lem:beta_gap}
Fix \(y\) and define \(f(x):=\|\mathcal{R}_h(x)\|\).
Let \(\mathcal{K}\subseteq \mathcal{C}_y\) be a compact set and assume that \(f\) is continuous on \(\mathcal{K}\) and admits a unique minimizer: there exists a unique \(x_1^\star\in\mathcal{K}\) such that
\[
f(x_1^\star)=\min_{x\in\mathcal{K}} f(x)=\delta^\star_{\mathcal{K}}(y;h).
\]
For any \(\rho\in \big(0,\sup_{z\in\mathcal{K}}\|z-x_1^\star\|\big)\), define the neighborhood
\[
\mathcal{U}_\rho:=\{x\in\mathcal{K}:\|x-x_1^\star\|\le \rho\}.
\]
Then the quantity
\[
\bar\beta(\rho) \;:=\; \min_{x\in \mathcal{K}\setminus \mathcal{U}_\rho} \Big(f(x)-\delta^\star_{\mathcal{K}}(y;h)\Big)
\]
is well-defined and satisfies \(\bar\beta(\rho)>0\).
Consequently, for any \(0<\beta\le \bar\beta(\rho)\),
\[
x\in\mathcal{K}\setminus \mathcal{U}_\rho \quad\Longrightarrow\quad f(x)\ge \delta^\star_{\mathcal{K}}(y;h)+\beta.
\]
\end{lemma}

\vspace{0.5em}

\begin{proof}
\(\\\)
Since \(\mathcal{K}\) is compact and \(f\) is continuous, \(f\) attains its minimum on \(\mathcal{K}\) by the Weierstrass extreme value theorem, and by assumption the minimizer is unique at \(x_1^\star\).

The set \(\mathcal{K}\setminus \mathcal{U}_\rho\) is compact.
Indeed, the set \(\{x\in\mathbb{R}^n:\|x-x_1^\star\|\ge \rho\}\) is closed, and hence
\[
\mathcal{K}\setminus \mathcal{U}_\rho = \mathcal{K}\cap\{x:\|x-x_1^\star\|\ge \rho\}
\]
is a closed subset of the compact set \(\mathcal{K}\), and therefore compact.

Since \(f(x)\ge \delta^\star_{\mathcal{K}}(y;h)\) for all \(x\in\mathcal{K}\), we have \(f(x)-\delta^\star_{\mathcal{K}}(y;h)\ge 0\) on \(\mathcal{K}\setminus \mathcal{U}_\rho\).
Thus the continuous function \(x\mapsto f(x)-\delta^\star_{\mathcal{K}}(y;h)\) attains its minimum on \(\mathcal{K}\setminus \mathcal{U}_\rho\), so \(\bar\beta(\rho)\) is well-defined and satisfies \(\bar\beta(\rho)\ge 0\).

To show \(\bar\beta(\rho)>0\), assume for contradiction that \(\bar\beta(\rho)=0\).
Then by definition of the minimum, there exists \(\bar{x}\in\mathcal{K}\setminus \mathcal{U}_\rho\) such that \(f(\bar{x})=\delta^\star_{\mathcal{K}}(y;h)\).
Since \(\bar{x}\notin \mathcal{U}_\rho\), we have \(\bar{x}\neq x_1^\star\), which contradicts the uniqueness of the minimizer.
Hence \(\bar\beta(\rho)>0\).
\end{proof}

\vspace{1em}

\begin{proposition}[Stability under identifiability]\label{proposition:related1}
Fix \(y\) and assume the continuous PDE admits a unique observation-consistent solution \(x_1^\star\in\mathcal{C}_y\).
Assume that local identifiability holds in a neighborhood of \(x_1^\star\), with constants \(c\!>\!0\) and \(\eta(h)\), and that the residual sublevel sets are localized with constant \(\beta\!>\!0\).
Then any \(\hat{x}_1\!\in\!\mathcal{C}_y\) satisfying \(\|\mathcal{R}_h(\hat{x}_1)\|\le \delta\) for some 
\(\delta\le \delta^\star(y;h)+\beta\) obeys
\begin{equation}
\|\hat{x}_1-x_1^\star\| \le c(\delta-\delta^\star(y;h))+\eta(h).
\end{equation}
\end{proposition}

\vspace{1em}

\begin{proof}\(\\\)
Fix \(y\) and let \(x_1^\star\in\mathcal{C}_y\) denote the unique observation-consistent solution.
Recall the residual floor on \(\mathcal{C}_y\),
\[
\delta^\star(y;h)\;:=\;\min_{x\in\mathcal{C}_y}\|\mathcal{R}_h(x)\|.
\]
We now make the assumptions in the statement precise.

\begin{assumption}[Local identifiability]\label{asmp:local_id}
There exist \(\rho>0\), a constant \(c>0\), and a function \(\eta(h)\ge 0\) with \(\eta(h)\to 0\) as \(h\to 0\)
such that, for the neighborhood
\[
\mathcal{U}\;:=\;\{x\in\mathbb{R}^n:\|x-x_1^\star\|\le \rho\},
\]
the following holds for all \(x\in\mathcal{U}\cap\mathcal{C}_y\):
\begin{equation}
\|x-x_1^\star\|
\;\le\;
c\big(\|\mathcal{R}_h(x)\|-\delta^\star(y;h)\big)+\eta(h).
\label{eq:pf_A1}
\end{equation}
\end{assumption}

\paragraph{Deriving a localization constant on a compact relevant set.}
Let \(\mathcal{K}\subseteq \mathcal{C}_y\) be a compact set that contains \(\hat{x}_1\), and define
\[
\delta^\star_{\mathcal{K}}(y;h)\;:=\;\min_{x\in\mathcal{K}}\|\mathcal{R}_h(x)\|.
\]
Assume that \(x_1^\star\) is the unique minimizer of \(f(x):=\|\mathcal{R}_h(x)\|\) over \(\mathcal{K}\).
For the same \(\rho>0\) as above, Lemma~\ref{lem:beta_gap} yields a constant
\[
\beta\;:=\;\bar\beta(\rho)\;>\;0
\quad\text{such that}\quad
\inf_{x\in\mathcal{K}\setminus\mathcal{U}}\|\mathcal{R}_h(x)\|
\;\ge\;
\delta^\star_{\mathcal{K}}(y;h)+\beta.
\]
Equivalently, we may record this as the following localization inequality:
\begin{equation}
\inf_{x\in\mathcal{K}\setminus\mathcal{U}}\|\mathcal{R}_h(x)\|
\;\ge\;
\delta^\star_{\mathcal{K}}(y;h)+\beta.
\label{eq:pf_A2K}
\end{equation}

\vspace{0.5em}
Now let \(\hat{x}_1\in\mathcal{K}\) satisfy \(\|\mathcal{R}_h(\hat{x}_1)\|\le \delta\) for some
\[
\delta\;\le\;\delta^\star_{\mathcal{K}}(y;h)+\beta.
\]
We first show that \(\hat{x}_1\in\mathcal{U}\).
If \(\hat{x}_1\notin\mathcal{U}\), then \(\hat{x}_1\in\mathcal{K}\setminus\mathcal{U}\) and Eq.~\eqref{eq:pf_A2K} gives
\[
\|\mathcal{R}_h(\hat{x}_1)\|
\;\ge\;
\inf_{x\in\mathcal{K}\setminus\mathcal{U}}\|\mathcal{R}_h(x)\|
\;\ge\;
\delta^\star_{\mathcal{K}}(y;h)+\beta,
\]
contradicting \(\|\mathcal{R}_h(\hat{x}_1)\|\le \delta \le \delta^\star_{\mathcal{K}}(y;h)+\beta\).
Therefore \(\hat{x}_1\in\mathcal{U}\).

Since \(\hat{x}_1\in\mathcal{U}\cap\mathcal{C}_y\), we may apply Eq.~\eqref{eq:pf_A1} with \(x=\hat{x}_1\):
\[
\|\hat{x}_1-x_1^\star\|
\;\le\;
c\big(\|\mathcal{R}_h(\hat{x}_1)\|-\delta^\star(y;h)\big)+\eta(h).
\]
Using \(\|\mathcal{R}_h(\hat{x}_1)\|\le \delta\) yields
\[
\|\hat{x}_1-x_1^\star\|
\;\le\;
c(\delta-\delta^\star(y;h))+\eta(h),
\]
which proves the claim.
\end{proof}

\vspace{1em}

The intuition regarding the upper bound on the difference between \(\hat{x}_1\) and \(x^*_1\) is as follows. First, the upper bound tends to increase as the fidelity of the target PDE becomes higher (HF). Conversely, the bound decreases as the spatial resolution is refined or the time interval narrows. We provide a theoretical verification of these intuitions specifically within the context of the Burgers' equation.

\begin{proposition}[Error Bound Stability of 1D Burgers' Equation]\label{burgersprop4.1}
Suppose the observation constraint set \(\mathcal{C}_y\) ensures the unique existence of a solution. Let \(x^*\) be the true state and \(\hat{x}\) be the approximate state. Let \(\delta'' = \|\mathcal{R}(\hat{x}) - \mathcal{R}(x^*)\|\) denote the residual difference. Then, there exists a state error bound function \(\epsilon(\delta'', \text{PDE}_{\text{fidelity}}, h)\) such that:
\begin{equation}
\| \hat{x} - x^* \| < \epsilon(\delta'', \text{PDE}_{\text{fidelity}}, h)
\end{equation}
where \(h\) is the scaling of grid. The error bound \(\epsilon\) satisfies:
\begin{itemize}
    \item \(\frac{\partial \epsilon}{\partial \delta''} > 0\): The bound is monotonically increasing with respect to the residual difference.
    \item \(\epsilon\) increases as the PDE exhibits higher frequency component
    \item \(\frac{\partial \epsilon}{\partial h} < 0\): As the grid resolution increases (\(h \to 0\)), the bound \(\epsilon\) tends to decrease.
\end{itemize}
\end{proposition}

\vspace{1em}

\begin{proof}\(\\\)
To establish the stability-based error bound for the 1D Burgers' equation, we analyze the relationship between the state error and the residual difference through the following four steps.

\paragraph{Residual and Jacobian Formulation} 
We consider the discretized residual operator \(\mathcal{R}(x)\) on a spatio-temporal grid \((i, j)\):
\begin{equation}
\mathcal{R}(x_{i,j}) = \frac{x_{i,j} - x_{i,j-1}}{\Delta t} + x_{i,j} \frac{x_{i+1,j} - x_{i-1,j}}{2h} - \nu \frac{x_{i+1,j} - 2x_{i,j} + x_{i-1,j}}{h^2}
\end{equation}
Linearizing \(\mathcal{R}\) around the reference state yields the Jacobian matrix \(J = \frac{1}{\Delta t} I + A + D\), where \(D\) is the symmetric diffusion matrix with \(D_{i,i} = \frac{2\nu}{h^2}\) and \(D_{i,i\pm1} = -\frac{\nu}{h^2}\), and \(A\) is the non-symmetric advection matrix capturing local transport.

\paragraph{Spectral Analysis of Spatial Operators} 
We estimate the lower bound of the minimum eigenvalue of the spatial operator \(M = A + D\) using the Gershgorin Circle Theorem. For each row \(i\), the Gershgorin disc is centered at \(M_{i,i} = \frac{2\nu}{h^2} + \frac{x_{i+1} - x_{i-1}}{2h}\) with radius \(R_i = |-\frac{\nu}{h^2} + \frac{x_i}{2h}| + |-\frac{\nu}{h^2} - \frac{x_i}{2h}|\). Under the refinement condition \(h < 2\nu/|x_i|\), the radius reduces to \(R_i = 2\nu/h^2\). By incorporating the global spectral property of the Laplacian \(\nu (\frac{\pi}{L})^2\), we obtain a refined lower bound:
\begin{equation}
\lambda_{\min}(A + D) \approx \nu \left( \frac{\pi}{L} \right)^2 + \min_i \left( \frac{x_{i+1,j} - x_{i-1,j}}{2h} \right)
\end{equation}

\paragraph{Norm Inequality and Stability Bound} 
Let \(e = \hat{x} - x^*\) be the state error. Applying a first-order Taylor expansion, the residual difference \(\delta'' = \| \mathcal{R}(\hat{x}) - \mathcal{R}(x^*) \|\) is approximated by \(\| J(x^*) \cdot e \|\). Using the matrix norm property \(\| J e \| \geq \sigma_{\min}(J) \| e \|\) and approximating the minimum singular value as \(\sigma_{\min}(J) \approx \frac{1}{\Delta t} + \lambda_{\min}(A+D)\), we derive the stability bound.

\paragraph{Final Error Bound and Interpretation} 
Combining the above results, we arrive at the final upper bound for the state error:
\begin{equation}
\| \hat{x} - x^* \| \leq \frac{\delta''}{\frac{1}{\Delta t} + \nu \left( \frac{\pi}{L} \right)^2 + \min_i \left( \frac{x_{i+1,j} - x_{i-1,j}}{2h} \right)} 
\end{equation}
This bound validates our intuition: as \(\nu \to 0\) (High Fidelity), the diffusive term vanishes and the bound increases. Conversely, as \(\Delta t \to 0\) (Higher Resolution), the temporal term \(\frac{1}{\Delta t}\) dominates, effectively suppressing the state error.
\end{proof}

\clearpage

\subsection{Additional Theoretical Notes --- Self-consistency at a fixed point}\label{app:thm:convergence}
\begin{lemma}[Hard observation feasibility]\label{lem:hard_feasibility_fp}
Suppose \(\mathrm{Sample}_{\theta}\) applies the coordinate projection \(\Pi_{\mathcal{C}_y}\) at every discretized step.
Then its terminal output \(x_1=\mathrm{Sample}_{\theta}(x_0\,|\,y,0)\) satisfies \(x_1\in\mathcal{C}_y\), i.e.,
\(\mathcal{H}(x_1)=y\) (if \(\mathrm{Sample}_{\theta}\) is randomized).
\end{lemma}

\begin{proof}
At each discretized step, the projection \(\Pi_{\mathcal{C}_y}\) overwrites the observed coordinates so that the state
lies in \(\mathcal{C}_y=\{x:\mathcal{H}(x)=y\}\). By induction over the sampling steps, all intermediate states,
and in particular the terminal state, belong to \(\mathcal{C}_y\). Hence \(\mathcal{H}(x_1)=y\).
\end{proof}

\vspace{0.75em}

\begin{lemma}[Existence of an inner-loop minimizer]\label{lem:innerloop_argmin}
Fix \(y\) and a reference law \(\tilde{\pi}(\cdot\,|\,y)\).
Assume \(J_{\tilde{\pi}}:\Theta\!\to\!\mathbb{R}\) is lower semi-continuous and bounded below, where
\[
J_{\tilde{\pi}}(\theta) := \mathbb{E}\big[\|\mathcal{R}_h(\hat{x}_1(\theta))\|\big],
\qquad
\hat{x}_1(\theta)=\mathrm{Sample}_{\theta}(\tilde{x}_\tau\,|\,y,\tau).
\]
Assume moreover that there exists \(\theta_0\in\Theta\) such that the sublevel set
\[
\mathcal{S} := \{\theta\in\Theta:\ J_{\tilde{\pi}}(\theta)\le J_{\tilde{\pi}}(\theta_0)\}
\]
is nonempty and compact.
Then \(\mathrm{argmin}_{\theta\in\Theta} J_{\tilde{\pi}}(\theta)\) is nonempty.
\end{lemma}

\vspace{1em}

\begin{proof}
We first restate the assumptions in a form used below.

\emph{Lower semi-continuity:} for every \(\theta\in\Theta\) and every sequence \(\{\theta_k\}\subset\Theta\)
with \(\theta_k\to\theta\), we have
\begin{equation}\label{eq:lsc_def}
J_{\tilde{\pi}}(\theta)\le \liminf_{k\to\infty} J_{\tilde{\pi}}(\theta_k).
\end{equation}

\vspace{1.0em}

\emph{Bounded below:} there exists \(L\!\in\!\mathbb{R}\) such that
\begin{equation}\label{eq:bb_def}
J_{\tilde{\pi}}(\theta)\ge L
\qquad\text{for all }\theta\in\Theta.
\end{equation}
Hence \(m:=\inf_{\theta\in\Theta}J_{\tilde{\pi}}(\theta)\) is finite.

\vspace{1.0em}

By definition of \(m\), there exists a sequence \(\{\theta_n\}\subset\Theta\) such that
\begin{equation}\label{eq:min_seq}
J_{\tilde{\pi}}(\theta_n)\to m \qquad\text{as }n\to\infty.
\end{equation}
Since \(J_{\tilde{\pi}}(\theta_0)\ge m\), the convergence Eq.~\eqref{eq:min_seq} implies that there exists \(N\in\mathbb{N}\)
such that for all \(n\ge N\),
\[
J_{\tilde{\pi}}(\theta_n)\le J_{\tilde{\pi}}(\theta_0),
\]
hence \(\theta_n\in\mathcal{S}\) for all \(n\ge N\) by definition of \(\mathcal{S}\).

\vspace{0.5em}

Because \(\mathcal{S}\) is compact, the tail sequence \(\{\theta_n\}_{n\ge N}\subset\mathcal{S}\) admits a convergent
subsequence: there exist indices \(n_k\ge N\) and \(\theta^\star\in\mathcal{S}\) such that
\[
\theta_{n_k}\to \theta^\star \qquad\text{as }k\to\infty.
\]
Applying lower semi-continuity Eq.~\eqref{eq:lsc_def} to the convergent subsequence gives
\[
J_{\tilde{\pi}}(\theta^\star)\le \liminf_{k\to\infty} J_{\tilde{\pi}}(\theta_{n_k}).
\]
Since \(J_{\tilde{\pi}}(\theta_n)\to m\) along the full sequence, the subsequence also satisfies
\(J_{\tilde{\pi}}(\theta_{n_k})\to m\), hence \(\liminf_{k\to\infty}J_{\tilde{\pi}}(\theta_{n_k})=m\).
Therefore,
\[
J_{\tilde{\pi}}(\theta^\star)\le m.
\]
On the other hand, by definition of \(m\) as the infimum over \(\Theta\), we have \(m\le J_{\tilde{\pi}}(\theta^\star)\).
Combining the two inequalities yields \(J_{\tilde{\pi}}(\theta^\star)=m\).

Thus \(\theta^\star\) attains the minimum of \(J_{\tilde{\pi}}\) over \(\Theta\), i.e.,
\(\theta^\star\in \mathrm{argmin}_{\theta\in\Theta} J_{\tilde{\pi}}(\theta)\), proving that
\(\mathrm{argmin}_{\theta\in\Theta} J_{\tilde{\pi}}(\theta)\neq\emptyset\).
\end{proof}

\vspace{0.75em}

\begin{proposition}[Self-consistency at a fixed point]
Consider the sequence \(\{\tilde{\pi}_1^{(k)}\}_{k\ge 0}\) generated by Algorithm~\ref{alg:PICSB}. 
Assume \(\mathcal{R}_h\) is Lipschitz on the relevant support and the model class for \(v_\theta\) is sufficiently expressive. Assume further that each refresh round performs exact inner-loop optimization. 
If \(\tilde{\pi}_1^{(k)}\!\equiv\!\tilde{\pi}_1^{(k+1)}\!\equiv\!\tilde{\pi}_1^\infty\), then samples \(\tilde{x}_1\!\sim\!\tilde{\pi}_1^\infty(\cdot\,|\,y)\) satisfy 
\begin{equation}\label{eq:prop42_conc} 
\mathcal{H}(\tilde{x}_1)\!=\!y \qquad\text{and}\qquad\mathbb{E}\Big[\|\mathcal{R}_h(\tilde{x}_1)\|\Big] \!=\! \inf_{\theta}\mathbb{E}\Big[\|\mathcal{R}_h(\hat{x}_1)\|\Big],
\end{equation} 
where \(\hat{x}_1\!=\!\mathrm{Sample}_{\theta}(\tilde{x}_\tau\,|\,y,\tau)\) is defined using the fixed reference \(\tilde{\pi}_1^\infty(\cdot\,|\,y)\) as in Algorithm~\ref{alg:PICSB}, with expectation over the sampling randomness.
\end{proposition}

\begin{proof}\(\\\)
At refresh round \(k+1\), the surrogate terminal law \(\tilde{\pi}_1^{(k+1)}(\cdot\,|\,y)\) is, by definition, the distribution of \(\tilde{x}_1=\mathrm{Sample}_{\tilde{\theta}}(x_0\,|\,y,0)\) under hard conditioning (Algorithm~\ref{alg:PICSB}). 
Lemma~\ref{lem:hard_feasibility_fp} therefore implies \(\mathcal{H}(\tilde{x}_1)=y\) for \(\tilde{x}_1\sim\tilde{\pi}_1^{(k+1)}(\cdot\,|\,y)\).
If \(\tilde{\pi}_1^{(k+1)}\equiv\tilde{\pi}_1^\infty\), this gives the first claim in~Eq.~\eqref{eq:prop42_conc}.

\vspace{1.0em}

Fix the reference law \(\tilde{\pi}_1^\infty(\cdot\,|\,y)\) and define the fixed-reference objective
\[
J_\infty(\theta)\;:=\;J_{\tilde{\pi}_1^\infty}(\theta)
\;=\;
\mathbb{E}\Big[\|\mathcal{R}_h(\hat{x}_1)\|\Big],
\qquad
\hat{x}_1=\mathrm{Sample}_{\theta}(\tilde{x}_\tau\,|\,y,\tau),
\]
where the expectation is taken over all randomness in the sampling procedure under the fixed reference
\(\tilde{x}_1\sim\tilde{\pi}_1^\infty(\cdot\,|\,y)\).

By Lemma~\ref{lem:innerloop_argmin} (applied with \(\tilde{\pi}=\tilde{\pi}_1^\infty\)), the set
\(\mathrm{argmin}_{\theta\in\Theta}J_\infty(\theta)\) is nonempty; let \(\theta^\star\in\mathrm{argmin}_{\theta\in\Theta}J_\infty(\theta)\).
This ensures that, for the fixed reference law \(\tilde{\pi}_1^\infty(\cdot\,|\,y)\), the refresh step of Algorithm~\ref{alg:PICSB} is well-defined and induces a deterministic update of the surrogate terminal law.
A fixed point corresponds to a reference law that is invariant under this update.

Now suppose we are at a fixed point \(\tilde{\pi}_1^{(k)}\equiv\tilde{\pi}_1^{(k+1)}\equiv\tilde{\pi}_1^\infty\).
Then the round-\(k\) objective coincides with \(J_\infty\) by construction of Algorithm~\ref{alg:PICSB}.
By the exact inner-loop optimization assumption, the update selects a global minimizer of the round-\(k\) objective,
hence we may take the refreshed parameters \(\theta^{(k+1)}\) to satisfy
\[
\theta^{(k+1)}\in \mathrm{argmin}_{\theta\in\Theta}J_\infty(\theta).
\]
Therefore,
\[
\mathbb{E}\Big[\|\mathcal{R}_h(\tilde{x}_1)\|\Big]
\;=\;
J_\infty(\theta^{(k+1)})
\;=\;
\inf_{\theta\in\Theta}J_\infty(\theta)
\;=\;
\inf_{\theta\in\Theta}\mathbb{E}\Big[\|\mathcal{R}_h(\hat{x}_1)\|\Big],
\]
which is exactly the second claim in~Eq.~\eqref{eq:prop42_conc}.
\end{proof}

\clearpage
\section{Experimental Setup}\label{app:sec:ex}
\raggedbottom
\setlength{\intextsep}{8pt plus 1pt minus 1pt}
\setlength{\textfloatsep}{12pt plus 2pt minus 2pt}
\setlength{\floatsep}{12pt plus 2pt minus 2pt}
\captionsetup[figure]{font=small,skip=5pt}
\captionsetup[table]{font=small,skip=5pt}
\setlength{\abovedisplayskip}{7pt plus 2pt minus 2pt}
\setlength{\belowdisplayskip}{7pt plus 2pt minus 2pt}
\setlength{\abovedisplayshortskip}{4pt plus 1pt minus 1pt}
\setlength{\belowdisplayshortskip}{5pt plus 1pt minus 1pt}
\makeatletter
\setlength{\@fptop}{0pt}
\setlength{\@fpsep}{12pt}
\setlength{\@fpbot}{0pt plus 1fil}
\makeatother
\let\AppendixOriginalParagraph\paragraph
\renewcommand{\paragraph}{\Needspace{4\baselineskip}\AppendixOriginalParagraph}
\let\AppendixOriginalSubsection\subsection
\renewcommand{\subsection}{\FloatBarrier\Needspace{7\baselineskip}\AppendixOriginalSubsection}
\let\AppendixOriginalSubsubsection\subsubsection
\renewcommand{\subsubsection}{\FloatBarrier\Needspace{6\baselineskip}\AppendixOriginalSubsubsection}

\subsection{Datasets Description}\label{app:subsec:DD}
PDEs are defined over a spatial coordinate \(\xi\) and a physical time variable \(\gamma\).
For dynamic systems, \(\xi\) denotes the spatial coordinate and \(\gamma\) denotes physical time.
For steady-state problems, only the spatial coordinate \(\xi\) is used.
Across benchmarks, the HF reference denotes the target field to be reconstructed: the scalar solution field \(u\) for Burgers, Darcy, and 1D Diffusion--Reaction, the two-channel concentration field \((A,B)\) for Gray--Scott, and the vorticity field \(\omega\) for Kolmogorov.

\paragraph{Dynamic 1D Burgers' Equation.}
We consider the one-dimensional Burgers' equation on
\(\xi\!\in\!(0,1)\) over the time horizon \(\gamma\!\in\!(0,1]\),
\begin{align*}
\partial_\gamma u(\xi,\gamma)
+ u(\xi,\gamma)\nabla u(\xi,\gamma)
= \nu\,\nabla^2 u(\xi,\gamma),
\end{align*}
with initial condition \(u(\xi,0)=u_0(\xi)\).

We adapt the Burgers data-generation code used in DiffusionPDE~\citep{huang2024diffusionpde}, with the solver configurations modified to construct the HF and LF settings described below.
For the main Burgers benchmark, we construct the LF input using an independent numerical simulation rather than an observation-derived interpolation proxy.
The HF and LF simulations are initialized from the same analytical initial condition, discretized on their respective grids, and subsequently evolved independently.
To introduce a controlled cross-fidelity discrepancy, we vary both the numerical resolution and the physical parameter governing diffusion.
The HF reference is generated with viscosity
\(\nu_{\mathrm{HF}}=10^{-2}\) on a \(128\times128\) space--time grid, whereas the LF trajectories are simulated with the larger viscosity
\(\nu_{\mathrm{LF}}=10^{-1}\) on coarser \(N\times N\) space--time grids, with
\(N\in\{16,32\}\).
The results reported in the main text correspond to \(N=32\), while \(N=16\) is reported in the supplementary experiments.
Here, \(N\) denotes the resolution along both the spatial and temporal dimensions.
The larger LF viscosity produces smoother trajectories, while the lower HF viscosity yields sharper gradients, thereby inducing a systematic discrepancy in the solution dynamics in addition to the coarse-grid discretization error.

Before being provided to the model, each LF trajectory is upsampled to the \(128\times128\) resolution of the HF reference.
Importantly, the LF field is obtained entirely from its own numerical simulation and does not use sparse HF measurements during its construction.
Sparse HF observations are sampled separately from the HF reference and are used only for sparse training supervision and inference-time conditioning.
This setting therefore evaluates PICSB with a simulator-based LF prior that differs from the HF target in both discretization resolution and physical fidelity.

For the additional Burgers experiments, following \citep{huang2024diffusionpde}, HF trajectories are simulated on a \(512\times512\) space--time grid with \(\nu_{\mathrm{HF}}=0.01\) and downsampled to \(128\times128\), while LF inputs are generated from the same initial condition using \(\nu_{\mathrm{LF}}=0.1\).
These experiments are reported in Appendix~\ref{app:sensor_regimes}

\paragraph{2D Darcy Flow.}
We consider the steady-state Darcy equation on \(\xi\in(0,1)^2\):
\begin{align*}
-\nabla\cdot\big(a(\xi)\,\nabla u(\xi)\big) = f(\xi),
\end{align*}
where \(u(\xi)\) denotes the pressure field and \(a(\xi)\) is a spatially varying permeability coefficient.
We assume a known governing operator with Dirichlet boundary condition \(u(\xi)=0\) on \(\partial(0,1)^2\) and constant forcing \(f(\xi)=1\).
The permeability field \(a(\xi)\) is provided for each sample, and all fields are represented on a \(128\times128\) grid following the setup used in prior work~\citep{huang2024diffusionpde}.
HF references correspond to the PDE solution \(u(\xi)\) for the given \(a(\xi)\).
LF inputs are constructed by applying nearest-neighbor interpolation to the sparse HF measurements, yielding a coarse proxy field.

\Needspace{15\baselineskip}
\paragraph{(2+1)D Gray--Scott Reaction--Diffusion.}
We use the (2+1)D Gray--Scott dataset introduced in The Well~\citep{ohana2024well}.
Here, \((2+1)\)D indicates two spatial dimensions \((\xi_1,\xi_2)\) together with one temporal dimension \(\gamma\).
We consider the Gray--Scott system on a periodic spatial domain \(\xi=(\xi_1,\xi_2)\in[-1,1]^2\) over the time horizon \(\gamma\in(0,10000]\):
\begin{align*}
\partial_\gamma A(\xi,\gamma) &= \delta_A\nabla^2 A(\xi,\gamma) - A(\xi,\gamma)B(\xi,\gamma)^2 + f\big(1-A(\xi,\gamma)\big), \\
\partial_\gamma B(\xi,\gamma) &= \delta_B\nabla^2 B(\xi,\gamma) + A(\xi,\gamma)B(\xi,\gamma)^2 - (f+k)B(\xi,\gamma),
\end{align*}
where \(A(\xi,\gamma)\) and \(B(\xi,\gamma)\) denote the scalar concentration fields of two chemical species, and \(f\) and \(k\) are the feed and kill rates, respectively.
We use diffusion coefficients \(\delta_A=2\times10^{-5}\) and \(\delta_B=1\times10^{-5}\), feed rate \(f=0.098\), and kill rate \(k=0.057\).
The original trajectories contain \(1001\) time frames of \(128\times128\) spatial fields; we downsample them to a \(64\times64\times64\) space--time grid, corresponding to \(64\) time frames on a \(64\times64\) spatial grid.
HF references correspond to the two-channel scalar concentration field \((A,B)\).
LF inputs are constructed by subsampling sparse HF concentration observations at each physical time frame and applying bi-cubic interpolation channel-wise to obtain a provisional full-field estimate.

\paragraph{(2+1)D Kolmogorov Flow.}
We use the turbulent (2+1)D Kolmogorov flow dataset from prior work~\citep{shu2023physics}, also adopted in a subsequent field-reconstruction study~\citep{li2025physics}.
Here, \((2+1)\)D indicates two spatial dimensions \((\xi_1,\xi_2)\) together with one temporal dimension \(\gamma\).
We consider the vorticity form of the Navier--Stokes dynamics on a periodic spatial domain \(\xi=(\xi_1,\xi_2)\in(0,2\pi)^2\) over the time horizon \(\gamma\in(0,1.25]\):
\begin{align*}
\partial_\gamma \omega(\xi,\gamma) + v(\xi,\gamma)\cdot\nabla \omega(\xi,\gamma) = \frac{1}{Re}\nabla^2 \omega(\xi,\gamma) + f(\xi,\gamma),
\end{align*}
where \(\omega(\xi,\gamma)\) denotes the vorticity field and \(v(\xi,\gamma)\) is the corresponding velocity field.
We use the benchmark configuration from prior work~\citep{shu2023physics}, with Reynolds number \(Re\!=\!1000\) and forcing \(f(\xi,\gamma) \!=\! -4\cos(4\xi_2) \!-\! 0.1\,\omega(\xi,\gamma)\).
Spatial fields are discretized on a \(256\times256\) grid, and we use \(40\) uniformly spaced time frames over \(\gamma\in(0,1.25]\).
HF references are obtained from the provided simulations.
LF inputs are constructed by subsampling sparse HF vorticity observations and applying bi-cubic interpolation to obtain a provisional full-field estimate.
To evaluate physics consistency, we follow the residual formulation used in prior work~\citep{li2025physics} and compute PDE residuals from the vorticity equation; under periodic boundary conditions, the velocity field is recovered from vorticity via the stream function \(\psi\):
\begin{align*}
v(\xi,\gamma)=\nabla\times\psi(\xi,\gamma), \quad \psi(\xi,\gamma)=\nabla^{-2}\omega(\xi,\gamma).
\end{align*}

\paragraph{Dynamic 1D Diffusion--Reaction Equation.}
We use the Dynamic 1D Diffusion--Reaction dataset from PDEBench~\citep{takamoto2022pdebench}.
We consider the one-dimensional Diffusion--Reaction equation on \(\xi\!\in\!(0,1)\) over the time horizon \(\gamma\!\in\!(0,1]\),
\begin{align*}
\partial_\gamma u(\xi,\gamma) = \nu\,\nabla^2 u(\xi,\gamma) + \rho\,u(\xi,\gamma)\big(1-u(\xi,\gamma)\big),
\end{align*}
with initial condition \(u(\xi,0)=u_0(\xi)\).
HF trajectories are generated by simulating the equation with diffusion coefficient \(\nu=5\) and reaction coefficient \(\rho=5\).
HF references correspond to the scalar solution field \(u(\xi,\gamma)\).
LF inputs are constructed by applying nearest-neighbor interpolation to the sparse HF measurements, yielding a coarse proxy field.

\clearpage
\subsection{Implementation Details}\label{app:subsec:ID}
PICSB parameterizes the velocity network \(v_\theta(x,\mathrm{cond})\) using an AdaIN-conditioned U-Net based on the architecture introduced in prior work~\citep{bortoli2024schrodinger}.
The network takes as input the current state \(x_{\tau_{t}}\) and a conditioning vector \(\mathrm{cond}\).
The conditioning vector \(\mathrm{cond}=\mathrm{MLP}(\tau_{t})\) is generated by a time-embedding network that encodes the diffusion time.
Conditioning is injected through Adaptive Instance Normalization (AdaIN)~\citep{huang2017arbitrary} in the encoder blocks, enabling the drift to vary with diffusion time while keeping the model fully convolutional over the spatial grid.

\paragraph{Input representation.}
Depending on the benchmark, the input field \(x_{\tau_t}\) is represented using different tensor encodings.
For the Burgers' equation, 1D Diffusion--Reaction equation and Darcy flow, the field is represented as a single channel tensor of shape \(x_{\tau_t}\in\mathbb{R}^{B\times 1\times H\times W}\).
For Gray--Scott, the field is represented as a tensor of shape \(x_{\tau_t}\in\mathbb{R}^{B\times C\times T\times H\times W}\), where \(C\) denotes the number of channels of concentration field and \(T\) denotes the number of time frames.
For the Kolmogorov flow, the temporal dimension is encoded along the channel axis, resulting in a tensor representation \(x_{\tau_t}\in\mathbb{R}^{B\times C\times H\times W}\), where \(C\) denotes the number of time frames and is set to \(40\) in our experiments.
Sparse observations are represented by a boolean mask \(\mathbf{M}\) with the same shape as \(x_{\tau_t}\), and the corresponding observed values are denoted by \(y\).
Hard conditioning is enforced by overwriting the observed entries at every update according to
\[
x_{\tau_t} \leftarrow \mathbf{M}\odot y + (1-\mathbf{M})\odot x_{\tau_t},
\]
which ensures that the observed locations remain fixed throughout both training and sampling.

\paragraph{AdaIN conditioning.}
Given an activation tensor \(x_{\tau_t}\in\mathbb{R}^{B\times C\times H\times W}\) and a conditioning vector \(\mathrm{cond}\in\mathbb{R}^{B\times d_{\mathrm{cond}}}\) (\(d_{\mathrm{cond}}\) corresponds to the time embedding dimension), AdaIN injects the conditioning information through channel-wise affine transformations.
Specifically, the conditioning vector is linearly projected to produce a pair of scale and shift parameters,
\[
(\gamma,\beta)=\mathrm{Linear}(\mathrm{cond})\in\mathbb{R}^{B\times (2C)}, \qquad
\gamma,\beta\in\mathbb{R}^{B\times C},
\]
which are then reshaped to \((B,C,1,\ldots,1)\), i.e., \((B,C,1,1,1)\) for Gray--Scott and \((B,C,1,1)\) for the remaining benchmarks, and applied to the normalized activations as
\[
\mathrm{AdaIN}(x_{\tau_t},\mathrm{cond}) = \gamma\odot \mathrm{IN}(x_{\tau_t}) + \beta.
\]
Here, \(\mathrm{IN}(\cdot)\) denotes instance normalization, which normalizes each channel independently using per-sample spatial statistics.
By modulating feature statistics in this manner, AdaIN enables the velocity field to vary smoothly with diffusion time while preserving a fully convolutional structure over the spatial grid.

\paragraph{Network architecture and training hyperparameters.}
Table~\ref{tab:picsb_hparam_table} summarizes the key architectural and training hyperparameters used for PICSB across all PDE benchmarks.
All benchmarks share the same AdaIN-conditioned U-Net backbone, while benchmark-specific settings include the input/output channel dimensionality, noise scale \(\epsilon\), learning rate, and gradient clipping threshold.
The time-embedding network is implemented as a two-layer MLP with a hidden dimension of \(16\).

\vspace{0pt}
All models are trained using a fixed solver discretization with \(10\) bridge-time steps and the noise scale specified in Table~\ref{tab:picsb_hparam_table}, and the same settings are used during inference-time sampling to ensure consistency between training and inference dynamics.
A refresh period is set to \(100\), meaning that the corresponding refresh operation is performed every \(100\) training iterations, resulting in a total of \(1000\) refreshes over \(100{,}000\) training iterations.
To stabilize optimization under different PDE scalings, gradient clipping is applied during training: for the Burgers, Gray--Scott, and Kolmogorov benchmarks, gradients of both the velocity network and the time-embedding network are clipped with a maximum norm of \(1\), whereas for the 1D Diffusion--Reaction and Darcy benchmarks—which exhibit different scaling characteristics—a much smaller clipping threshold of \(10^{-5}\) is used.
These settings were found to be crucial for stable training across all benchmarks.
\begin{table}[H]
  \captionsetup{skip=0.05in}
  \caption{PICSB architectural and training hyperparameters used for each PDE benchmark.
  All benchmarks share the same network architecture and training schedule, except for the input/output channel dimensionality and learning rate.
  The table summarizes the key structural and optimization-related hyperparameters that define the PICSB model.}
  \label{tab:picsb_hparam_table}
  \centering
  \begin{small}
  \begin{sc}
  \resizebox{\columnwidth}{!}{%
    \begin{tabular}{lccccccccccc}
      \toprule
      PDE &
      Enc. ch. &
      Dec. ch. &
      \#Enc. &
      \#Dec. &
      Kernel &
      \(d_{\mathrm{cond}}\) &
      Bridge-time steps &
      \(\epsilon\) &
      LR &
      Train iters. \\
      \midrule
      Burgers
      & \(128,256\)
      & \(256,128\)
      & 2
      & 2
      & \(3\times3\)
      & 8
      & 10
      & \(10^{-2}\)
      & \(10^{-3}\)
      & 100k \\
      \midrule
      Darcy
      & \(128,256\)
      & \(256,128\)
      & 2
      & 2
      & \(3\times3\)
      & 8
      & 10
      & \(10^{-6}\)
      & \(10^{-6}\)
      & 100k \\
      \midrule
      Gray--Scott
      & \(128,256\)
      & \(256,128\)
      & 2
      & 2
      & \(3\times3\)
      & 8
      & 10
      & \(10^{-2}\)
      & \(10^{-4}\)
      & 100k \\
      \midrule
      Kolmogorov
      & \(128,256\)
      & \(256,128\)
      & 2
      & 2
      & \(3\times3\)
      & 8
      & 10
      & \(10^{-2}\)
      & \(10^{-3}\)
      & 100k \\
      \midrule
      Diffusion--Reaction
      & \(128,256\)
      & \(256,128\)
      & 2
      & 2
      & \(3\times3\)
      & 8
      & 10
      & \(10^{-7}\)
      & \(10^{-4}\)
      & 100k \\
      \bottomrule
    \end{tabular}%
  }
  \end{sc}
  \end{small}
\end{table}

In contrast to the baselines, PICSB does not use full-field HF references for hyperparameter tuning or model selection. Hyperparameter tuning based on training-set reconstruction error is performed only for \textit{PINNs} and \textit{Guidance}.

\Needspace{15\baselineskip}
\paragraph{PINNs.}
We implement the PINNs baseline as a coordinate-based MLP that parameterizes the unknown solution field as \(u_\phi(\xi,\gamma)\) (or \(u_\phi(\xi)\) for steady-state), and optimizes the network parameters independently for each test instance using Adam for \(100\)k steps (Table~\ref{tab:pinns_hparam_table}). The MLP consists of stacked fully connected layers with \(\tanh\) activations and Xavier initialization, and its input dimensionality is benchmark-dependent: \((\xi,\gamma)\) for Burgers and 1D Diffusion--Reaction, \((\xi_1,\xi_2)\) for Darcy, and \((\xi_1,\xi_2,\gamma)\) for Gray--Scott and Kolmogorov. The output is scalar-valued except for Gray--Scott, where the network predicts the two-channel concentration field \((A,B)\).

\vspace{0pt}
At each optimization step, the loss is a weighted sum of an observation-fitting term and a physics residual term, \(\lambda_{\mathrm{obs}}\mathcal{L}_{\mathrm{obs}}+\lambda_{\mathrm{phys}}\mathcal{L}_{\mathrm{phys}}\), where \(\mathcal{L}_{\mathrm{phys}}\) is computed by evaluating the governing PDE residual at \(N_{\mathrm{colloc}}\) collocation points via a dedicated physics-loss routine. For the observation term, we fit the network to the available sparse HF observations. For each benchmark, we fix \(\lambda_{\mathrm{obs}}=1\) and tune \(\lambda_{\mathrm{phys}}\) over \(10^{\alpha}\) with \(\alpha\in\{-5,\ldots,1\}\) based on the relative \(L^2\) error on the training set.

\vspace{0pt}
For Darcy, Dirichlet boundary conditions are explicitly enforced by sampling \(N_{\mathrm{bc}}\) boundary coordinates on \(\partial(0,1)^2\) and adding an extra boundary penalty weighted by \(\lambda_{\mathrm{bc}}\). To stabilize optimization under Darcy’s scaling, we additionally apply gradient clipping with a maximum norm of \(10^{-5}\) during each training step, whereas the other benchmarks do not require this clipping. After optimization, we recover the full-grid field by evaluating the trained coordinate MLP on a dense coordinate mesh.
\begin{table}[H]
  \captionsetup{skip=0.05in}
  \caption{PINNs architectural and optimization hyperparameters used for each PDE benchmark.
  The table summarizes the key structural and optimization-related hyperparameters that define the PINNs baseline, with Darcy-specific stabilization settings reported in the last column.}
  \label{tab:pinns_hparam_table}
  \centering
  \begin{small}
  \resizebox{\columnwidth}{!}{%
    \begin{tabular}{lccccccccc}
      \toprule
      PDE &
      Hidden dim &
      Depth &
      \(N_{\mathrm{colloc}}\) &
      \(\lambda_{\mathrm{obs}}\) &
      \(\lambda_{\mathrm{phys}}\) &
      LR &
      Train iters. &
      Extra settings \\
      \midrule
      Burgers
      & 512
      & 6
      & 8192
      & 1.0
      & 1.0
      & \(10^{-3}\)
      & 100k
      & -- \\
      \midrule
      Darcy
      & 512
      & 6
      & 8192
      & 1.0
      & 1.0
      & \(10^{-3}\)
      & 100k
      & \(\lambda_{\mathrm{bc}}{=}1,\; N_{\mathrm{bc}}{=}2048\), grad clip (1e-5) \\
      \midrule
      Gray--Scott
      & 128
      & 4
      & 8192
      & 1.0
      & \(10^{-5}\)
      & \(10^{-3}\)
      & 100k
      & -- \\
      \midrule
      Kolmogorov
      & 256
      & 4
      & 8192
      & 1.0
      & 1.0
      & \(10^{-3}\)
      & 100k
      & -- \\
      \midrule
      Diffusion--Reaction
      & 128
      & 4
      & 4096
      & 1.0
      & \(10^{-5}\)
      & \(10^{-3}\)
      & 100k
      & -- \\
      \bottomrule
    \end{tabular}%
  }
  \end{small}
\end{table}

\Needspace{7\baselineskip}
\paragraph{Guidance-based sampling (\textit{Guidance}).}\vspace{0pt}
We implement an inference-time guidance baseline adapted from PDE-guided diffusion pipelines \citep{huang2024diffusionpde}, in which conditional generation is achieved by augmenting an unconditional EDM sampler with gradient-based corrections derived from sparse HF observations and PDE residuals.
Since full-grid HF trajectories are unavailable for training in our multi-fidelity setting, the generative prior is pretrained solely on full-grid LF fields, and HF information is incorporated only at inference time through guidance.

\vspace{0pt}
Given an EDM noise schedule \((\sigma_{\min}, \sigma_{\max}, \rho)\) and a fixed number of sampling steps, each sampling step first performs a standard EDM update consisting of an Euler step followed by a second-order correction \citep{karras2022elucidating}, producing a proposal state \(x_t\).
Guidance is then applied by differentiating scalar loss functions with respect to the current iterate and updating the state along the resulting gradients.

\vspace{0pt}
Across all benchmarks, both the observation loss and the physics loss are computed using \(\ell_2\) norms of their respective residuals.
The observation term penalizes mismatches between the current field and the sparse HF measurements, while the physics term penalizes violations of the governing PDE evaluated on the current field.
During sampling, gradients of these losses with respect to \(x_t\) are used to steer the trajectory toward observation-consistent and physically plausible solutions.
For the 1D Diffusion--Reaction, Gray--Scott, and Kolmogorov benchmarks, guidance is applied in a stage-wise manner.
During the early portion of the sampling trajectory, only observation guidance is used to rapidly align the generated samples with the sparse HF measurements.
After the first 80\% of the total sampling steps, PDE guidance is activated, and the update combines a reduced observation step with an additional physics-based correction.

For 1D Diffusion--Reaction, guidance is applied to the scalar solution field, while for Gray--Scott it is applied jointly to the two concentration channels \((A,B)\) using the Gray--Scott residuals.
The corresponding weights \((\lambda_{\mathrm{obs}}, \lambda_{\mathrm{phys}})\) are reported in Table~\ref{tab:guidance_hparam_table}.

\vspace{0pt}
For Darcy, additional stabilization is required due to the markedly different scaling induced by the elliptic operator and the spatially varying permeability field.
In this case, boundary consistency is explicitly enforced during sampling by introducing a boundary residual term defined on a boundary mask.
PDE guidance is activated after the first 40\% of the sampling steps, and the update initially combines observation, PDE, and boundary guidance.
After the first 80\% of the sampling steps, only PDE guidance is retained to stabilize late-stage sample generation.
Separate step sizes are used for observation, PDE, and boundary guidance, and the corresponding gradients are optionally normalized and clipped to improve robustness under varying state scales.

\vspace{0pt}
Since \textit{Guidance} is highly sensitive to residual-gradient scaling, activation schedules, and guidance weights, we tune these hyperparameters separately for each benchmark.
For each residual-guidance coefficient, we perform a two-stage search.
We first select a scaling factor of the form \(10^{\alpha}\), with \(\alpha \in \{-5,\ldots,5\}\), based on the relative \(L^2\) error on the training set.
We then search over a coefficient factor \(2^{\beta}\), with \(\beta \in \{0,\ldots,6\}\), and use \(2^{\beta} \times 10^{\alpha}\) as the final coefficient.
The final values are reported in Table~\ref{tab:guidance_hparam_table}.
\vspace{0pt}
\begin{table}[H]
  \caption{Guidance-based sampling hyperparameters used for each PDE benchmark.
  All benchmarks use the same noise schedule parameters \((\sigma_{\min}, \sigma_{\max}, \rho)\).
  Observation and PDE guidance are applied in a stage-wise manner, while Darcy additionally employs boundary guidance with separate step sizes.}
  \label{tab:guidance_hparam_table}
  \centering
  \small
  \resizebox{\columnwidth}{!}{%
    \begin{tabular}{lcccccccc}
      \toprule
      PDE &
      Sampling steps &
      \(\sigma_{\min}\) &
      \(\sigma_{\max}\) &
      \(\rho\) &
      \(\lambda_{\mathrm{obs}}\) &
      \(\lambda_{\mathrm{phys}}\) &
      \(\lambda_{\mathrm{bc}}\) \\
      \midrule
      Burgers
      & 2000
      & 0.002
      & 80
      & 7
      & 320
      & 100
      & -- \\
      \midrule
      Darcy
      & 2000
      & 0.002
      & 80
      & 7
      & 10
      & 10
      & 10 \\
      \midrule
      Gray--Scott
      & 2000
      & 0.002
      & 80
      & 7
      & 10000
      & 1000
      & -- \\
      \midrule
      Kolmogorov
      & 2000
      & 0.002
      & 80
      & 7
      & \(3.2\times10^{6}\)
      & 1000
      & -- \\
      \midrule
      Diffusion--Reaction
      & 500
      & 0.002
      & 80
      & 7
      & 320
      & \(10^{-1}\)
      & -- \\
      \bottomrule
    \end{tabular}%
  }
\end{table}

\subsection{Visualization Details for Main Qualitative Comparison}
\label{app:fig2_vis_details}

Figure~\ref{fig:experiment_vis} visualizes representative test instances used for the main qualitative comparison. 
For 1D Diffusion--Reaction and Darcy flow, we show a single held-out test field.
For Gray--Scott, we visualize a \(B\)-concentration snapshot at \(\gamma \approx 3.02 \times 10^{3}\) seconds from the trajectory horizon \(\gamma \in (0, 10^{4}]\). 
For Kolmogorov flow, we visualize a vorticity snapshot at \(\gamma = 0.64\) from the trajectory horizon \(\gamma \in (0, 1.25]\). 
These snapshot choices are used only for visualization; quantitative metrics are computed over the full evaluated fields.

\clearpage
\section{Supplementary Experiments}\label{app:sup_ex}
\subsection{Extended Main Experimental Results}
\label{app:full_results}
Table~\ref{tab:full_results} provides the full quantitative results across all evaluated benchmarks and additional Burgers LF resolutions.
It extends the main-text comparison by reporting the same reconstruction metrics and inference wall-times for all considered datasets and settings.
\begin{table}[H]
\captionsetup{skip=0.05in}
\caption{
Main quantitative comparison over 40 test samples. 
Error entries report 95\% confidence intervals, and wall-time entries report average per-sample inference time.
For Burgers, the number in parentheses denotes the LF space--time grid resolution \(N\), where the LF simulation is performed on an \(N\times N\) grid.
Lower is better; \textbf{bold} marks the lowest mean, and \(\dagger\) denotes CI overlap with the best method.
}
\label{tab:full_results}
\centering
\begin{scriptsize}
\resizebox{\textwidth}{!}{
\begin{tabular}{llcccc}
\toprule
Dataset & Metric & \textit{LF} & \textit{PINNs} & \textit{Guidance} & PICSB \\
\midrule

\multirow{5}{*}{Burgers (16)}
& RMSE (\(\times10^{-2}\))
& [9.668, 11.303]	
& \textbf{[1.957, 2.665]}
& [3.723, 4.986]
& [2.489, 2.718]\(^\dagger\)\\

& MAE (\(\times10^{-3}\))
& [79.557, 92.326]
& [14.929, 20.462]
& [23.258, 31.769]
& \textbf{[6.347, 7.357]}  \\

& RelError (\%)
& [71.149, 72.087]
& \textbf{[13.589, 18.117]}
& [25.758, 30.551]
& [16.820, 21.915]\(^\dagger\)\\

& fRMSE (\(\times10^{-4}\))
& [10.375, 12.424]
& [3.533, 4.877]\(^\dagger\)
& [5.609, 7.584]
& \textbf{[3.467, 3.993]} \\

& Average wall-time (seconds)
& -
& 3,957
& 143
& \textbf{0.027} \\

\midrule

\multirow{5}{*}{Burgers (32)}
& RMSE (\(\times10^{-2}\))
& [9.576, 11.183]	
& [1.893, 2.842]\(^\dagger\)
& [2.969, 4.102]
& \textbf{[2.223, 2.381]} \\

& MAE (\(\times10^{-3}\))
& [78.461, 90.933]
& [14.110, 21.402]
& [17.417, 24.602]
& \textbf{[5.752, 6.798]} \\

& RelError (\%)
& [70.426, 71.428]
& [12.445, 22.565]\(^\dagger\)
& [20.177, 25.015]
& \textbf{[14.889, 19.700]} \\

& fRMSE (\(\times10^{-4}\))
& [9.959, 11.889]
& [3.463, 5.093]
& [4.520, 6.298]
& \textbf{[3.069, 3.453]} \\

& Average wall-time (seconds)
& -
& 4,068
& 155
& \textbf{0.026} \\

\midrule

\multirow{5}{*}{Diffusion--Reaction}
& RMSE (\(\times10^{-2}\))
& \textbf{[0.288, 0.352]}
& [0.355, 1.275]
& [0.152, 0.885]\(^\dagger\)
& [0.516, 0.531] \\

& MAE (\(\times10^{-3}\))
& [1.127, 1.507]
& [1.946, 7.897]
& [0.461, 5.387]\(^\dagger\)
& \textbf{[0.795, 0.871]} \\

& RelError (\%)
& \textbf{[0.309, 0.382]}
& [0.372, 1.411]\(^\dagger\)
& [0.153, 0.984]\(^\dagger\)
& [0.553, 0.575] \\

& fRMSE (\(\times10^{-4}\))
& \textbf{[0.595, 0.722]}
& [1.270, 4.884]
& [0.400, 3.306]\(^\dagger\)
& [1.534, 1.588] \\

& Average wall-time (seconds)
& -
& 4,706
& 28
& \textbf{0.012} \\

\midrule

\multirow{5}{*}{Darcy}
& RMSE (\(\times10^{-3}\))
& [1.808, 1.845]
& [2.496, 3.212]
& [1.519, 1.551]
& \textbf{[0.121, 0.129]} \\

& MAE (\(\times10^{-3}\))
& [1.296, 1.311]
& [1.998, 2.575]
& [1.093, 1.103]
& \textbf{[0.089, 0.094]} \\

& RelError (\%)
& [4.405, 4.451]
& [6.048, 7.833]
& [3.699, 3.745]
& \textbf{[0.293, 0.313]} \\

& fRMSE (\(\times10^{-4}\))
& [0.199, 0.205]
& [0.384, 0.518]
& [0.178, 0.185]
& \textbf{[0.018, 0.020]} \\

& Average wall-time (seconds)
& -
& 4,647
& 141
& \textbf{0.013} \\

\midrule

\multirow{5}{*}{Gray--Scott}
& RMSE (\(\times10^{-1}\))
& [1.159, 1.287]
& \textbf{[0.273, 0.362]}
& [1.240, 1.370]
& [0.330, 0.370]\(^\dagger\) \\

& MAE (\(\times10^{-2}\))
& [5.870, 6.990]
& [1.109, 1.667]
& [6.598, 7.689]
& \textbf{[0.859, 1.061]} \\

& RelError (\%)
& [17.521, 19.740]
& \textbf{[4.129, 5.534]}
& [18.736, 21.028]
& [4.989, 5.693]\(^\dagger\) \\

& fRMSE (\(\times10^{-3}\))
& [0.537, 0.600]
& \textbf{[0.167, 0.312]}
& [1.022, 1.142]
& [0.219, 0.268]\(^\dagger\) \\

& Average wall-time (seconds)
& -
& 2,557
& 498
& \textbf{0.109} \\

\midrule

\multirow{5}{*}{Kolmogorov}
& RMSE (\(\times10^{0}\))
& [0.668, 0.729]
& [2.971, 3.217]
& [0.547, 0.587]
& \textbf{[0.196, 0.220]} \\

& MAE (\(\times10^{0}\))
& [0.286, 0.319]
& [2.344, 2.555]
& [0.312, 0.338]
& \textbf{[0.109, 0.124]} \\

& RelError (\%)
& [13.497, 14.433]
& [60.590, 63.026]
& [11.043, 11.665]
& \textbf{[3.972, 4.299]} \\

& fRMSE (\(\times10^{-2}\))
& [0.159, 0.173]
& [3.387, 3.695]
& [0.232, 0.247]
& \textbf{[0.118, 0.144]} \\

& Average wall-time (seconds)
& -
& 9,527
& 787
& \textbf{0.028} \\

\bottomrule
\end{tabular}}
\end{scriptsize}
\end{table}
\clearpage

Figure~\ref{fig:experiment_vis_full} extends the main-text qualitative comparison to all evaluated benchmarks.
\begin{figure}[H]
    \centering
    \includegraphics[width=1.0\linewidth]{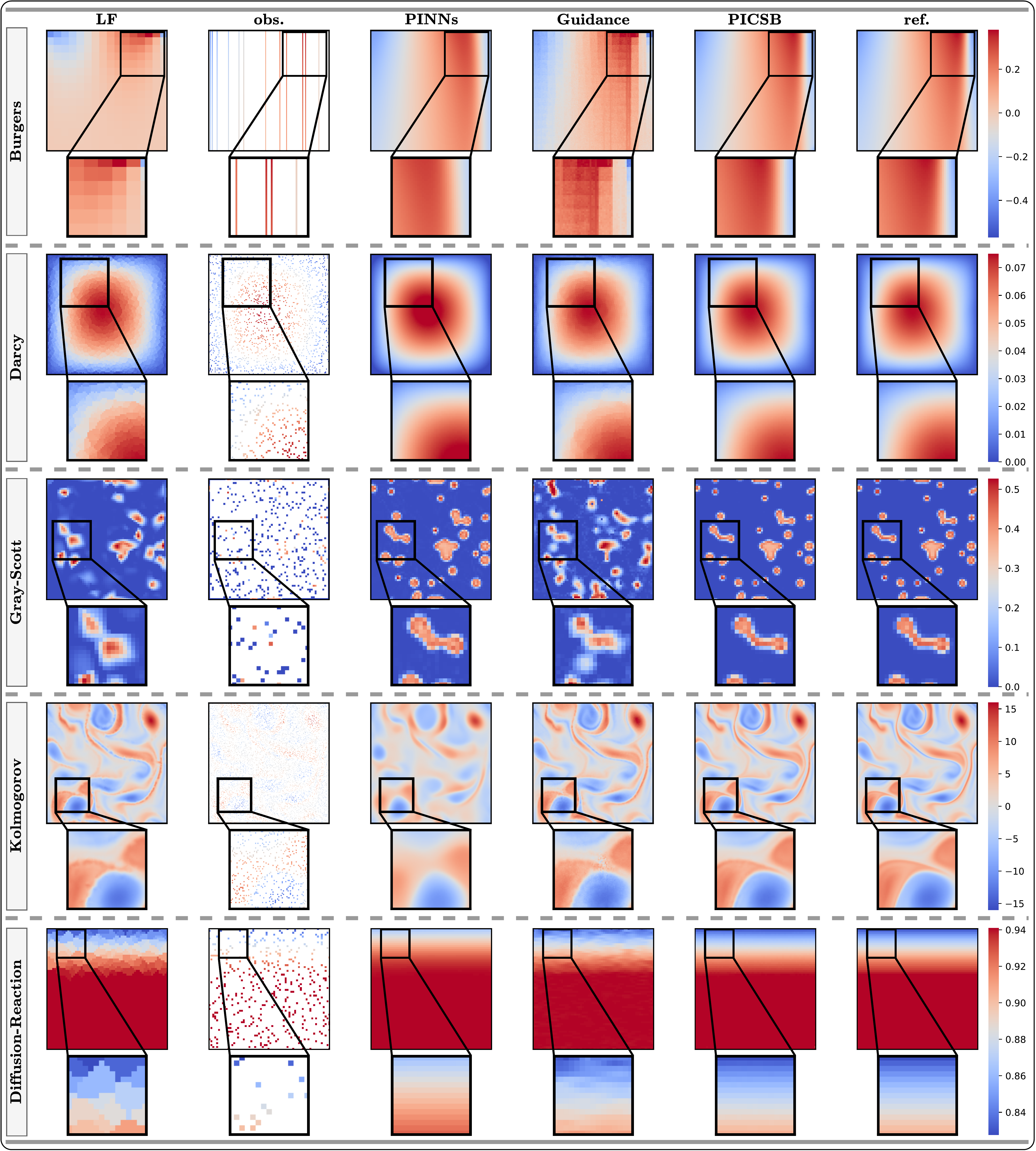}
    \vspace{0pt}
    \caption{
    Qualitative comparison of reconstructed fields. 
    We show representative test samples for Burgers, Darcy flow and Diffusion--Reaction, and snapshots for Gray--Scott and Kolmogorov flow; LF, sparse HF observations (obs.), \textit{PINNs}, \textit{Guidance}, and PICSB are compared with the HF reference (ref.), with zoomed-in regions highlighting local reconstruction quality.
    }
    \label{fig:experiment_vis_full}
    \vspace{0pt}
\end{figure}

\clearpage

Figure~\ref{fig:rell2} further visualizes representative reconstruction errors across the four main benchmarks.
The three methods exhibit different error characteristics across the systems.
For Burgers and Gray--Scott, both \textit{PINNs} and PICSB reproduce the dominant field structure, while \textit{Guidance} shows more spatially distributed discrepancies and visible artifacts.
For Darcy, the \textit{PINNs} and \textit{Guidance} errors remain spread over the domain, whereas the PICSB reconstruction follows the smooth pressure profile with only small residual differences.
For Kolmogorov flow, \textit{PINNs} substantially smooths the vortical field and \textit{Guidance} retains noticeable local discrepancies, while PICSB preserves the dominant vortical structures with errors concentrated around finer-scale regions.
Overall, the error maps indicate that reconstruction differences are primarily associated with sharp transitions, localized reaction structures, and fine vortical features rather than the global field organization.
\begin{figure}[H]
    \centering
    \includegraphics[width=1.0\linewidth]{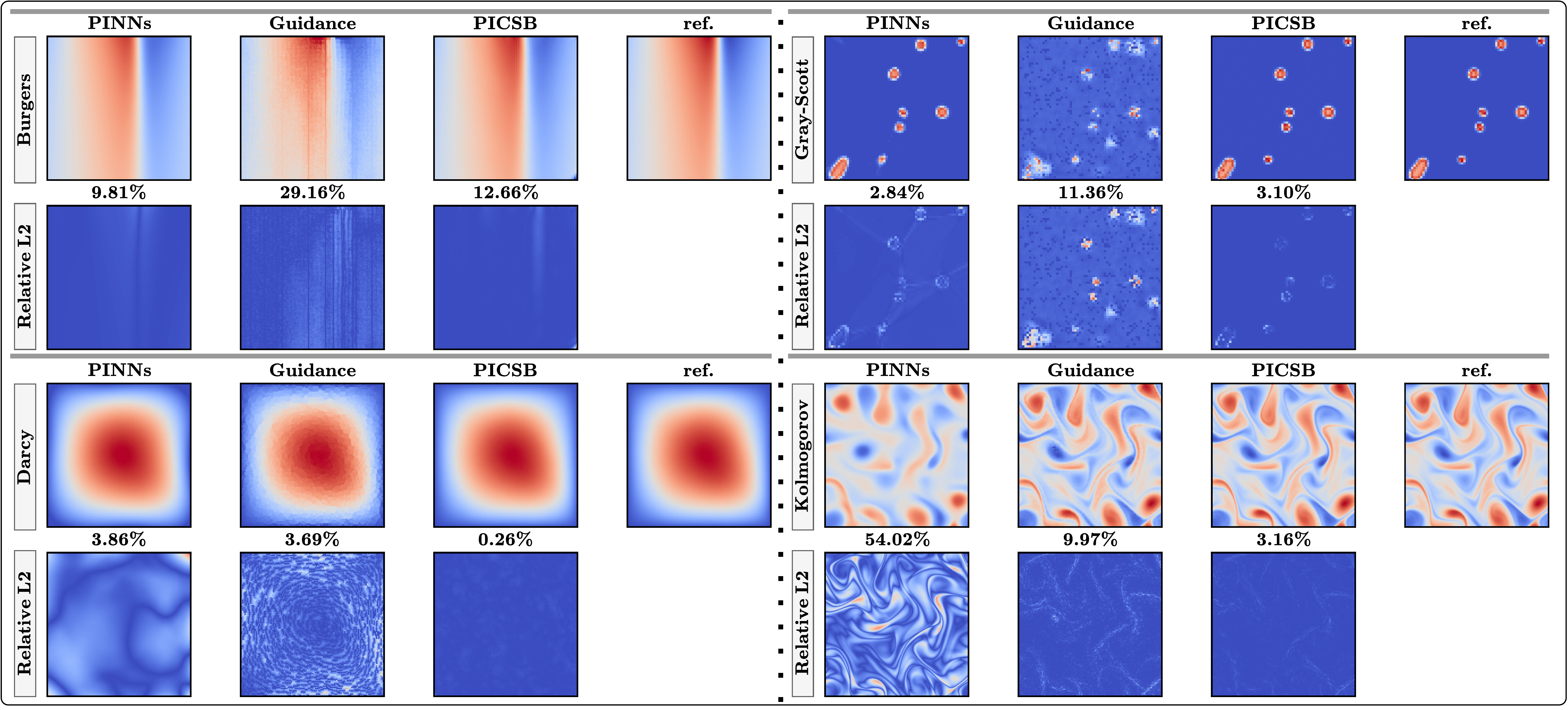}
    \vspace{0pt}
    \caption{
    Representative reconstructions and relative \(L^2\) error maps for Burgers, Gray--Scott, Darcy, and Kolmogorov flow.
    For each benchmark, the top row compares the reconstructed fields with the HF reference, while the bottom row visualizes the corresponding error distribution.
    Percentages denote the sample-wise relative \(L^2\) errors.
    }
    \label{fig:rell2}
    \vspace{0pt}
\end{figure}

\clearpage

\subsection{Additional Experiments under Challenging Settings}
\label{app:challenge_setting}
Table~\ref{tab:challenge_setting} reports additional experiments under more challenging observation regimes and noisy-training settings.
We first consider a globally fixed sensor network, where a single set of sensors is selected on \(10\%\) of the spatial grid and reused for all test samples.
To further examine diverse spatiotemporal observation patterns, we start from a globally fixed sensor network with temporally consistent observations and introduce intermittent sensor failures over time. Specifically, sensors are placed on 10\% of the spatial grid, and among these, up to 10\% are randomly deactivated for contiguous durations ranging from 3 to 12 time steps. This breaks the temporal homogeneity of the observations, resulting in sensor-specific missing measurements across time. For noisy HF supervision, we use the same globally fixed sensor configuration and add Gaussian noise to the supervision during training at levels of 5\% and 10\%, simulating increasing corruption relative to the data scale. All models are trained for \(100\)k iterations.
\begin{table}[H]
\captionsetup{skip=0.05in}
\caption{Robustness study on Kolmogorov (8 samples).
Lower values indicate better performance for all metrics.
\textbf{Boldface} denotes the method with the lowest mean value over the test samples}
\label{tab:challenge_setting}
\centering
\begin{scriptsize}
\setlength{\tabcolsep}{4.5pt}
\renewcommand{\arraystretch}{1.12}
\resizebox{\textwidth}{!}{
\begin{tabular}{llcccc}
\toprule
Setting & Method 
& RMSE
& MAE
& RelError (\%)
& fRMSE \\
\midrule

\multirow{3}{*}{\textit{globally fixed sensor network}}
& \textit{PINNs} 
& 3.090 & 2.436 & 64.012 & 0.035 \\
& \textit{Guidance} 
& 0.679 & 0.398 & 14.062 & 0.004 \\
& PICSB 
& \textbf{0.471} & \textbf{0.265} & \textbf{9.741} & \textbf{0.002} \\

\midrule

\multirow{3}{*}{\textit{w/o temporal homogeneity}}
& \textit{PINNs} 
& 3.099 & 2.443 & 64.215 & 0.036 \\
& \textit{Guidance} 
& 0.674 & 0.395 & 13.970 & 0.003 \\
& PICSB 
& \textbf{0.402} & \textbf{0.203} & \textbf{8.302} & \textbf{0.002} \\

\midrule

\multirow{3}{*}{\textit{noisy HF supervision} (\(5\%\))}
& \textit{PINNs} 
& 3.097 & 2.441 & 64.154 & 0.035 \\
& \textit{Guidance} 
& 0.693 & 0.430 & 14.360 & 0.004 \\
& PICSB 
& \textbf{0.429} & \textbf{0.264} & \textbf{8.878} & \textbf{0.003} \\

\midrule

\multirow{3}{*}{\textit{noisy HF supervision} (\(10\%\))}
& \textit{PINNs} 
& 3.093 & 2.439 & 64.081 & 0.035 \\
& \textit{Guidance} 
& 0.861 & 0.550 & 17.846 & \textbf{0.004} \\
& PICSB 
& \textbf{0.718} & \textbf{0.424} & \textbf{14.917} & 0.006 \\

\bottomrule
\end{tabular}
}
\end{scriptsize}
\end{table}

\Needspace{22\baselineskip}
\subsection{Additional Experiments under Sensor-Location Regimes}
\label{app:sensor_regimes}
To examine more realistic sensing scenarios, we evaluate Burgers, Darcy, and Kolmogorov under three observation regimes defined by how sensor locations are selected and reused across time and samples. For each benchmark, we use 32 training pairs and 4 test pairs, and all experiments are conducted on a single NVIDIA RTX 4090 GPU.

\paragraph{(R1) Resampled observations.}\phantomsection\label{regime:R1}
For dynamic problems (Burgers and Kolmogorov), we independently resample a random \(10\%\) set of spatial sensor locations at each physical time frame.
For the steady Darcy problem, this corresponds to sampling a random \(10\%\) subset of grid points per sample.

\paragraph{(R2) Per-sample fixed sensor set.}\phantomsection\label{regime:R2}
For dynamic problems, we draw one random \(10\%\) sensor set per trajectory and keep it fixed across all time frames of that trajectory, while different trajectories use different sensor sets.
For Darcy, this setting coincides with (\hyperref[regime:R1]{R1}); thus we do not report a separate Darcy result for (\hyperref[regime:R2]{R2}).

\paragraph{(R3) Globally fixed sensor network.}\phantomsection\label{regime:R3}
We choose a single \(10\%\) sensor set once and reuse it for all samples.
For dynamic problems, the same sensor coordinates are applied at every time frame.

\clearpage
\subsubsection{Additional qualitative visualizations}
We present additional qualitative visualizations of PICSB.
Figure~\ref{fig:appendix_main_burgers} and Figure~\ref{fig:appendix_main_darcy} report reconstruction results for the Burgers and Darcy benchmarks, respectively, while Figures~\ref{fig:appendix_main_kolmo_frame1}--\ref{fig:appendix_main_kolmo_frame38} visualize the Kolmogorov benchmark at multiple representative temporal frames.
All figures report reconstruction results across all regimes (R1--R3) and all test samples (four instances), together with the corresponding HF reference solutions.
Reported relative \(L^2\) errors are computed per sample and shown below each reconstructed field.
\begin{figure}[H]
    \centering
    \includegraphics[width=1.0\linewidth]{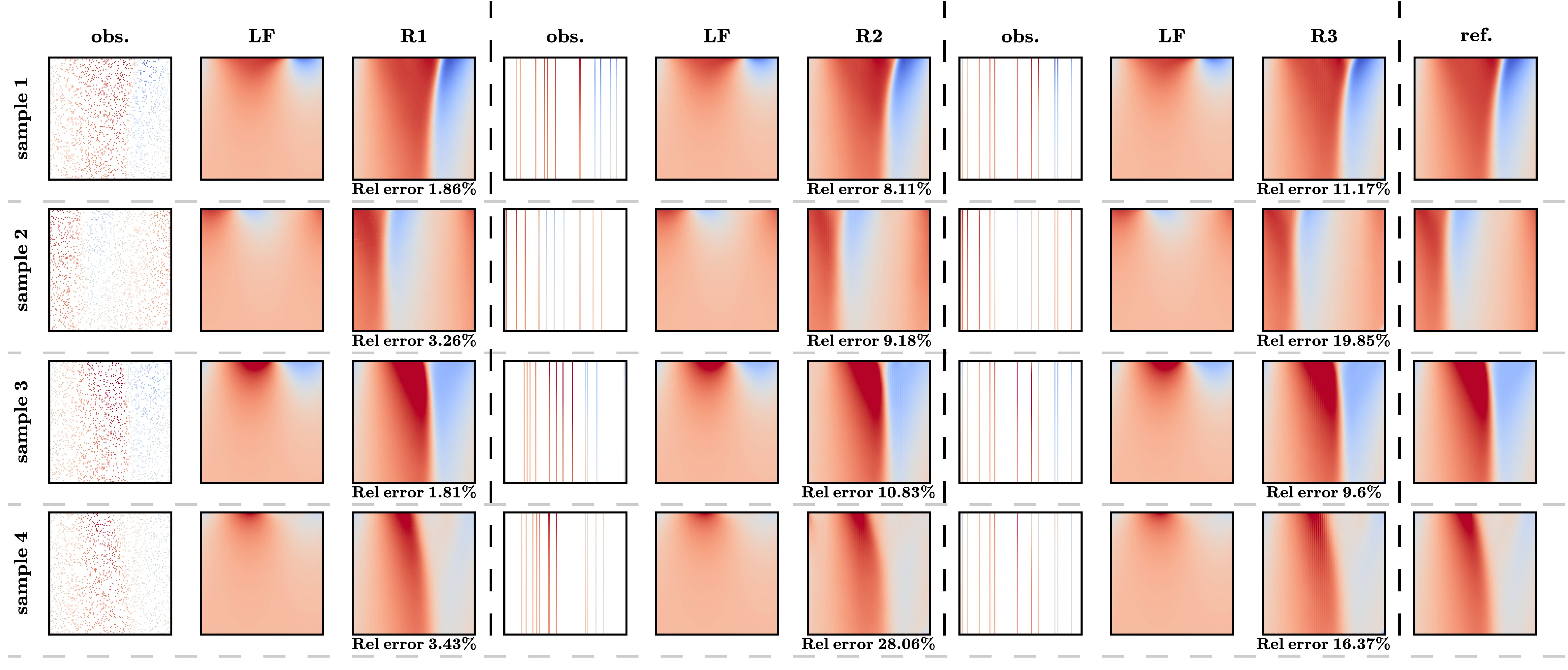}
    \caption{Additional visualization illustrating reconstruction performance for Burgers' equation.
    Each row corresponds to a different test sample.
    Columns show the sparse HF observations used for conditioning (obs.), the LF input, reconstructions obtained under different reconstruction regimes (R1--R3), and the HF reference solution (ref.).
    Relative \(L^2\) errors are reported below each reconstructed field.}
    \label{fig:appendix_main_burgers}
    \vspace{0pt}
\end{figure}

\begin{figure}[H]
    \centering
    \includegraphics[width=1.0\linewidth]{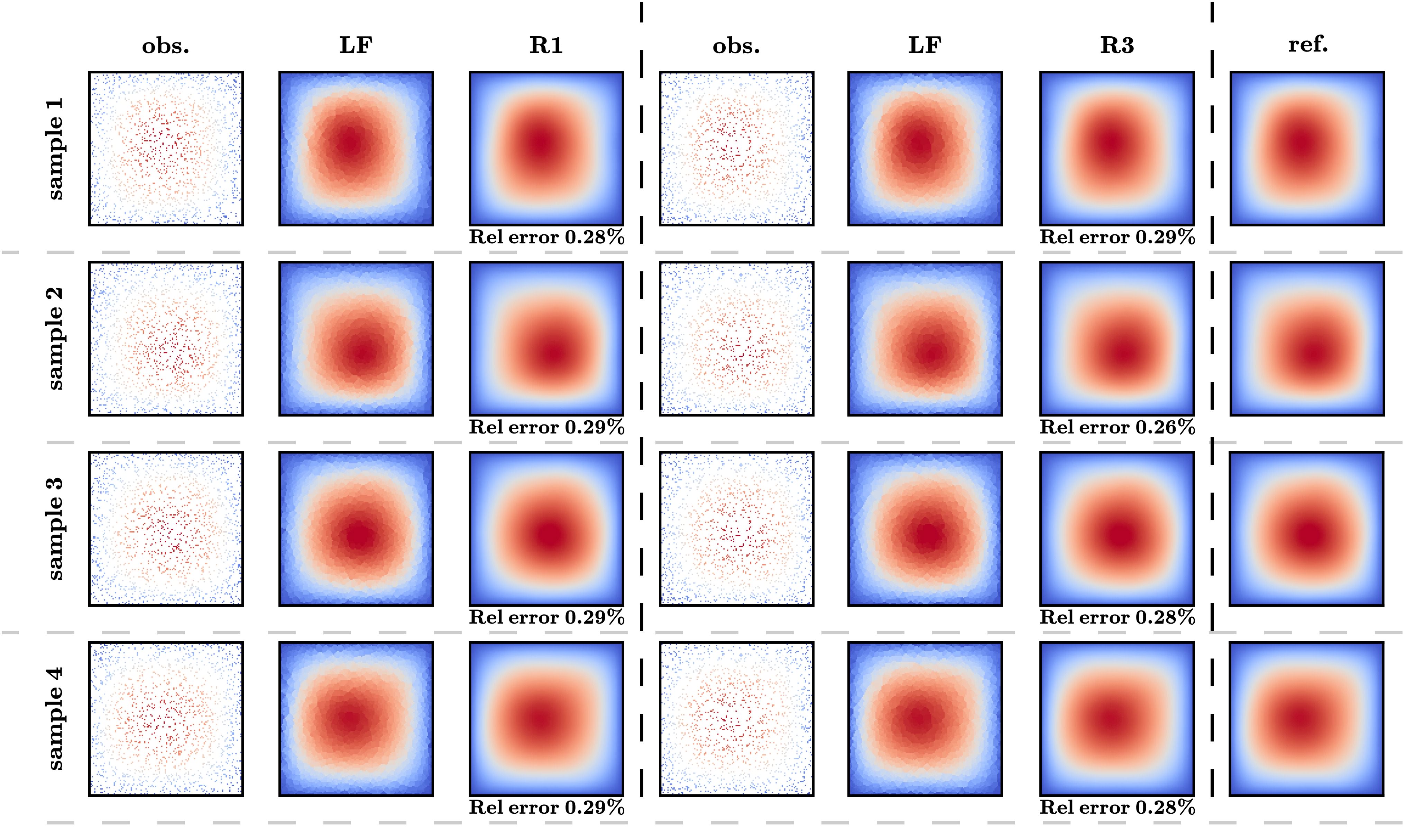}
    \caption{Additional visualization illustrating reconstruction performance for Darcy flow.
    Each row corresponds to a different test sample.
    Columns show the sparse HF observations used for conditioning (obs.), the LF input, reconstructions obtained under different reconstruction regimes (R1/R3), and the HF reference solution (ref.).
    Relative \(L^2\) errors are reported below each reconstructed field.}
    \label{fig:appendix_main_darcy}
    \vspace{0pt}
\end{figure}

\clearpage

\begin{figure}[H]
    \centering
    \includegraphics[width=1.0\linewidth,height=.230\textheight,keepaspectratio]{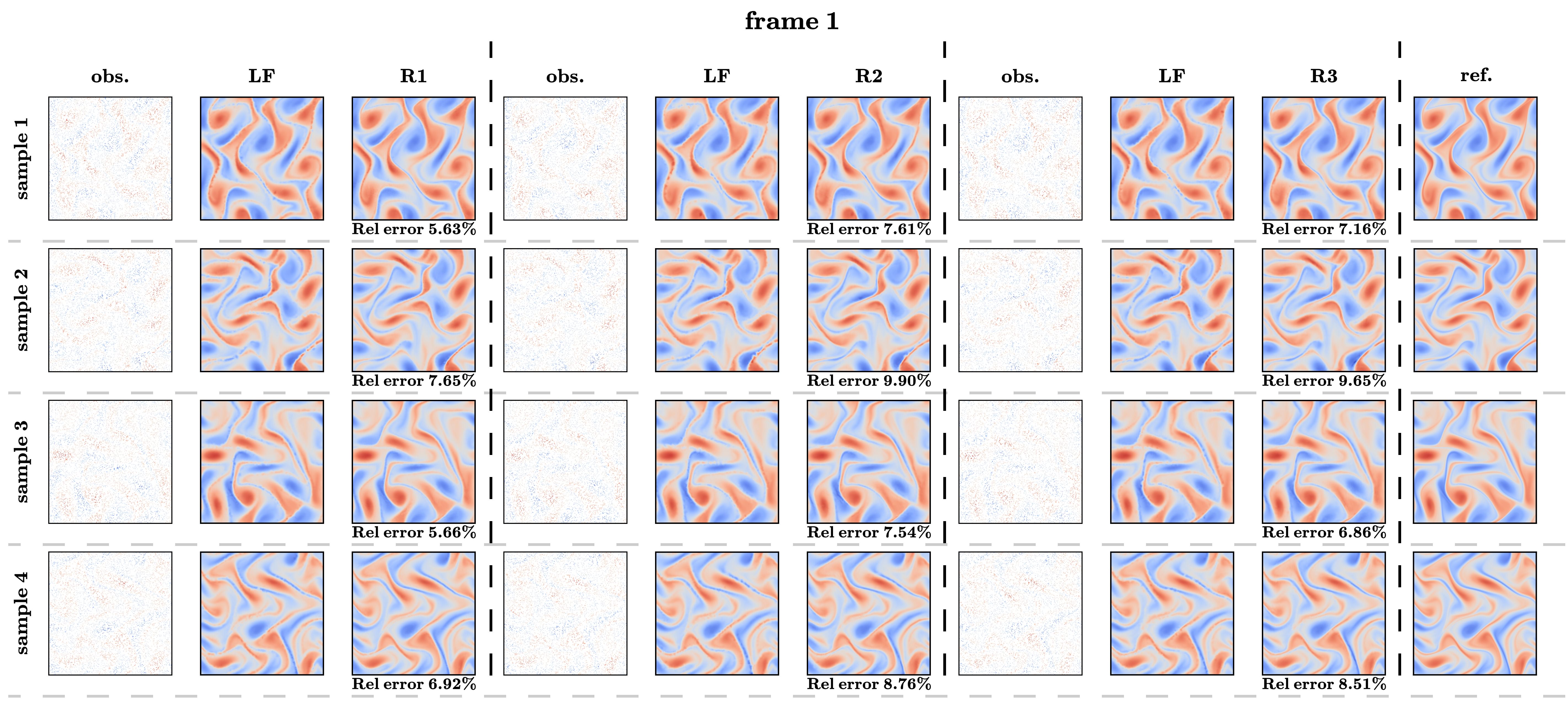}
    \caption{Additional visualization illustrating reconstruction performance for Kolmogorov flow at frame 1.
    Each row corresponds to a different test sample.
    Columns show the sparse HF observations used for conditioning (obs.), the LF input, reconstructions obtained under different reconstruction regimes (R1--R3), and the HF reference solution (ref.).
    Relative \(L^2\) errors are reported below each reconstructed field.}
    \label{fig:appendix_main_kolmo_frame1}
    \vspace{0pt}
\end{figure}

\begin{figure}[H]
    \centering
    \includegraphics[width=1.0\linewidth,height=.230\textheight,keepaspectratio]{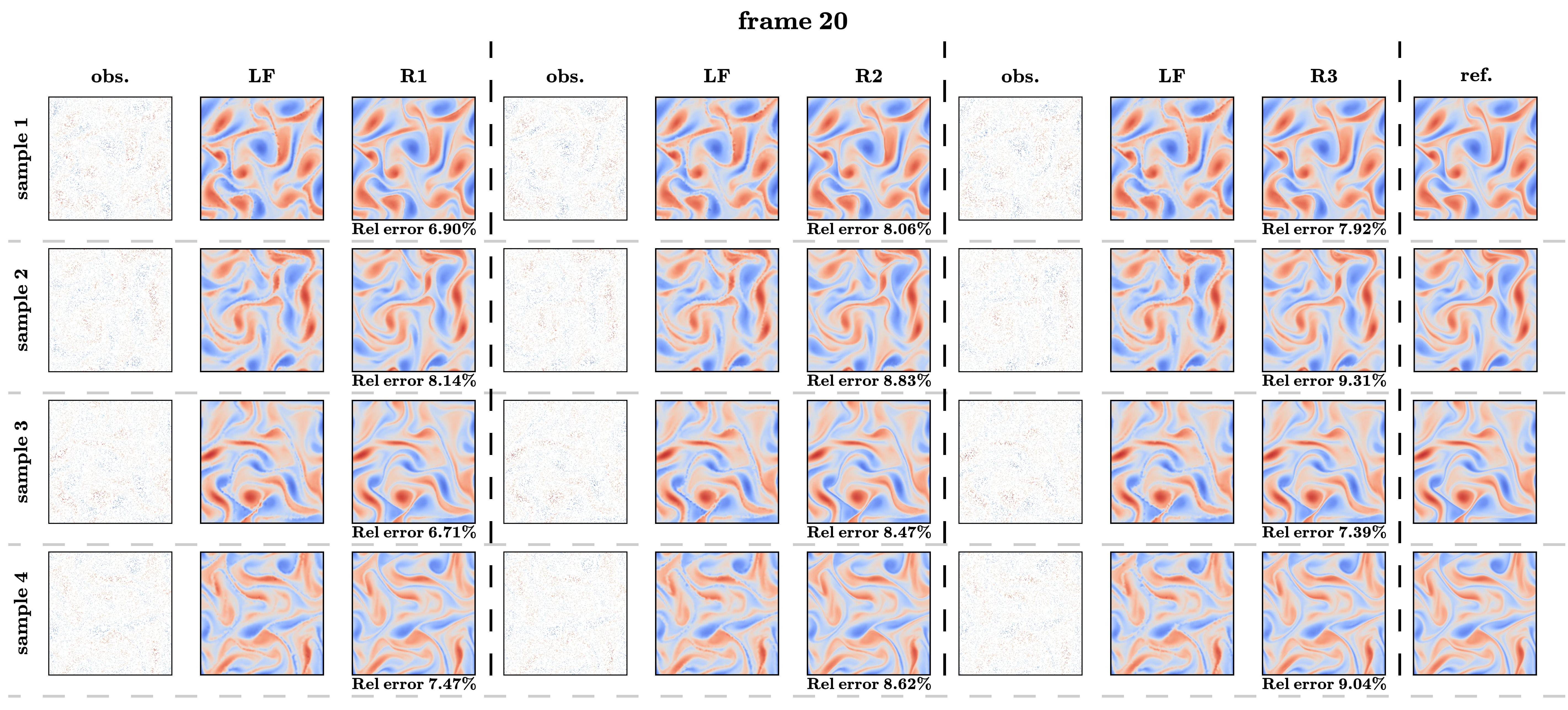}
    \caption{Additional visualization illustrating reconstruction performance for Kolmogorov flow at frame 20.
    Each row corresponds to a different test sample.
    Columns show the sparse HF observations used for conditioning (obs.), the LF input constructed from sparse observations via interpolation, reconstructions obtained under different reconstruction regimes (R1--R3), and the HF reference solution (ref.).
    Relative \(L^2\) errors are reported below each reconstructed field.}
    \label{fig:appendix_main_kolmo_frame20}
    \vspace{0pt}
\end{figure}

\begin{figure}[H]
    \centering
    \includegraphics[width=1.0\linewidth,height=.230\textheight,keepaspectratio]{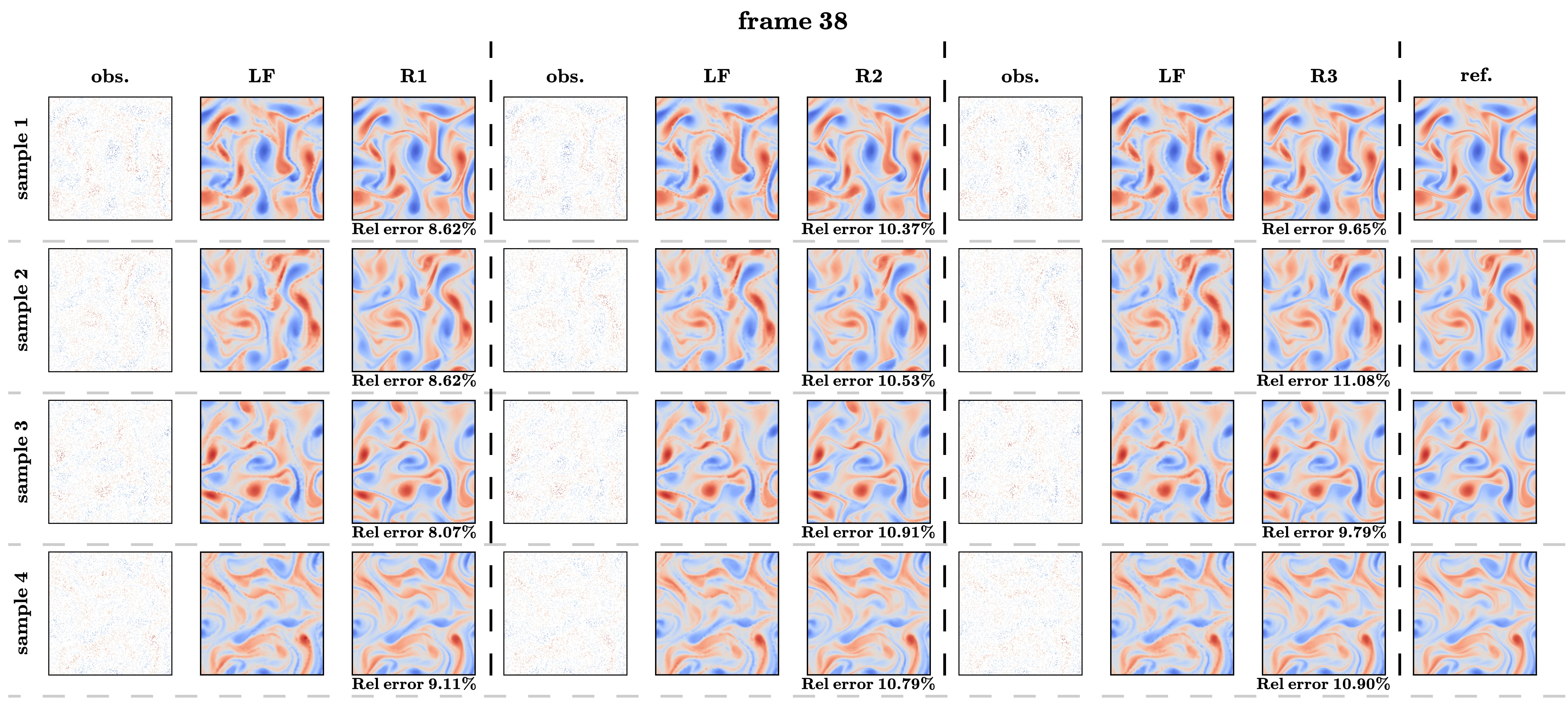}
    \caption{Additional visualization illustrating reconstruction performance for Kolmogorov flow at frame 38.
    Each row corresponds to a different test sample.
    Columns show the sparse HF observations used for conditioning (obs.), the LF input constructed from sparse observations via interpolation, reconstructions obtained under different reconstruction regimes (R1--R3), and the HF reference solution (ref.).
    Relative \(L^2\) errors are reported below each reconstructed field.}
    \label{fig:appendix_main_kolmo_frame38}
    \vspace{0pt}
\end{figure}

\clearpage
\subsubsection{Inference dynamics comparison between PICSB and Guidance}
This section provides a comparison of inference dynamics between PICSB and the guidance-based diffusion baseline.
Figures~\ref{fig:vs_burgers}, \ref{fig:vs_darcy}, and \ref{fig:vs_kolmo_20} visualize the evolution of reconstructed fields for the Burgers, Darcy, and Kolmogorov benchmarks, respectively.

\vspace{0pt}
Across all benchmarks, the guidance-based sampling method relies on a long diffusion trajectory with hundreds to thousands of sampling steps before meaningful large-scale structures emerge.
During early and intermediate stages of sampling, the reconstructed fields remain dominated by noise, and coherent solution structures typically appear only near the end of the diffusion process.

\vspace{0pt}
In contrast, PICSB directly transports the LF prior toward the observation-conditioned HF posterior through a Schr\"odinger bridge formulation.
As a result, physically consistent large-scale structures are recovered within a significantly smaller number of inference steps.
Intermediate states produced by PICSB remain stable and structured throughout the inference trajectory, demonstrating a substantial improvement in inference efficiency compared to diffusion-based guidance sampling.

\begin{figure}[H]
\centering
\includegraphics[width=1.0\linewidth,height=.625\textheight,keepaspectratio]{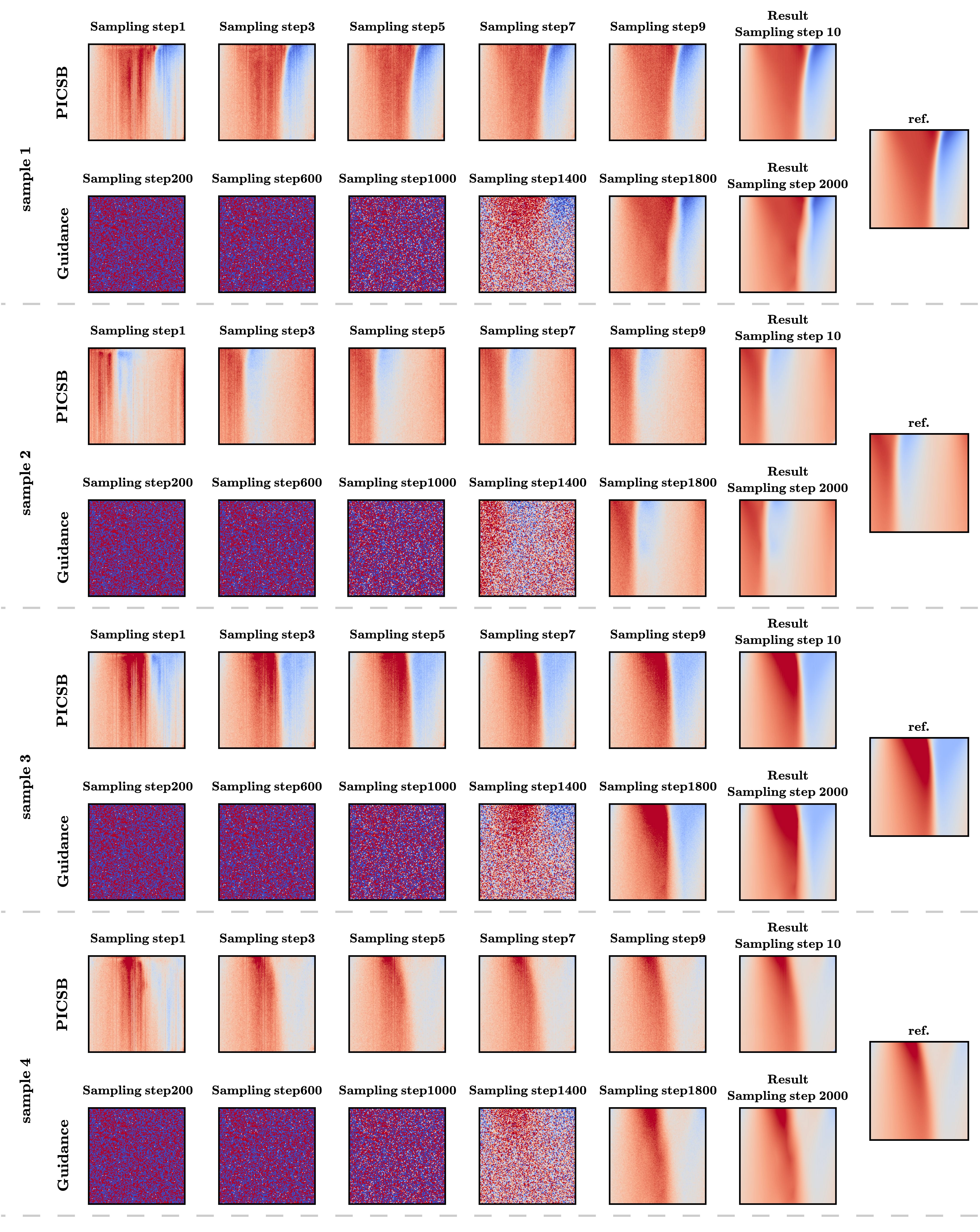}
\vspace{0pt}
\caption{Inference dynamics comparison between PICSB and guidance-based sampling for Burgers' equation.}
\label{fig:vs_burgers}
\end{figure}

\clearpage

\begin{figure}[H]
\centering
\includegraphics[width=1.0\linewidth,height=.82\textheight,keepaspectratio]{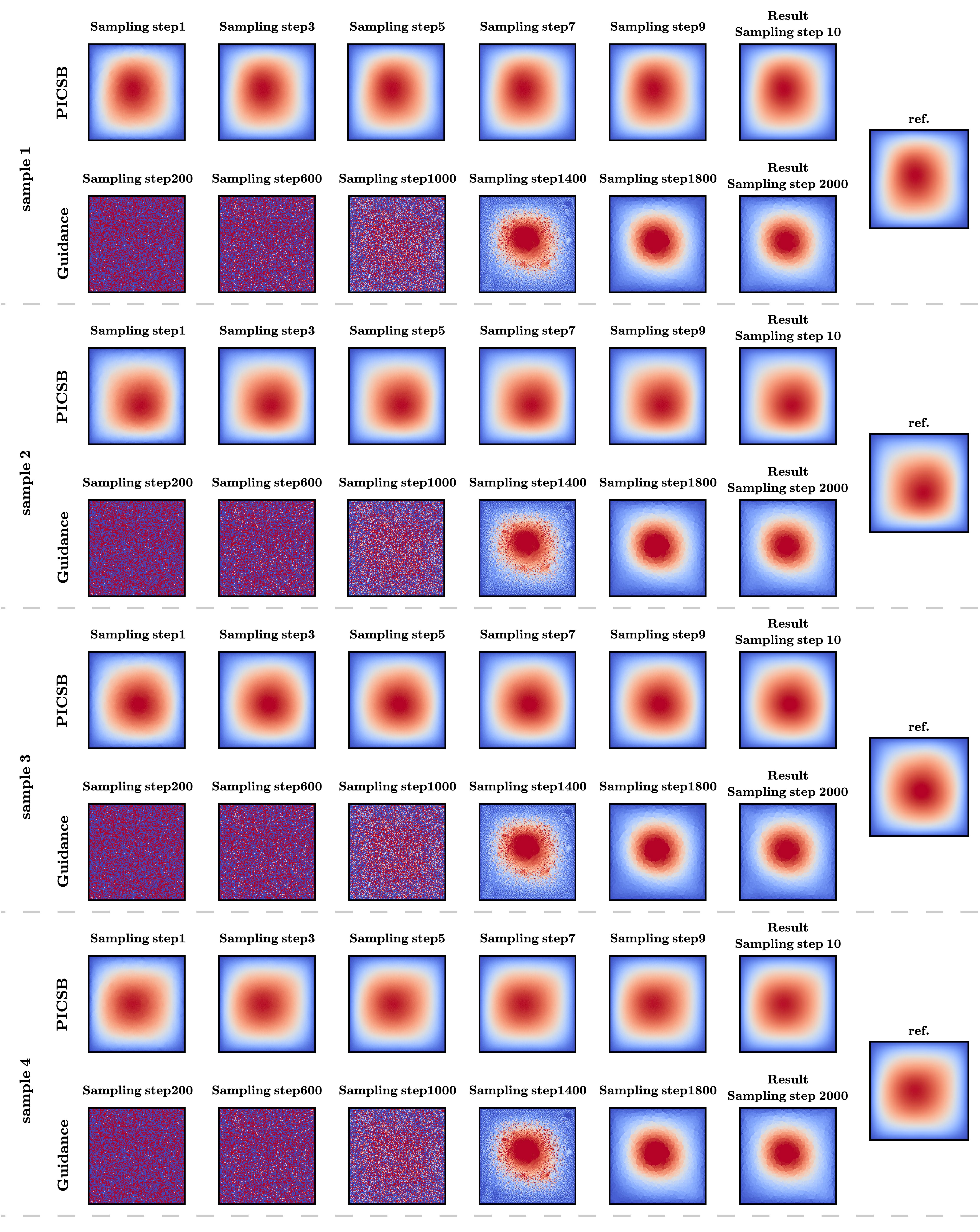}
\vspace{0pt}
\caption{Inference dynamics comparison between PICSB and guidance-based sampling for Darcy flow.}
\label{fig:vs_darcy}
\end{figure}

\clearpage

\begin{figure}[H]
\centering
\includegraphics[width=1.0\linewidth,height=.82\textheight,keepaspectratio]{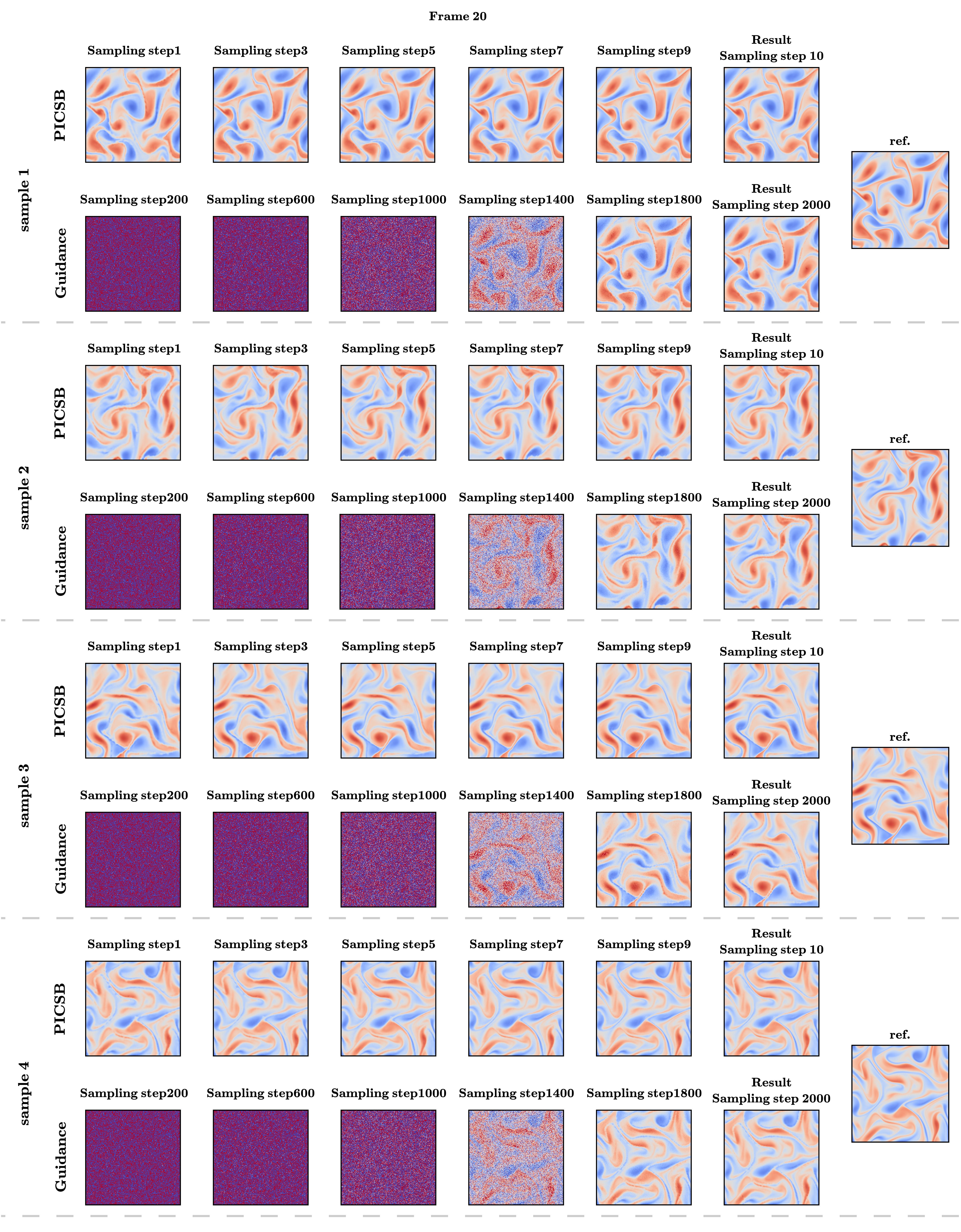}
\vspace{0pt}
\caption{Inference dynamics comparison between PICSB and guidance-based sampling for Kolmogorov flow at frame 20.}
\label{fig:vs_kolmo_20}
\end{figure}

\clearpage
\subsubsection{Inference with reduced observation availability}\label{app:lessobs}
In this experiment, we investigate the robustness of PICSB to further reductions in observation availability at inference time, while keeping the LF prior and the trained model fixed.
All models are trained using the same LF--HF pairs described in Appendix~\ref{app:sensor_regimes} and are optimized with sparse HF observations covering \(10\%\) of the spatial grid.

\vspace{0pt}
At inference time, instead of providing the full \(10\%\) observation set used during training, we deliberately subsample this set and condition the model on only a fraction of the available measurements.
Specifically, we randomly retain either \(50\%\) or \(10\%\) of the original observation locations, corresponding to conditioning on \(5\%\) and \(1\%\) of the full grid, respectively.
This setup allows us to isolate the effect of reduced observational information during inference without changing the learned LF-to-HF transport mechanism.

\vspace{0pt}
Figure~\ref{fig:sameLF_less_obs1} and Figure~\ref{fig:sameLF_less_obs2} show reconstruction results for the Burgers' equation across all observation regimes (R1--R3) and all test samples. When \(5\%\) of observations are available, reconstruction quality is consistently better than in the \(1\%\) setting, indicating that increased observation availability alleviates some of the degradation caused by extreme sparsity.
However, the reconstructed fields still exhibit discrepancies from the HF reference, and the overall reconstruction quality remains limited.
Moreover, this behavior is primarily observed for Burgers' equation; although results for Darcy flow and Kolmogorov flow are not visualized here, we observed that inference with \(5\%\) or fewer observations remains largely unsuccessful for these more complex systems.

\vspace{0pt}
These observations suggest that, while PICSB can partially tolerate reduced observation availability at inference time, simply subsampling the observation set without modifying the trained model is insufficient for reliably handling more complex dynamics.
Motivated by this limitation, we further investigate whether finetuning can improve inference under severely reduced observation regimes in Section~\ref{app:appendix_finetuning}.
\clearpage

\begin{figure}[H]
\centering
\includegraphics[width=1.0\linewidth]{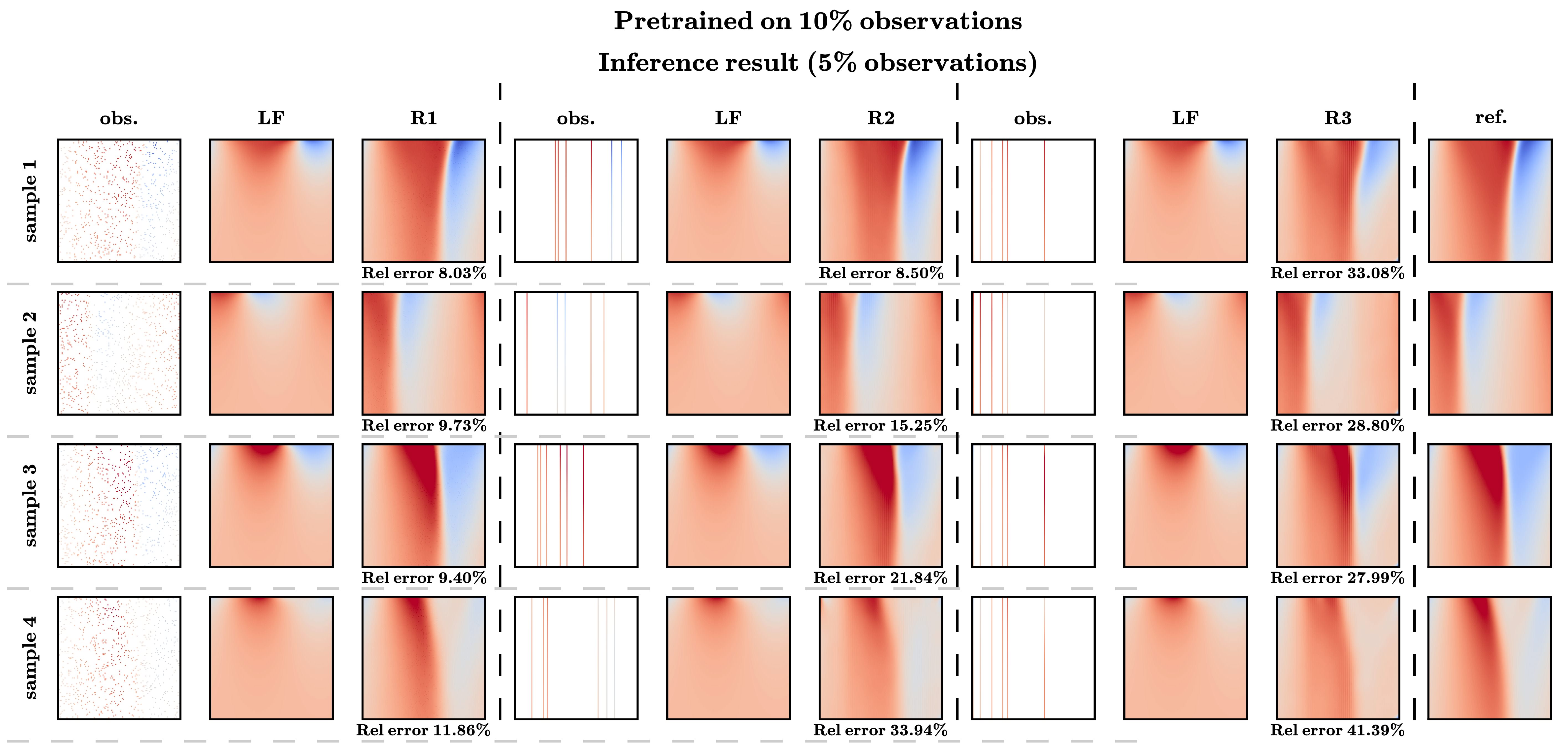}
\vspace{0pt}
\caption{Reconstruction results under \(5\%\) sparse observations for the Burgers' equation.}

\label{fig:sameLF_less_obs1}
\vspace{0pt}
\end{figure}

\begin{figure}[H]
\centering
\includegraphics[width=1.0\linewidth]{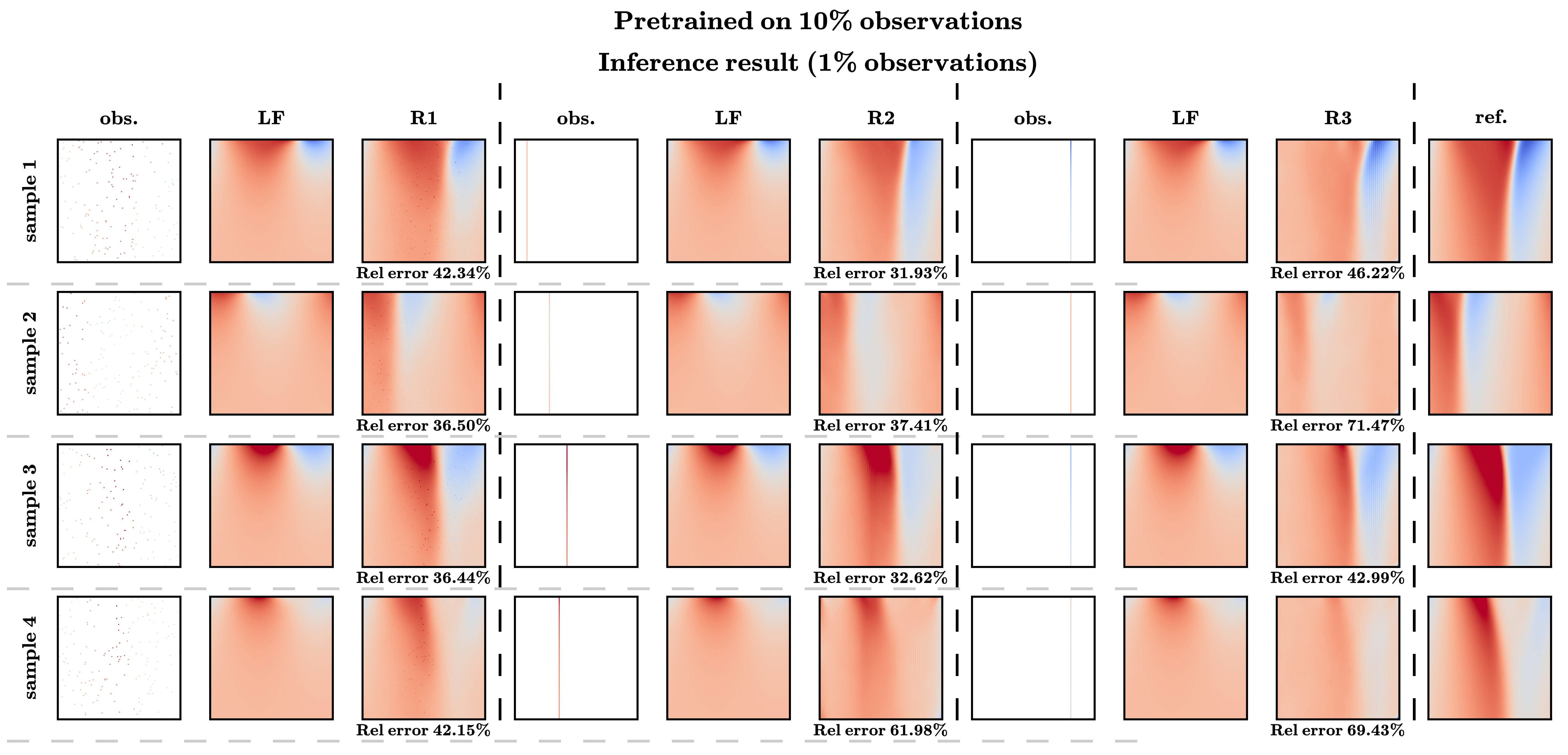}
\vspace{0pt}
\caption{Reconstruction results under \(1\%\) sparse observations for the Burgers' equation.}
\label{fig:sameLF_less_obs2}
\vspace{0pt}
\end{figure}

\clearpage
\subsection{Hyperparameter Sensitivity}
\label{app:hyperparam_sensitivity}
We analyze the sensitivity of PICSB to three key sampling and training hyperparameters: the number of bridge steps, the refresh period, and the Markov noise scale.
All experiments are conducted under the globally fixed sensor regime (R3), and each model is trained for \(30\)k iterations.
Starting from the default configuration, Refresh \(=100\) and Bridge step \(=10\), with Markov noise \(10^{-2}\) for Burgers and Kolmogorov and \(10^{-6}\) for Darcy, we vary one hyperparameter at a time while keeping the others fixed.

\vspace{0pt}
Table~\ref{tab:hyperparam_sensitivity} reports the resulting relative \(L^2\) errors.
Increasing the number of bridge steps degrades performance for Burgers, but improves Darcy and Kolmogorov, suggesting that the optimal bridge discretization depends on the underlying PDE dynamics and LF--HF discrepancy.
The Refresh period has a relatively minor impact across all PDEs, with only small variations observed. In contrast, the Markov noise level significantly affects performance, where excessively large noise leads to substantial degradation, while smaller noise levels generally yield better results, particularly for Burgers and Kolmogorov.

\begin{table}[H]
\centering
\caption{Hyperparameter sensitivity analysis based on RelError computed over 40 test samples (8 for Kolmogorov) under the R3 observation regime. Lower values indicate better performance.}
\label{tab:hyperparam_sensitivity}
\begin{small}
\begin{tabular*}{\textwidth}{@{\extracolsep{\fill}}lcccc}
\toprule
PDE & Hyperparameter & Value & RelError(\%) \(\downarrow\) \\
\midrule

\multirow{3}{*}{Burgers}
& \multirow{3}{*}{Bridge step} 
& 5  & 16.55 \\
& & 10 & 23.18 \\
& & 15 & 28.74 \\

\midrule

\multirow{3}{*}{Darcy}
& \multirow{3}{*}{Bridge step} 
& 5  & 0.46 \\
& & 10 & 0.38 \\
& & 15 & 0.30 \\

\midrule

\multirow{3}{*}{Kolmogorov}
& \multirow{3}{*}{Bridge step} 
& 5  & 9.10 \\
& & 10 & 8.22 \\
& & 15 & 7.93 \\

\midrule

\multirow{3}{*}{Burgers}
& \multirow{3}{*}{Refresh period}
& 10   & 23.43 \\
& & 100 & 23.18 \\
& & 1000 & 25.62 \\

\midrule

\multirow{3}{*}{Darcy}
& \multirow{3}{*}{Refresh period}
& 10   & 0.37 \\
& & 100 & 0.38 \\
& & 1000 & 0.39 \\

\midrule

\multirow{3}{*}{Kolmogorov}
& \multirow{3}{*}{Refresh period}
& 10   & 8.12 \\
& & 100 & 8.22 \\
& & 1000 & 8.17 \\

\midrule

\multirow{3}{*}{Burgers}
& \multirow{3}{*}{Markov noise}
& \(10^{-3}\) & 12.14 \\
& & \(10^{-2}\) & 23.18 \\
& & \(10^{-1}\) & 33.78 \\

\midrule

\multirow{3}{*}{Darcy}
& \multirow{3}{*}{Markov noise}
& \(10^{-7}\) & 0.36 \\
& & \(10^{-6}\) & 0.38 \\
& & \(10^{-5}\) & 0.66 \\

\midrule

\multirow{3}{*}{Kolmogorov}
& \multirow{3}{*}{Markov noise}
& \(10^{-3}\) & 8.16 \\
& & \(10^{-2}\) & 8.22 \\
& & \(10^{-1}\) & 126.65 \\
\bottomrule
\end{tabular*}
\end{small}
\end{table}

\clearpage
\subsection{Ablation Study}
\label{app:ablation_study}
We further ablate the main components of PICSB on the Kolmogorov benchmark under the R3 observation regime.
The bridge component refers to the time discretization of the Brownian bridge process, while conditioning denotes hard projection onto the sparse HF observations at each sampling step.
For variants without hard conditioning, observation consistency is instead imposed through an explicit observation fitting loss.
All models are trained for \(30\)k iterations and evaluated on \(8\) test samples.

\vspace{0pt}
Table~\ref{tab:ablation_mean} shows that removing either the bridge or conditioning degrades reconstruction accuracy.
The degradation is especially severe when hard conditioning is removed, indicating that direct enforcement of sparse HF observations is essential for stable data assimilation.
Removing both bridge and conditioning also worsens performance relative to full PICSB, confirming that the bridge dynamics and hard conditioning are complementary components for accurate reconstruction.
\begin{table}[H]
\captionsetup{skip=0.05in}
\caption{Ablation study on Kolmogorov (8 samples). Lower is better. \textbf{Bold} denotes the best RelError.}
\label{tab:ablation_mean}
\centering
\begin{scriptsize}
\resizebox{\textwidth}{!}{
\begin{tabular}{lcccc}
\toprule
Method & RMSE $\downarrow$ & MAE $\downarrow$ & RelError (\%) $\downarrow$ & fRMSE ($\times10^{-3}$) $\downarrow$ \\
\midrule

PICSB
& \textbf{0.40} & \textbf{0.20} & \textbf{8.22} & \textbf{1.86} \\

\midrule

\textit{w/o Refresh}
& 0.46 & 0.26 & 9.61 & 2.33 \\

\textit{w/o Bridge}
& 0.53 & 0.27 & 11.02 & 2.18 \\

\textit{w/o Conditioning} (PalSB)
& 4.03 & 3.21 & 83.43 & 53.50 \\

\textit{w/o Bridge \& w/o Conditioning}
& 1.64 & 1.25 & 34.01 & 12.99 \\

\bottomrule
\end{tabular}
}
\end{scriptsize}
\end{table}

\subsection{Statistical Stability under Repeated Inference}
\label{app:stochastic_variability}
Because PICSB involves stochastic sampling, we evaluate its inference stability by repeating the reconstruction process independently over multiple runs.
For each benchmark, we perform \(40\) independent inference runs and report the standard deviation of the mean evaluation metrics across runs.
For Kolmogorov, the same \(5\)-fold evaluation protocol as in the main experiments is used.
All experiments follow the R3 observation regime, and the models are trained for \(30\)k iterations.

\vspace{0pt}
In addition to metric-level variability, we compute a pixelwise relative standard deviation, defined as \(\sigma_i/(|x_i|+\epsilon)\) at each spatiotemporal grid point \(i\), where \(\sigma_i\) is the standard deviation of predictions across repeated runs and \(x_i\) is the corresponding HF reference value.
Table~\ref{tab:stochastic_variability} shows that the standard deviations of all evaluation metrics are extremely small, indicating that PICSB produces stable and consistent reconstructions despite stochastic inference.

\begin{table}[H]
\centering
\caption{Statistical stability analysis under repeated inference. We report the standard deviation of mean metrics across 40 independent inference runs. Lower values indicate lower variability.}
\label{tab:stochastic_variability}
\begin{small}
\resizebox{\textwidth}{!}{
\begin{tabular}{lccccc}
\toprule
PDE & RMSE & MAE & RelError(\%) & fRMSE & pixelwise\\
\midrule




Burgers 
& 9.20\(\times10^{-5}\) & 4.47\(\times10^{-5}\) & 6.53\(\times10^{-4}\)
& 1.86\(\times10^{-6}\) & 0.261\\

Darcy 
& 6.00\(\times10^{-7}\) & 4.90\(\times10^{-7}\) & 1.46\(\times10^{-5}\)
& 1.00\(\times10^{-8}\) & 0.005\\

Kolmogorov 
& 6.26\(\times10^{-5}\) & 3.85\(\times10^{-5}\) & 1.40\(\times10^{-5}\)
& 1.41\(\times10^{-6}\) & 0.045\\
\bottomrule
\end{tabular}
}
\end{small}
\end{table}

\clearpage
\subsection{Inference under Noisy Sparse Observations}
\label{app:inference_with_noise}
In this experiment, we compare \textbf{PICSB} with \textbf{guidance-based sampling} under inference-time noise corruption in sparse HF observations. For PICSB, we use the same pretrained models as in Appendix~\ref{app:sensor_regimes}, without any retraining or finetuning.
At inference time, Gaussian noise is added to the sparse HF observations as
\[
\epsilon \sim \mathcal{N}(0, \sigma^2), \quad \sigma = \alpha \cdot \mathrm{std}(\text{data}),
\]
where the noise level \(\alpha\) is set to \(1\%\), \(3\%\), \(5\%\), \(10\%\), and \(15\%\).

\vspace{0pt}
For Darcy and Kolmogorov, the noisy sparse HF observations are interpolated to construct the corresponding LF inputs. Nearest-neighbor interpolation is used for Darcy, and bicubic interpolation is used for Kolmogorov.
These noisy LF inputs are then provided to PICSB, while the noisy sparse HF observations are enforced via hard conditioning.
For Burgers, we use the original LF inputs described in Appendix~\ref{app:subsec:DD} and apply hard conditioning using only the noisy sparse HF observations.

\vspace{0pt}
For guidance-based sampling, the same LF-pretrained models are used.
During inference, guidance is applied using the noisy sparse HF observations, following the standard guidance formulation without any additional noise-aware modification.
Both methods are evaluated under identical noise realizations for fair comparison.

\paragraph{Qualitative analysis.}
Figures~\ref{fig:with_noise_burgers_sparse}--\ref{fig:with_noise_kolmo_sensor_fix} provide qualitative comparisons under increasing noise levels.
For Burgers, PICSB preserves the global solution structure and smooth transitions even under moderate noise, whereas guidance-based sampling exhibits visible artifacts and amplified noise patterns.
For Darcy and Kolmogorov systems, the effect of noisy interpolation is evident for both methods; nevertheless, PICSB produces reconstructions that remain visually coherent and closer to the HF reference, while guidance-based sampling increasingly suffers from noise amplification and loss of structural consistency.

\vspace{0pt}
\paragraph{Quantitative results.}
Figure~\ref{fig:noise_curve} compares the relative \(L^2\) error under different inference-time noise levels across the considered datasets and observation regimes.
As the noise level increases, the reconstruction error generally increases for both PICSB and guidance-based sampling in all settings.

\vspace{0pt}
For the Kolmogorov flow, the two methods exhibit largely similar trends with respect to noise level across all sensor regimes (R1--R3).
In particular, the relative \(L^2\) error increases at a comparable rate for both PICSB and guidance-based sampling, indicating that the overall sensitivity to observation noise is similar for this system.

\vspace{0pt}
In contrast, for the Burgers equation, clearer differences emerge between the two approaches.
Except for the per-sample fixed sensor regime (R2), PICSB consistently shows a more moderate increase in relative \(L^2\) error as noise level increases, whereas guidance-based sampling exhibits a steeper error growth.
This difference is especially pronounced in the randomly resampled sensor (R1) and globally fixed sensor regimes (R3), where the error of guidance-based sampling grows rapidly with increasing noise, while PICSB maintains a more gradual degradation.

\vspace{0pt}
For Darcy flow, the quantitative curves may initially suggest that guidance-based sampling is less sensitive to increasing noise levels, as its relative \(L^2\) error remains nearly constant across noise magnitudes.
However, this behavior should not be interpreted as robustness to noise.
As illustrated by the qualitative results in Figures~\ref{fig:with_noise_darcy_sparse} and~\ref{fig:with_noise_darcy_sensor_fix}, guidance-based sampling fails to produce physically meaningful target samples even at low noise levels. In this setting, the LF prior for Darcy is constructed via nearest-neighbor interpolation from sparse HF observations, resulting in a discretized and piecewise-constant LF representation.
When combined with guidance-based sampling, the noisy observation gradients are insufficient to effectively steer the generative process toward a valid HF solution.
Consequently, the guidance-based method exhibits consistently high error regardless of the noise level, reflecting a fundamental limitation in leveraging the available observations rather than genuine robustness to noise.

\vspace{0pt}
In contrast, PICSB maintains lower reconstruction error with a gradual increase as noise grows, indicating that it is able to utilize noisy sparse observations more effectively through hard conditioning within the learned LF-to-HF transport, even in the presence of discretized LF inputs.

Overall, these results characterize reconstruction behavior under varying inference-time noise levels in sparse observations.
As the noise level increases, the relative \(L^2\) error increases for both PICSB and guidance-based sampling across all datasets and observation regimes.
However, the rate at which the error grows differs between the two approaches, depending on the system and observation setting.

\vspace{0pt}
Across many configurations, PICSB exhibits a more gradual increase in relative \(L^2\) error as noise increases, whereas guidance-based sampling tends to show a steeper degradation, particularly at higher noise levels.
This difference reflects how noisy sparse observations are incorporated during inference: PICSB enforces noisy measurements through hard conditioning within a learned LF-to-HF transport, while guidance-based sampling directly injects gradients derived from noisy observations, which can lead to stronger error amplification as noise increases.
\begin{figure}[H]
    \centering
    \includegraphics[width=1.0\linewidth]{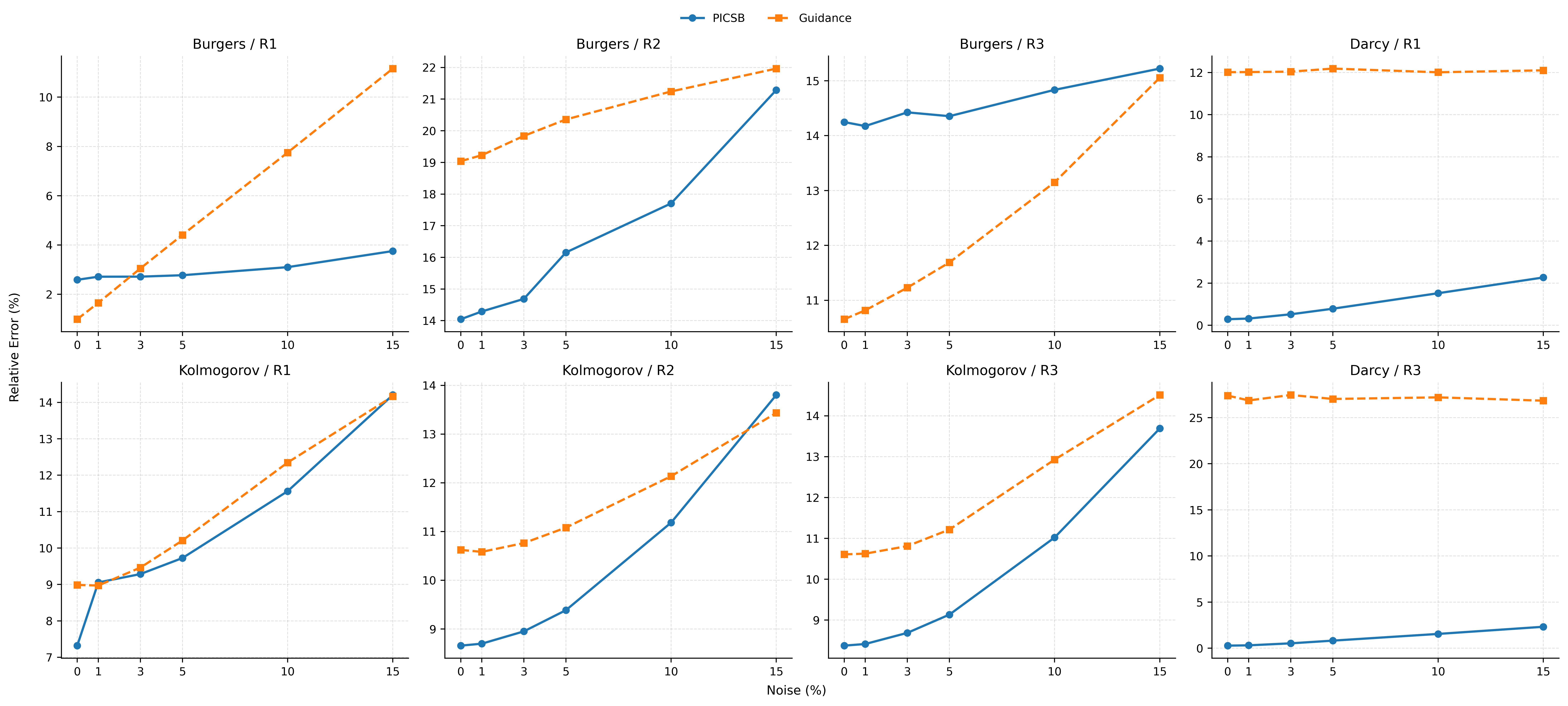}
    \caption{Relative \(L^2\) error under noisy observations across Burgers, Darcy, and Kolmogorov observation regimes.}
    \label{fig:noise_curve}
    \vspace{0pt}
\end{figure}

\clearpage

\begin{figure}[H]
    \centering
    \includegraphics[width=1.0\linewidth]{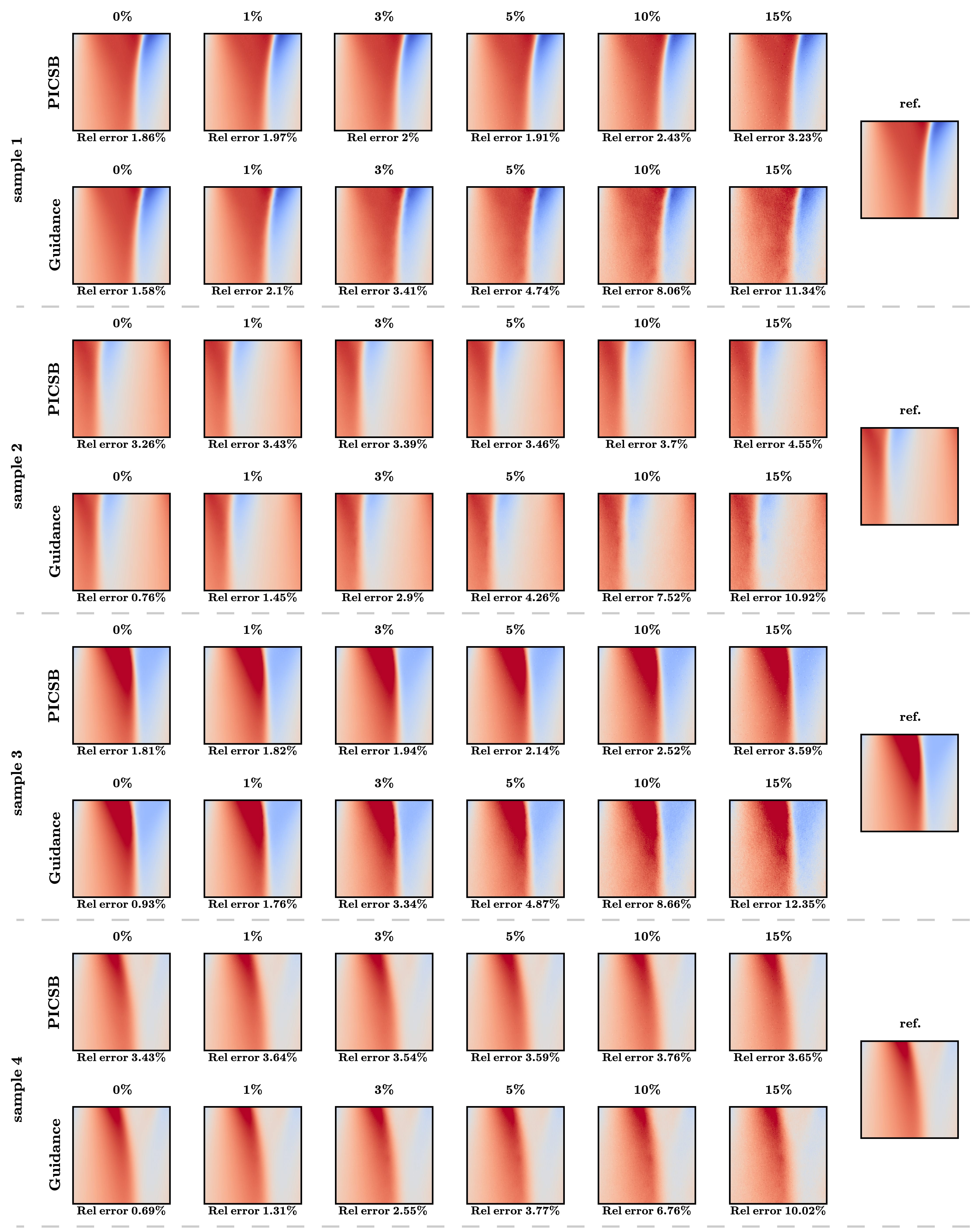}
    \caption{Inference under noisy observations for Burgers' equation with randomly resampled sensors (R1).
    Each row corresponds to a different test sample, and each column corresponds to a noise level, with the HF reference solution shown in the rightmost column.
Relative \(L^2\) errors are reported below each predicted field.}
    \label{fig:with_noise_burgers_sparse}
    \vspace{0pt}
\end{figure}

\clearpage

\begin{figure}[H]
    \centering
    \includegraphics[width=1.0\linewidth]{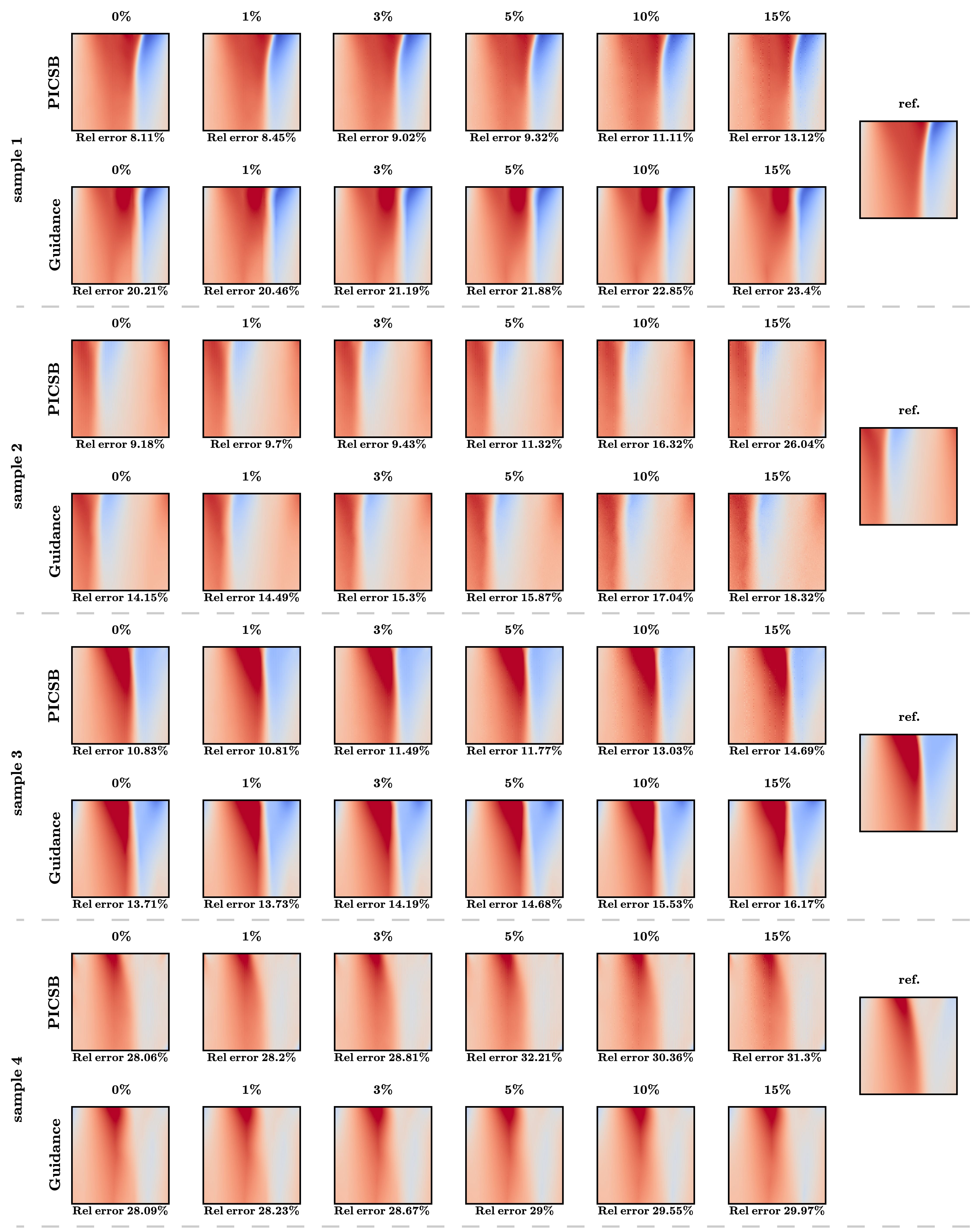}
    \caption{Inference under noisy observations for Burgers' equation with per-sample fixed sensors (R2).
    Each row corresponds to a different test sample, and each column corresponds to a noise level, with the HF reference solution shown in the rightmost column.
Relative \(L^2\) errors are reported below each predicted field.}
    \label{fig:with_noise_burgers_sensor_random}
    \vspace{0pt}
\end{figure}

\clearpage

\begin{figure}[H]
    \centering
    \includegraphics[width=1.0\linewidth]{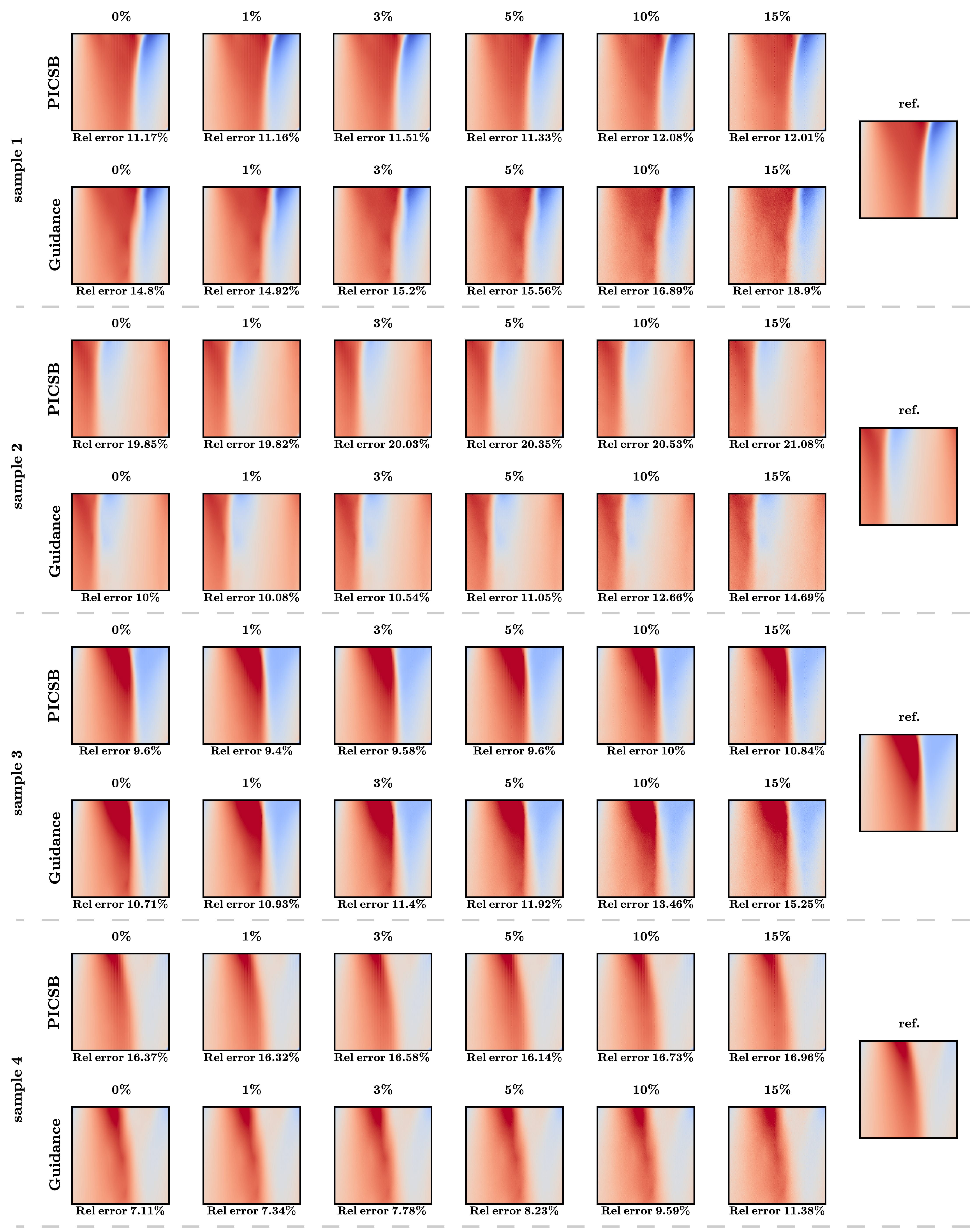}
    \caption{Inference under noisy observations for Burgers' equation with a fixed sensor set (R3).
    Each row corresponds to a different test sample, and each column corresponds to a noise level, with the HF reference solution shown in the rightmost column.
Relative \(L^2\) errors are reported below each predicted field.}
    \label{fig:with_noise_burgers_sensor_fix}
    \vspace{0pt}
\end{figure}

\clearpage

\begin{figure}[H]
    \centering
    \includegraphics[width=1.0\linewidth]{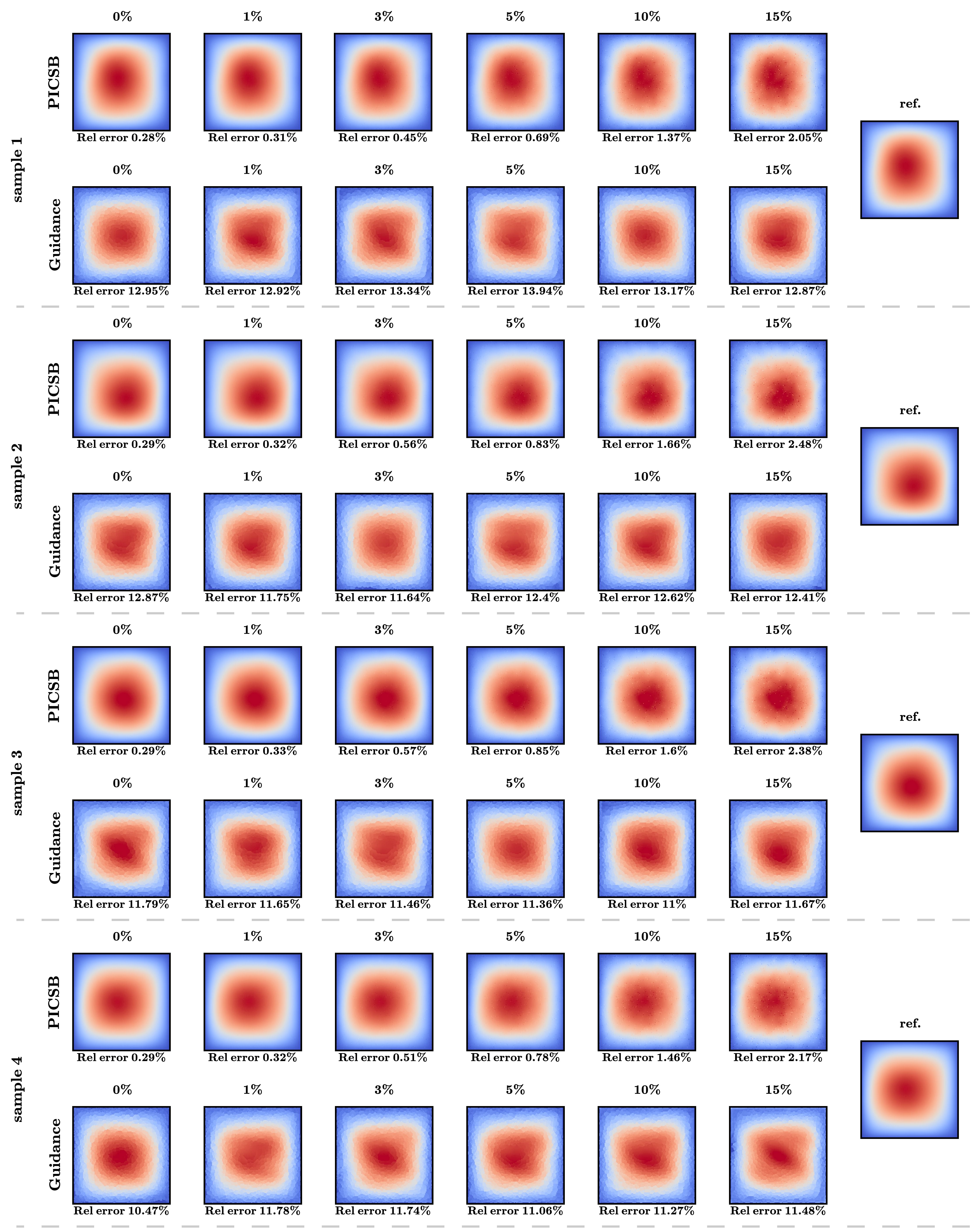}
    \caption{Inference under noisy observations for Darcy flow with randomly resampled sensors (R1).
    Each row corresponds to a different test sample, and each column corresponds to a noise level, with the HF reference solution shown in the rightmost column.
Relative \(L^2\) errors are reported below each predicted field.}
    \label{fig:with_noise_darcy_sparse}
    \vspace{0pt}
\end{figure}

\clearpage

\begin{figure}[H]
    \centering
    \includegraphics[width=1.0\linewidth]{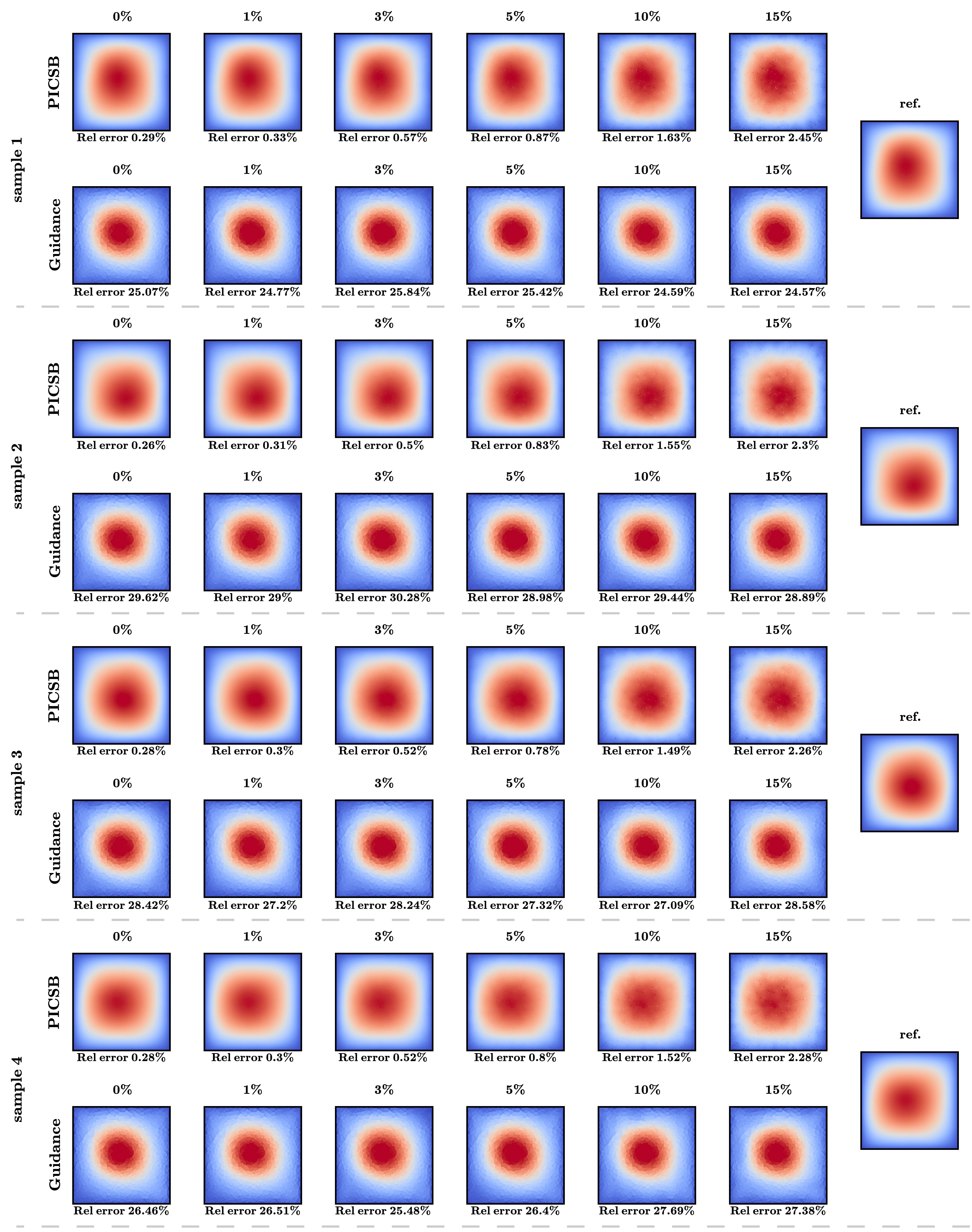}
    \caption{Inference under noisy observations for Darcy flow with a fixed sensor set (R3).
    Each row corresponds to a different test sample, and each column corresponds to a noise level, with the HF reference solution shown in the rightmost column.
Relative \(L^2\) errors are reported below each predicted field.}
    \label{fig:with_noise_darcy_sensor_fix}
    \vspace{0pt}
\end{figure}

\clearpage

\begin{figure}[H]
    \centering
    \includegraphics[width=1.0\linewidth]{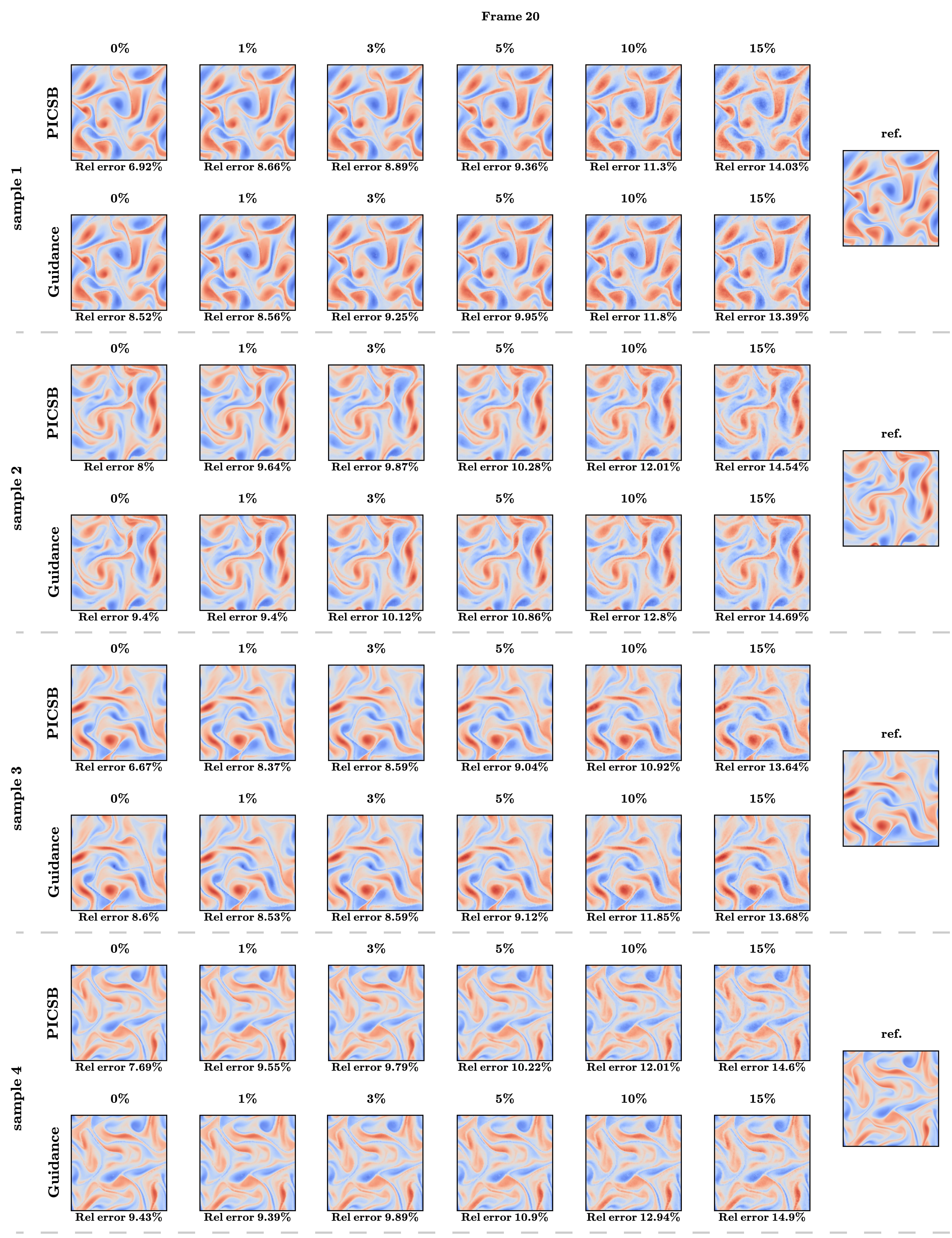}
    \caption{Inference under noisy sparse observations for Kolmogorov flow with randomly resampled sensors (R1).
    Each row corresponds to a different test sample, and each column corresponds to a noise level, with the HF reference solution shown in the rightmost column.
Relative \(L^2\) errors are reported below each predicted field.}
    \label{fig:with_noise_kolmo_sparse}
    \vspace{0pt}
\end{figure}

\clearpage

\begin{figure}[H]
    \centering
    \includegraphics[width=1.0\linewidth]{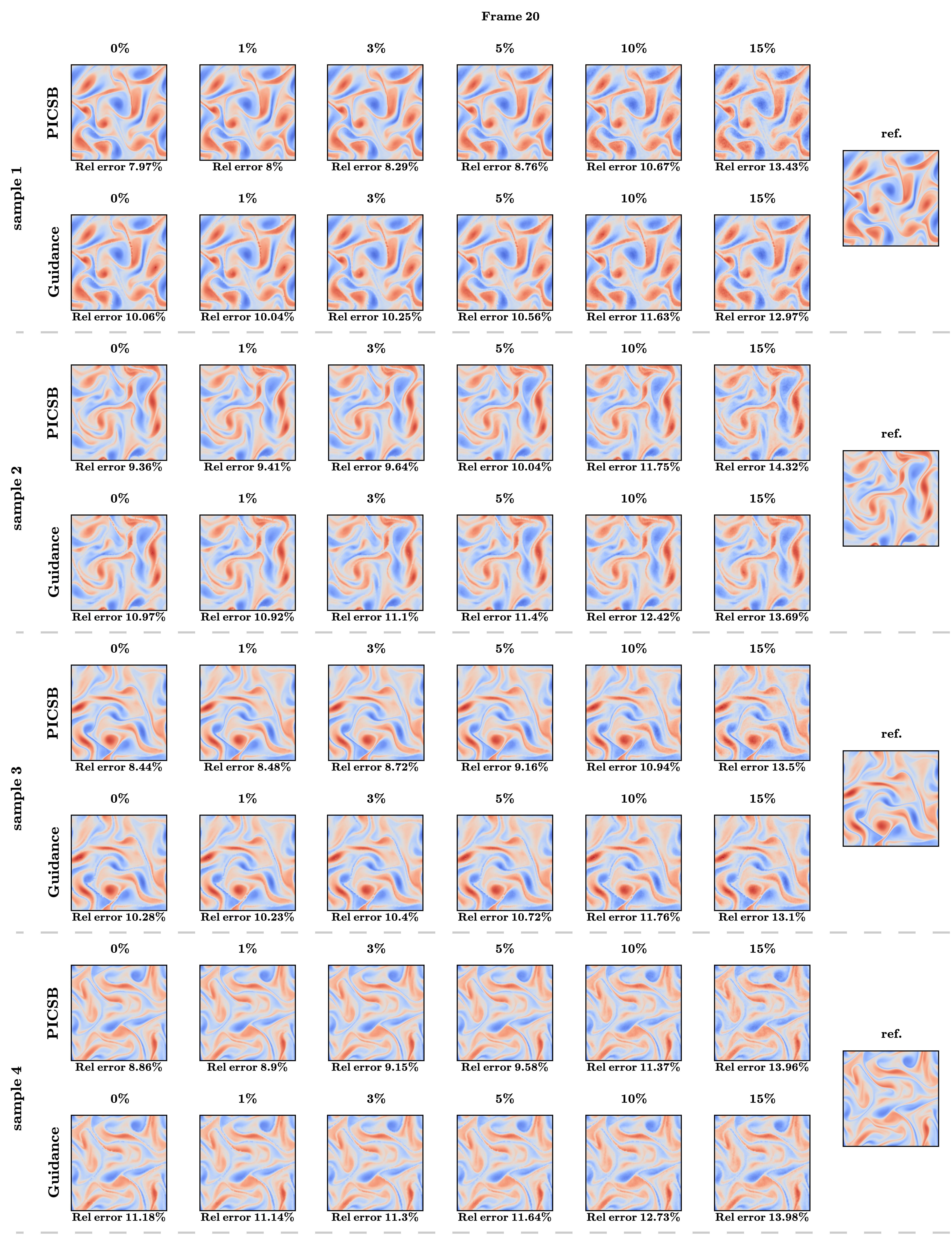}
    \caption{Inference under noisy observations for Kolmogorov flow with per-sample fixed sensors (R2).
    Each row corresponds to a different test sample, and each column corresponds to a noise level, with the HF reference solution shown in the rightmost column.
Relative \(L^2\) errors are reported below each predicted field.}
    \label{fig:with_noise_kolmo_sensor_random}
    \vspace{0pt}
\end{figure}

\clearpage

\begin{figure}[H]
    \centering
    \includegraphics[width=1.0\linewidth]{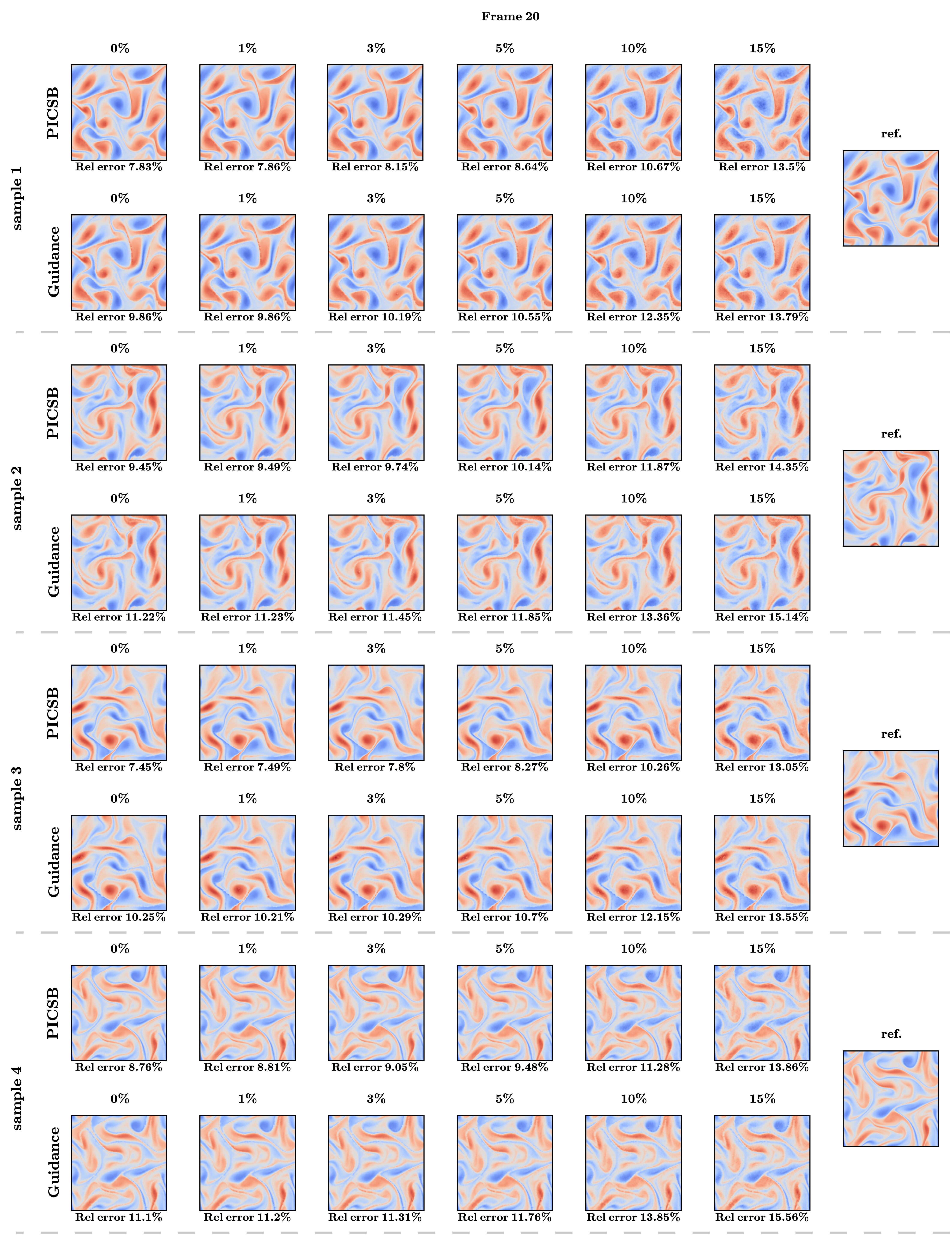}
    \caption{Inference under noisy observations for Kolmogorov flow with a fixed sensor set (R3).
    Each row corresponds to a different test sample, and each column corresponds to a noise level, with the HF reference solution shown in the rightmost column.
Relative \(L^2\) errors are reported below each predicted field.}
    \label{fig:with_noise_kolmo_sensor_fix}
    \vspace{0pt}
\end{figure}

\clearpage
\subsection{Finetuning across Different Levels of Observation Sparsity}
\label{app:appendix_finetuning}
In Section~\ref{app:lessobs}, we observed that directly applying models trained with \(10\%\) sparse observations to substantially sparser inference regimes (e.g., \(1\%\) or \(5\%\) of the full grid) leads to severe degradation.
Inference under such extreme sparsity remains challenging in general when no additional adaptation is performed, particularly for more complex dynamics.
To further investigate whether limited adaptation can mitigate this limitation, we conduct additional finetuning experiments across different levels of observation sparsity.

All models are first pretrained using \(10\%\) sparse HF observations, following the same configuration as in the main experiments.
For Darcy and Kolmogorov, LF inputs are constructed from the \(10\%\) sparse HF observations via nearest-neighbor and bi-cubic interpolation, respectively, and used during training.
For Burgers, the same LF inputs as in the main experiments are used.

\paragraph{Finetuning Regime 1.}
Starting from these pretrained models, we first consider finetuning under lower observation availability by subsampling the original \(10\%\) observation set.
Specifically, we retain either \(50\%\) or \(10\%\) of the available observation points, corresponding to conditioning at the \(5\%\) and \(1\%\) observation levels, respectively.
For Darcy and Kolmogorov, the resulting observation sets are used to construct the corresponding LF inputs using the same interpolation schemes as in training, while hard conditioning is imposed at the selected sparse locations.
For Burgers, the LF inputs remain unchanged, and only the hard conditioning set is reduced accordingly.

\begin{figure}[H]
    \centering
    \includegraphics[width=1.0\linewidth,height=.275\textheight,keepaspectratio]{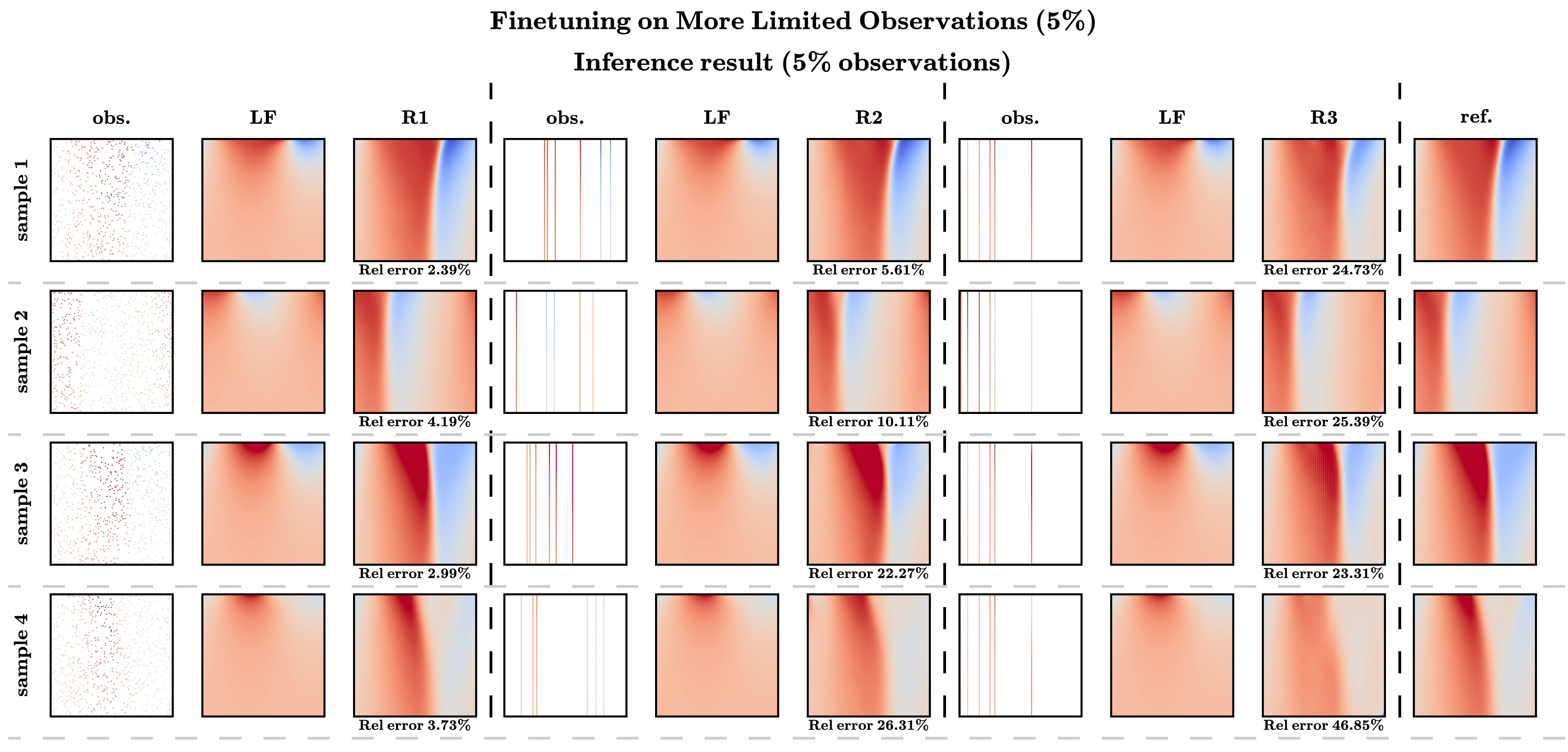}
    \caption{Finetuning and inference results under \(5\%\) sparse observations for the Burgers' equation.}
    \label{fig:finetune_burgers_5}
    \vspace{0pt}
\end{figure}

\begin{figure}[H]
    \centering
    \includegraphics[width=1.0\linewidth,height=.275\textheight,keepaspectratio]{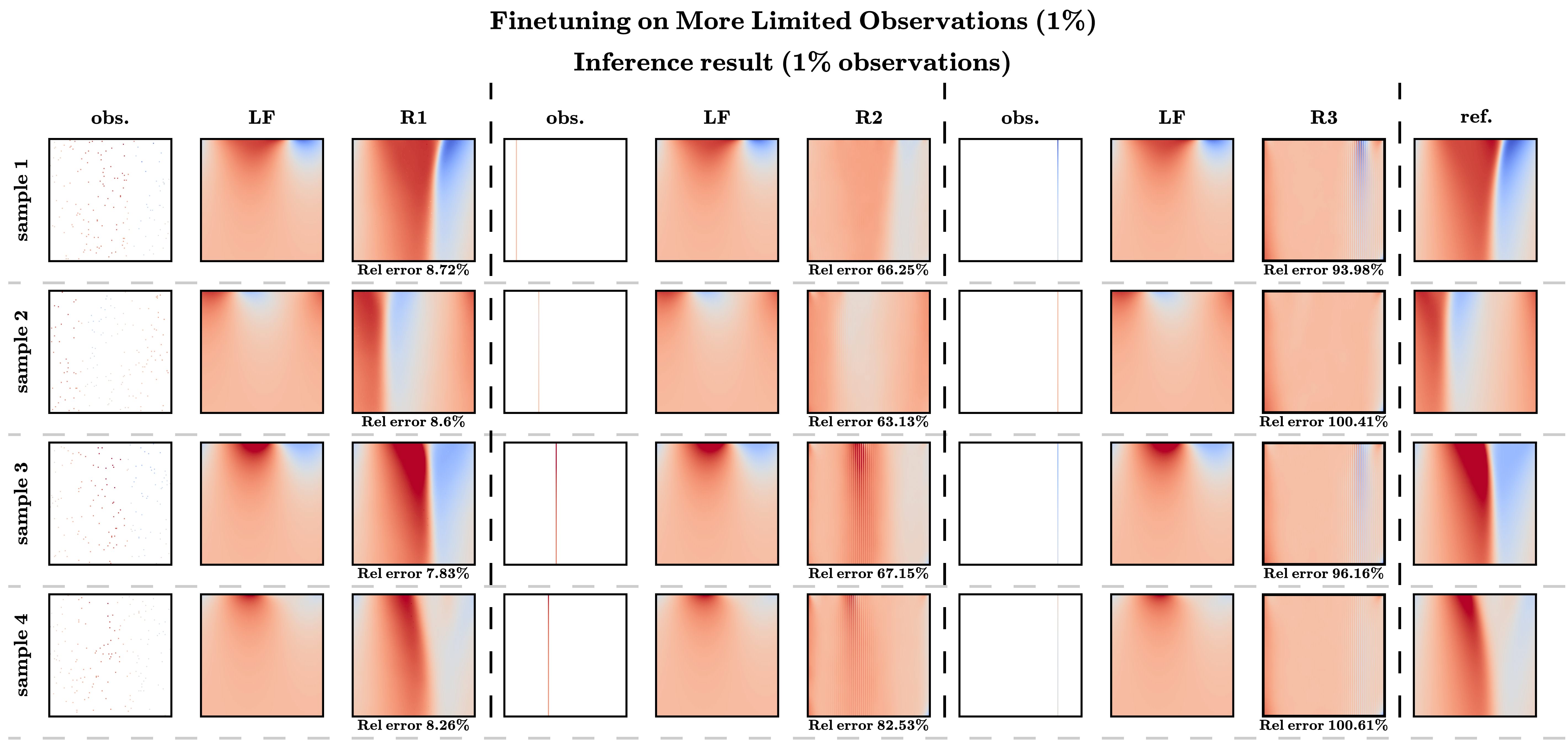}
    \caption{Finetuning and inference results under \(1\%\) sparse observations for the Burgers' equation.}
    \label{fig:finetune_burgers_1}
    \vspace{0pt}
\end{figure}

\clearpage

\begin{figure}[H]
    \centering
    \includegraphics[width=1.0\linewidth]{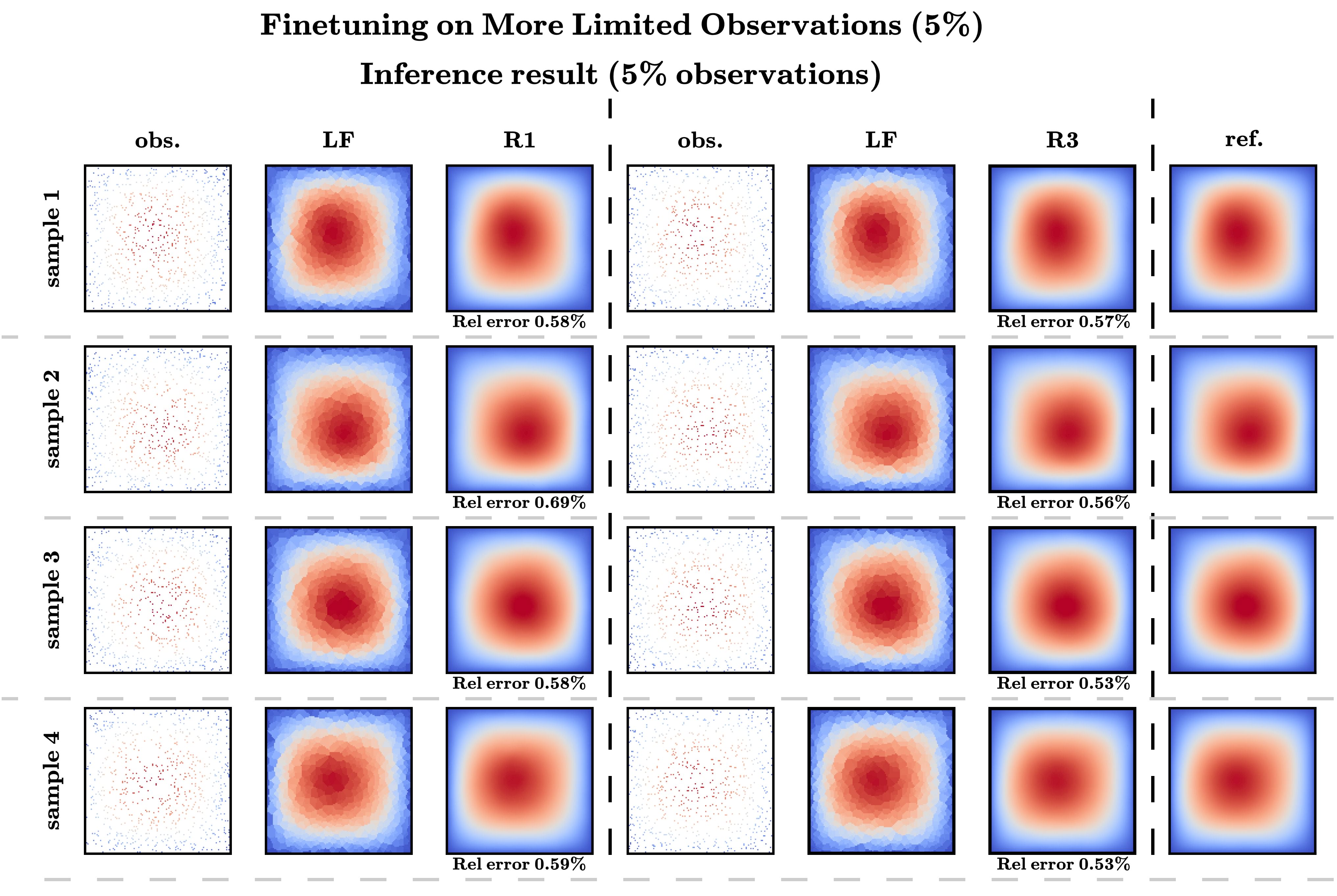}
    \caption{Finetuning and inference results under \(5\%\) sparse observations for Darcy flow.}
    \label{fig:finetune_darcy_5}
    \vspace{0pt}
\end{figure}

\begin{figure}[H]
    \centering
    \includegraphics[width=1.0\linewidth]{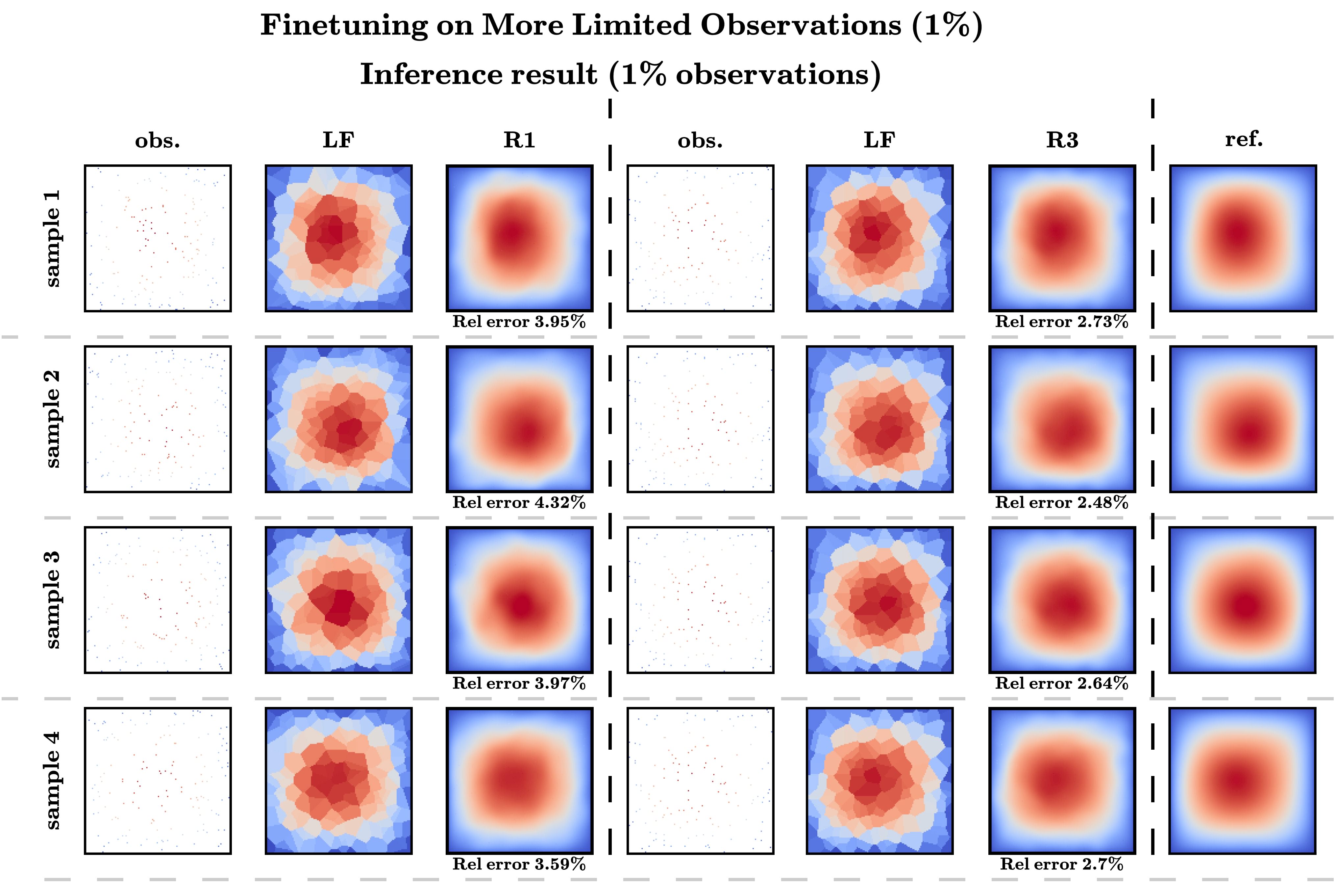}
    \caption{Finetuning and inference results under \(1\%\) sparse observations for Darcy flow.}
    \label{fig:finetune_darcy_1}
    \vspace{0pt}
\end{figure}

\clearpage

\begin{figure}[H]
    \centering
    \includegraphics[width=1.0\linewidth]{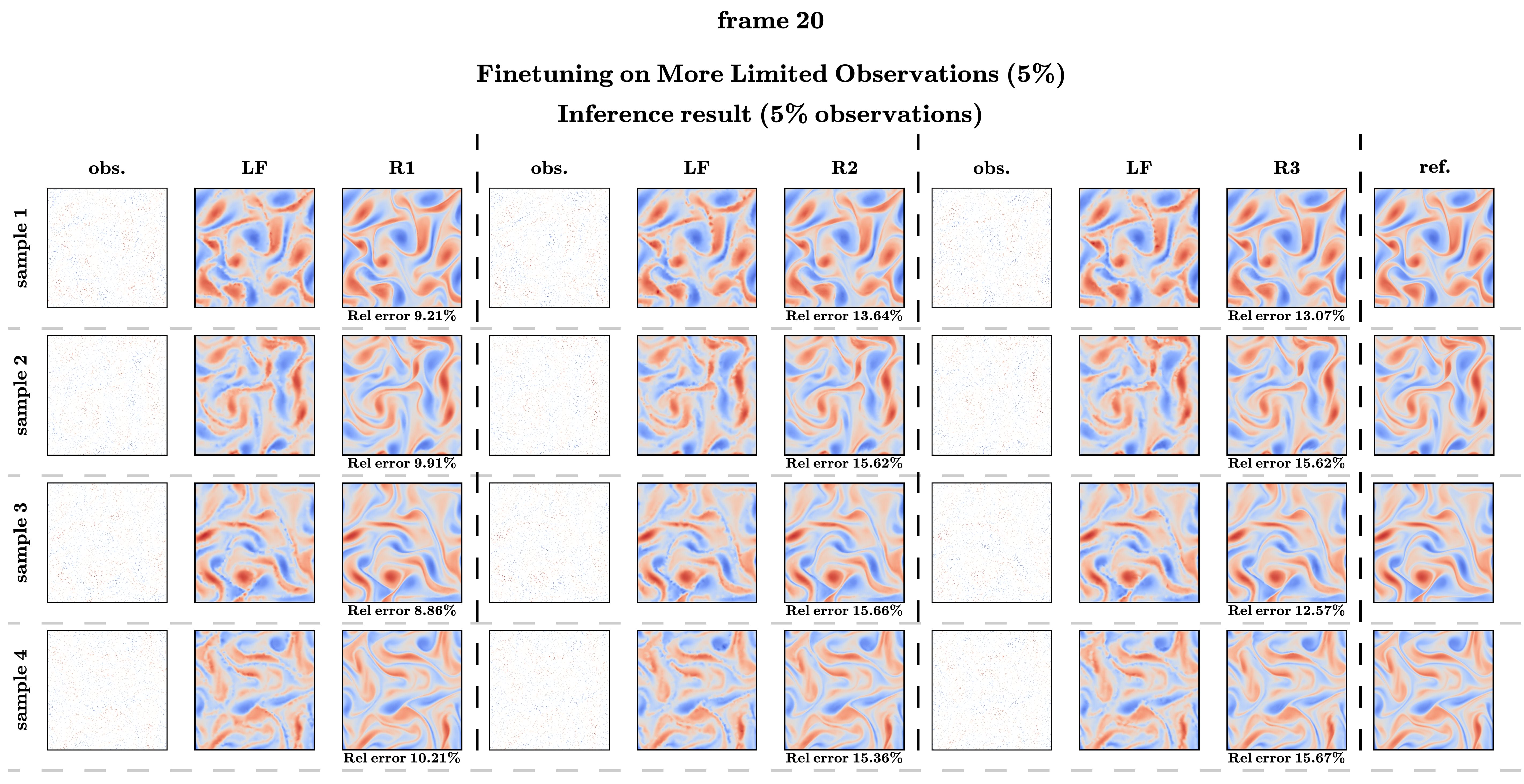}
    \caption{Finetuning and inference results under \(5\%\) sparse observations for the Kolmogorov flow.}
    \label{fig:finetune_kolmo_5}
    \vspace{0pt}
\end{figure}

\begin{figure}[H]
    \centering
    \includegraphics[width=1.0\linewidth]{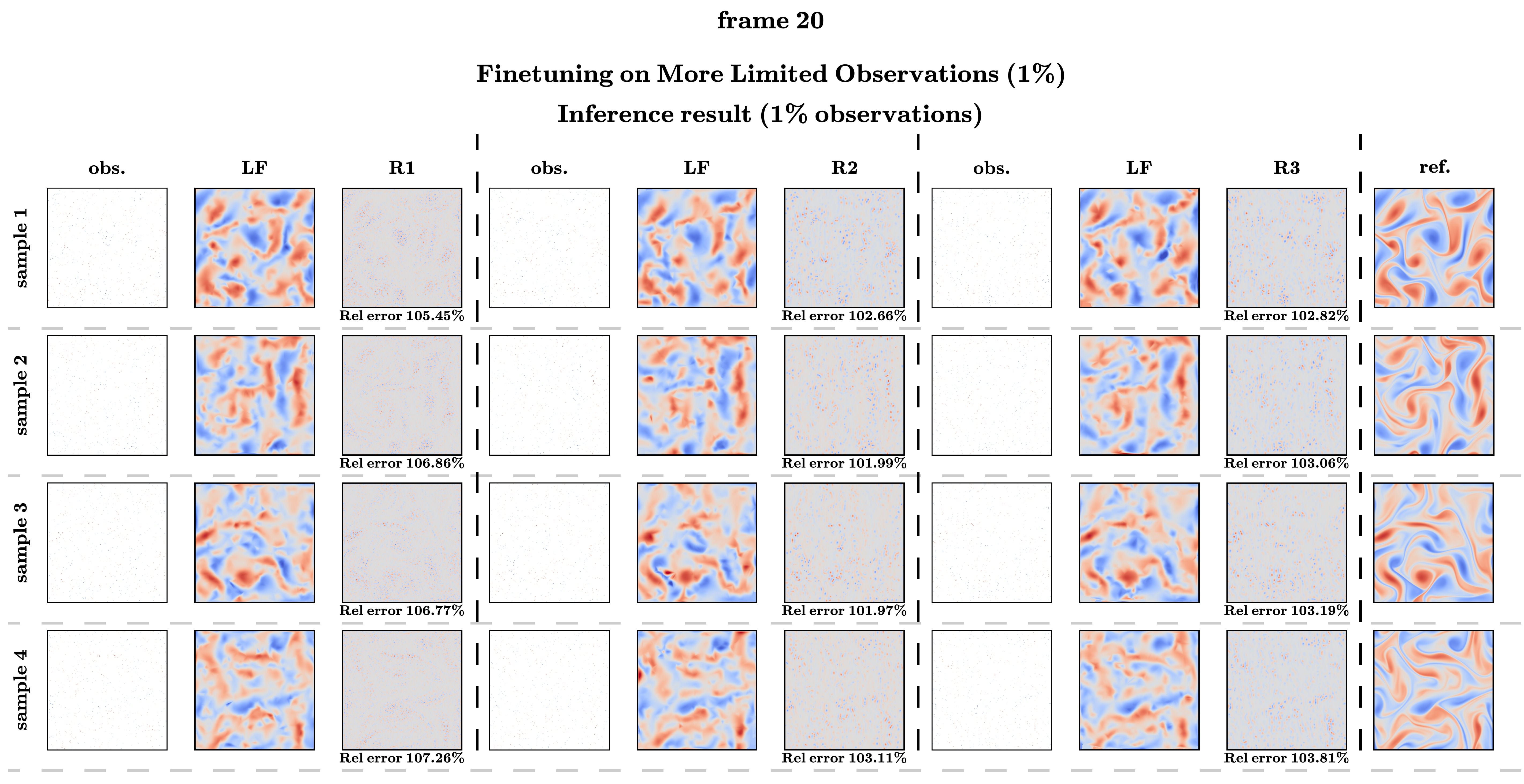}
    \caption{Finetuning and inference results under \(1\%\) sparse observations for the Kolmogorov flow.}
    \label{fig:finetune_kolmo_1}
    \vspace{0pt}
\end{figure}

\clearpage
\paragraph{Finetuning Regime 2.}
While the finetuning strategy in Regime 1 yields reasonable reconstructions at the \(5\%\) observation level, it fails to maintain performance when the observation level is reduced to \(1\%\) (see Figures~\ref{fig:finetune_burgers_5}--\ref{fig:finetune_kolmo_1}).
To further examine this limitation, we consider an alternative finetuning strategy aimed at improving robustness across different observation sparsity levels.

\vspace{0pt}
Starting from the same \(10\%\)-pretrained models, finetuning is performed using the same observation datasets as in Regime 1.
However, rather than finetuning separately for each observation level, we adopt an \emph{integrated mode} in which finetuning is carried out on a nested collection of observation subsets corresponding to \(10\%\), \(5\%\), and \(1\%\) of the full grid.
In this setting, the finetuning dataset is formed by aggregating these nested sparsity levels, allowing the model to adapt simultaneously across different observation regimes.

\vspace{0pt}
The optimizer state and learning rate settings are identical to those used in Regime 1, with finetuning performed for the same total budget of \(20{,}000\) iterations.
After finetuning, inference is evaluated separately under \(5\%\) and \(1\%\) observation regimes.

\vspace{0pt}
\begin{figure}[H]
    \centering
    \includegraphics[width=1.0\linewidth]{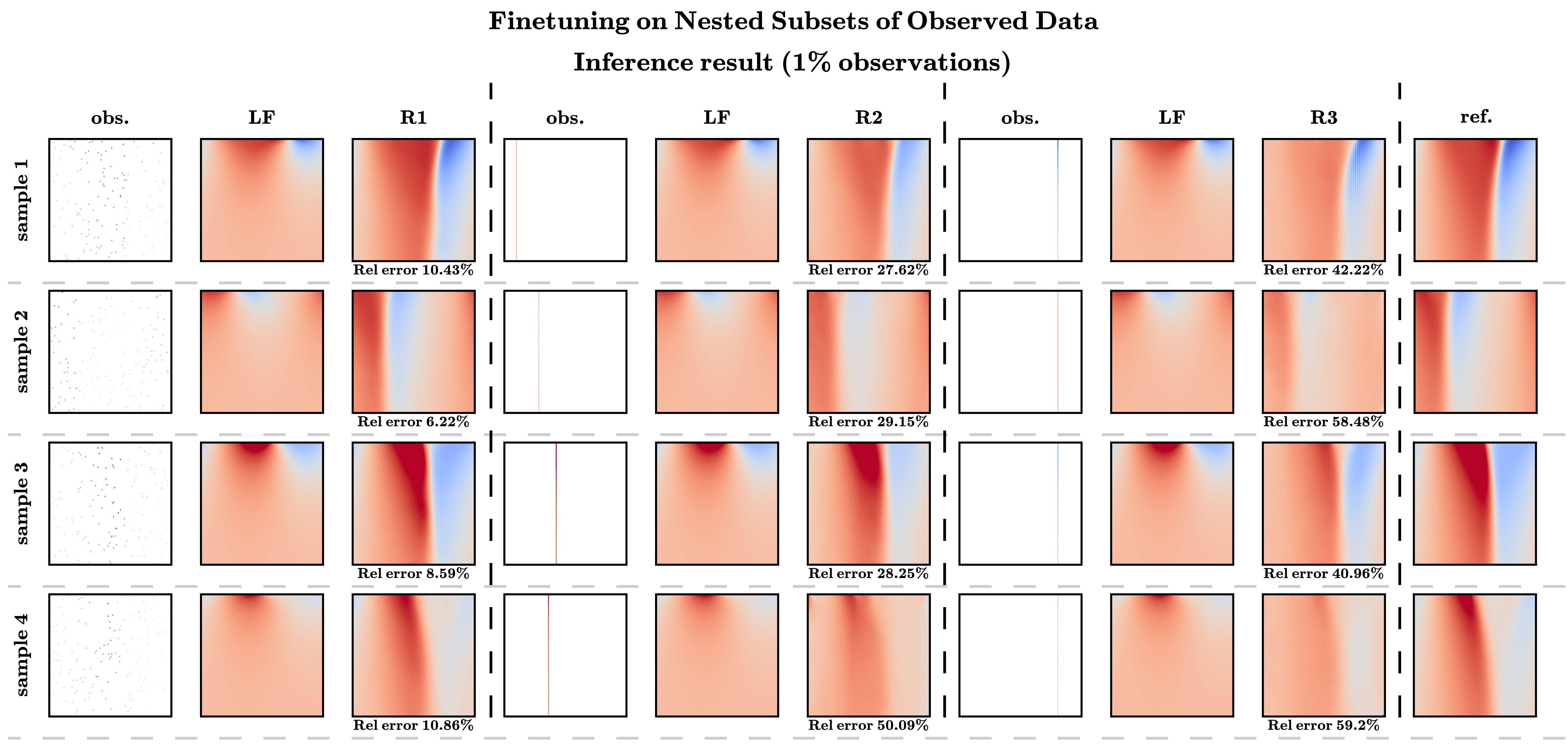}
    \caption{Finetuning and inference results under \(1\%\) sparse observations for the Burgers' equation with integrated finetuning.}
    \label{fig:finetune2_burgers_1}
    \vspace{0pt}
\end{figure}

\begin{figure}[H]
    \centering
    \includegraphics[width=1.0\linewidth]{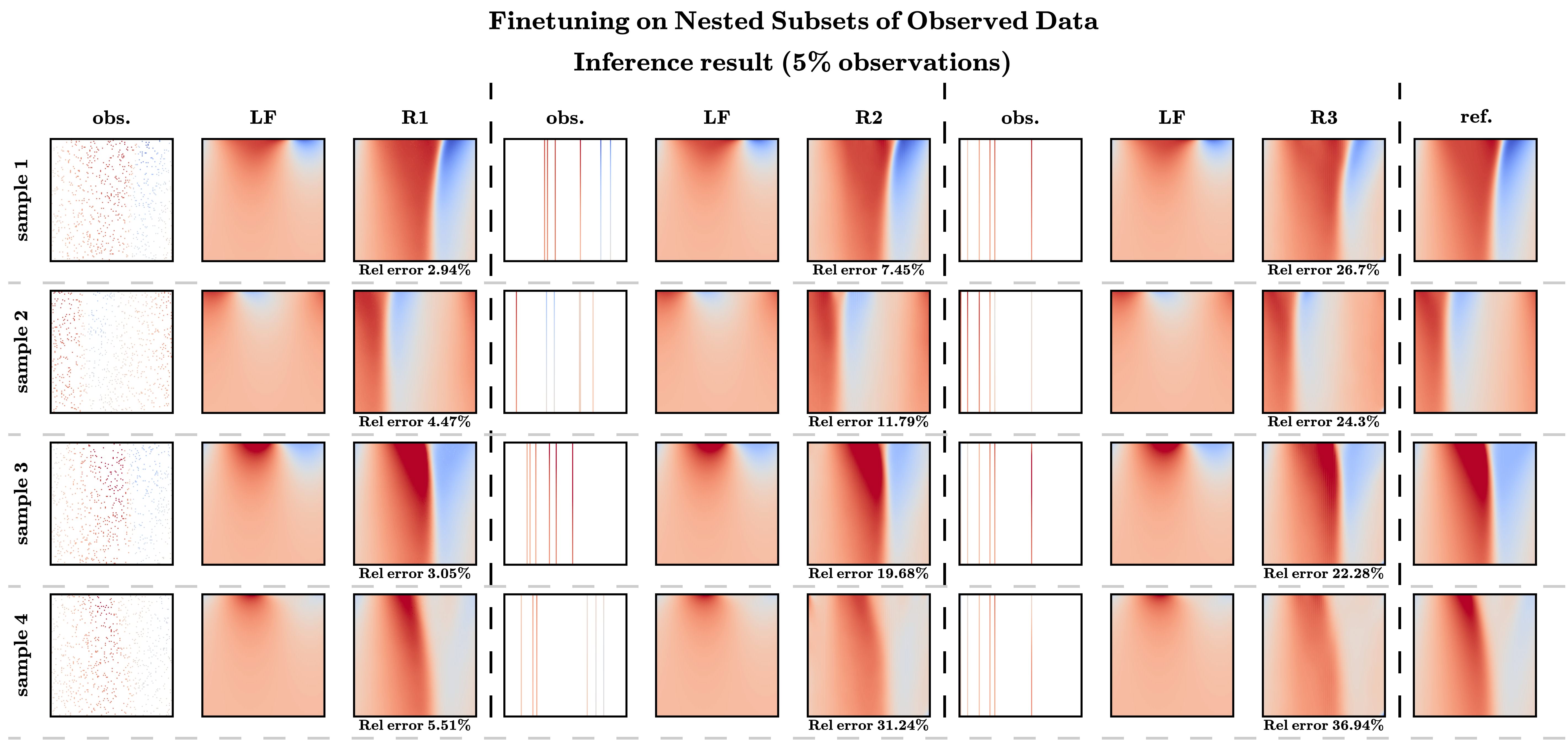}
    \caption{Finetuning and inference results under \(5\%\) sparse observations for the Burgers' equation with integrated finetuning.}
    \label{fig:finetune2_burgers_5}
    \vspace{0pt}
\end{figure}

\clearpage

\begin{figure}[H]
    \centering
    \includegraphics[width=1.0\linewidth]{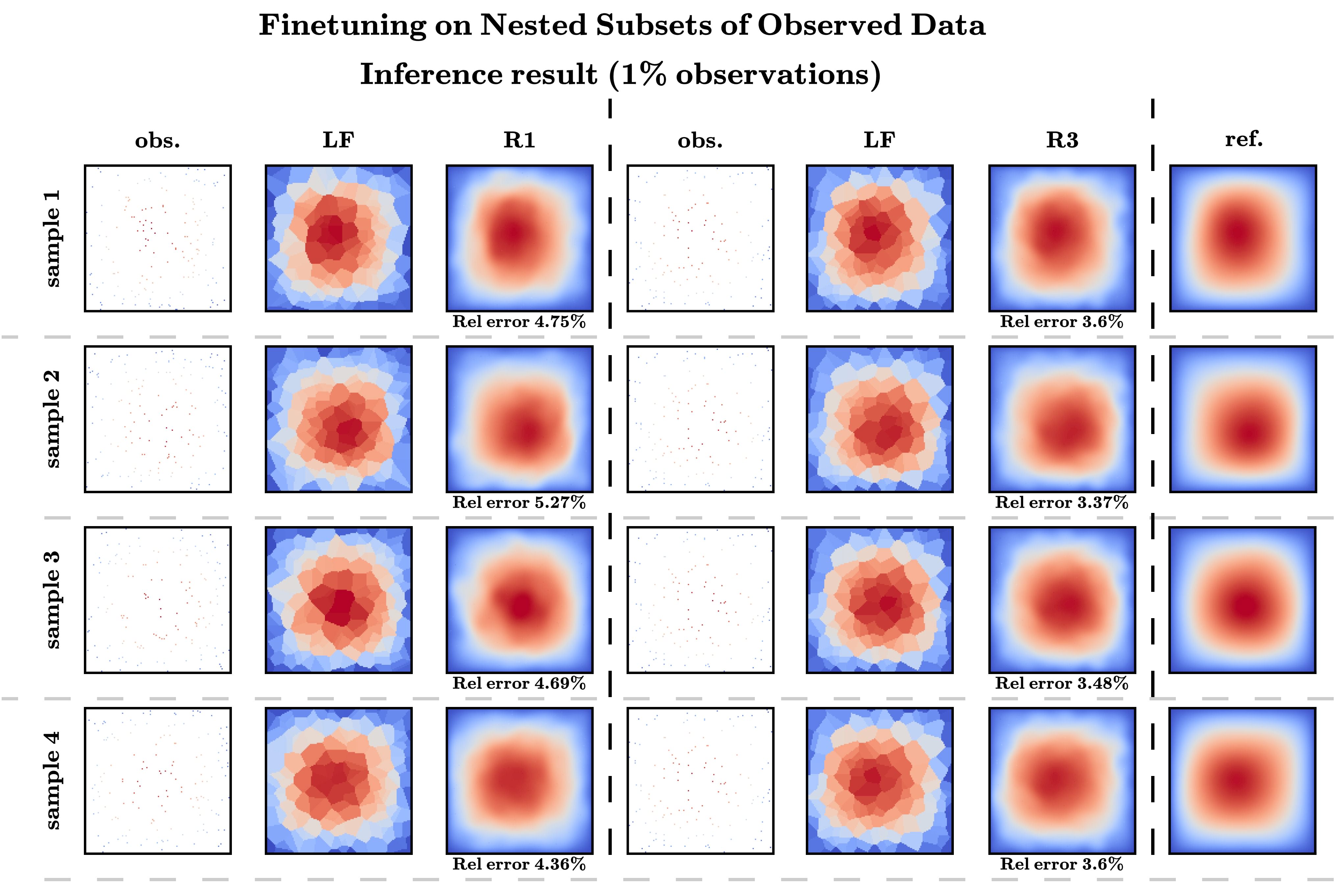}
    \caption{Finetuning and inference results under \(1\%\) sparse observations for Darcy flow with integrated finetuning.}
    \label{fig:finetune2_darcy_1}
    \vspace{0pt}
\end{figure}

\begin{figure}[H]
    \centering
    \includegraphics[width=1.0\linewidth]{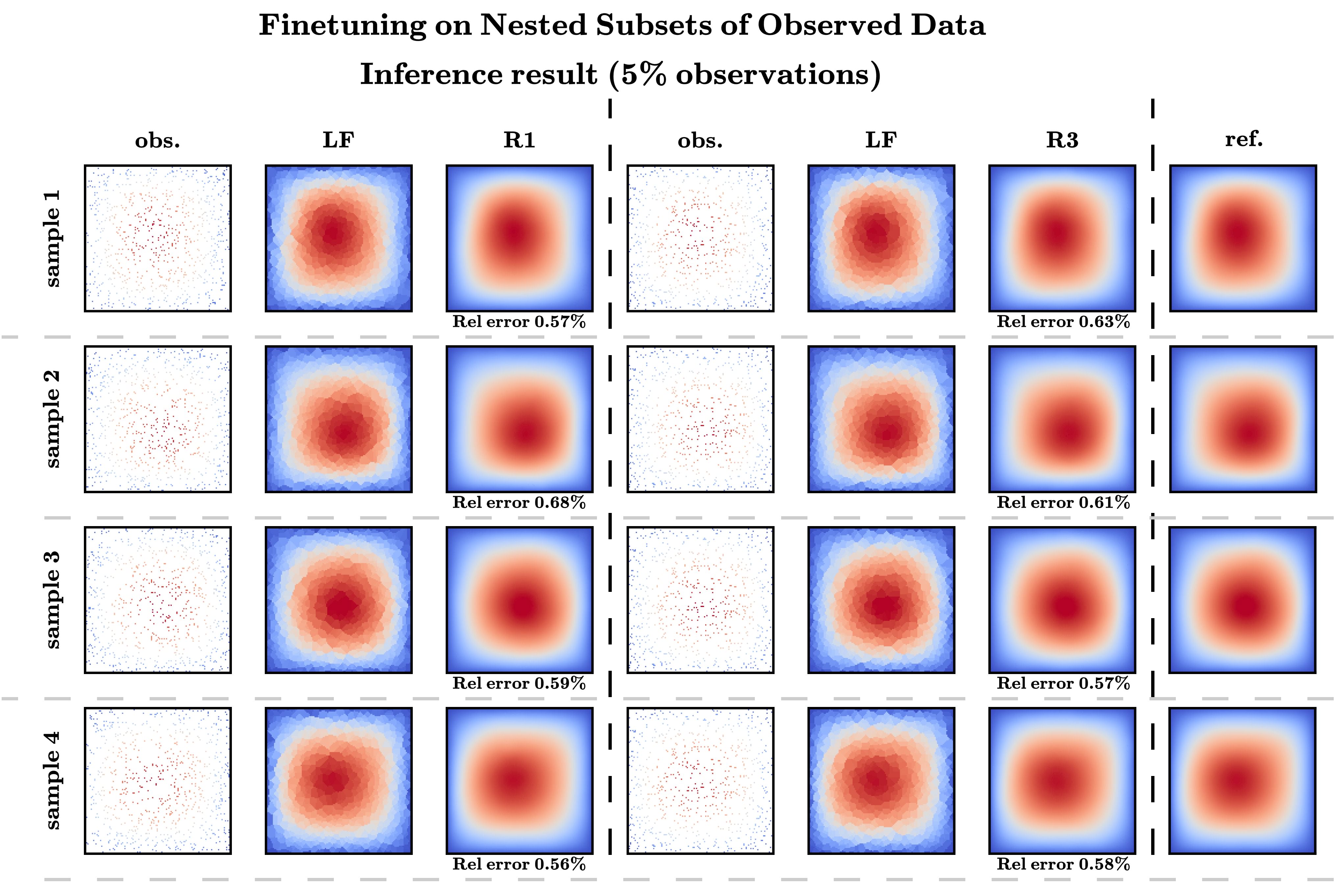}
    \caption{Finetuning and inference results under \(5\%\) sparse observations for Darcy flow with integrated finetuning.}
    \label{fig:finetune2_darcy_5}
    \vspace{0pt}
\end{figure}

\clearpage

\begin{figure}[H]
    \centering
    \includegraphics[width=1.0\linewidth]{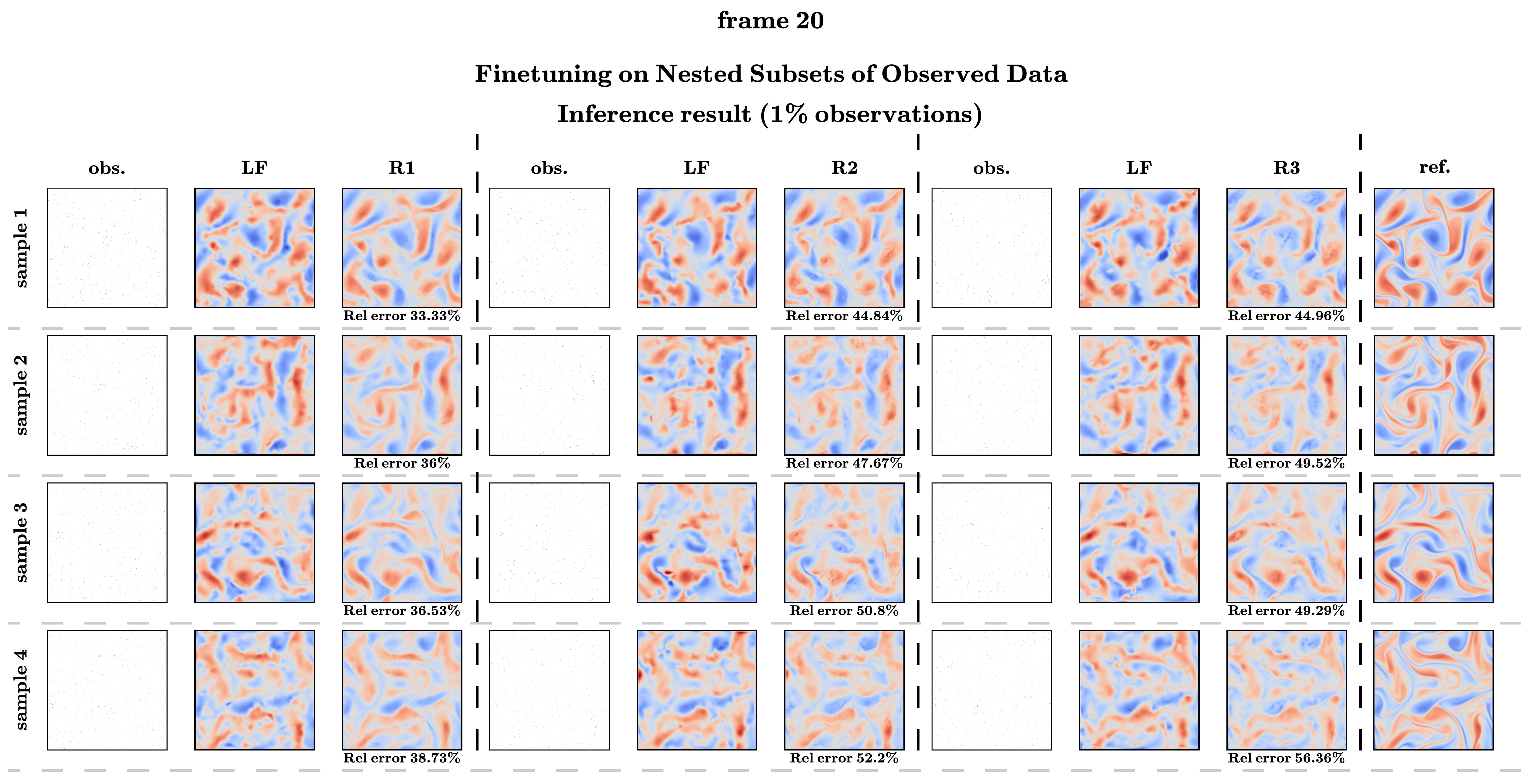}
    \caption{Finetuning and inference results under \(1\%\) sparse observations for the Kolmogorov flow with integrated finetuning.}
    \label{fig:finetune2_kolmo_1}
    \vspace{0pt}
\end{figure}

\begin{figure}[H]
    \centering
    \includegraphics[width=1.0\linewidth]{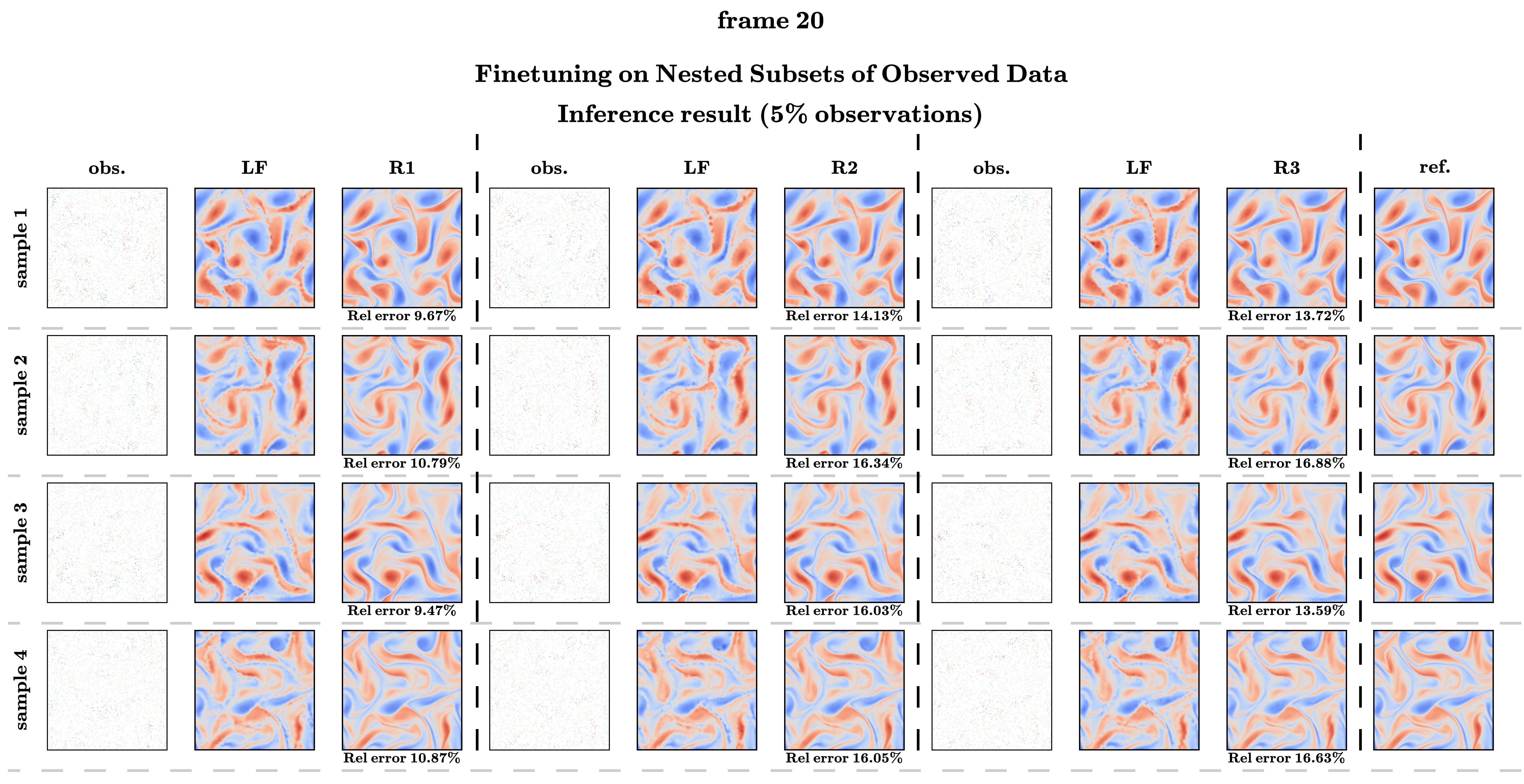}
    \caption{Finetuning and inference results under \(5\%\) sparse observations for the Kolmogorov flow with integrated finetuning.}
    \label{fig:finetune2_kolmo_5}
    \vspace{0pt}
\end{figure}

As a result, Figures~\ref{fig:finetune2_burgers_5}, \ref{fig:finetune2_darcy_5}, and \ref{fig:finetune2_kolmo_5} show that finetuning under the integrated mode maintains stable reconstruction quality at the \(5\%\) observation level.
At the more challenging \(1\%\) observation level, Figures~\ref{fig:finetune2_burgers_1}, \ref{fig:finetune2_darcy_1}, and \ref{fig:finetune2_kolmo_1} demonstrate that the integrated finetuning strategy yields improved robustness relative to finetuning performed separately for each observation level.

\vspace{0pt}
Although the reconstructions obtained at \(1\%\) sparsity remain imperfect and exhibit non-negligible errors, integrated finetuning enables meaningful recovery in scenarios where per-level finetuning largely fails.
Overall, these results indicate that aggregating multiple observation sparsity levels during finetuning can partially alleviate the degradation caused by extreme observation scarcity, while substantial performance gaps persist in such severely underdetermined settings.


\end{document}